\documentclass{article} 
\usepackage[accepted]{icml2019}
\usepackage{hyperref}
\usepackage{url}

\usepackage{subfigure}
\usepackage{wrapfig}

\usepackage{graphicx}

\usepackage[utf8]{inputenc} 
\usepackage[T1]{fontenc}    
\usepackage{hyperref}       
\usepackage{url}            
\usepackage{booktabs}       
\usepackage{amsfonts}       
\usepackage{nicefrac}       
\usepackage{microtype}      

\newcommand{\ignore}[1]{}
\newcommand{\notinproc}[1]{#1}
\newcommand{\onlyinproc}[1]{}
\usepackage{url}
\usepackage{amssymb,amsmath,amsfonts}
\usepackage{amsmath,amsthm}
\usepackage{graphicx}
\usepackage{times}
\usepackage{algorithm}
\usepackage[algo2e, ruled, vlined]{algorithm2e}
\usepackage{xcolor} 

\SetCommentSty{mycommfont}

\newtheorem{thm}{Theorem}[section]

\newtheorem{lemma}[thm]{Lemma}

\def\vecf{\boldsymbol{f}}
\def\vecc{\boldsymbol{c}}
\def\vecv{\boldsymbol{v}}
\def\vecu{\boldsymbol{u}}

\newcommand{\E}{\mathbb{E}}

\def\ind{{\sc ind}}
\def\coo{{\sc coo}}
\def\coolsh{{\sc coo$+$LSH}}
\def\opt{{\sc coo$+$OptLSH}}
\def\mix{{\sc mix}}

\def\movielens{{\sc movielens1M}}
\def\amazon{{\sc Amazon}}

\DeclareMathOperator{\Prod}{\Pi}
\DeclareMathOperator{\Exp}{Exp}
\DeclareMathOperator{\cossim}{\cos_{\text{sim}}}

\icmltitlerunning{Self-Similar Epochs}

\begin{document}

\twocolumn[
\icmltitle{Self-Similar Epochs: Value in Arrangement}

\icmlsetsymbol{equal}{*}

\begin{icmlauthorlist}
\icmlauthor{Eliav Buchnik}{equal,tau,gooAI}
\icmlauthor{Edith Cohen}{equal,gooAI,tau}
\icmlauthor{Avinatan Hassidim}{gooAI}
\icmlauthor{Yossi Matias}{gooAI}
\end{icmlauthorlist}

\icmlaffiliation{gooAI}{Google Research}
\icmlaffiliation{tau}{Tel Aviv University, Israel}

\icmlcorrespondingauthor{Eliav Buchnik}{eliavbuh@gmail.com}
\icmlcorrespondingauthor{Edith Cohen}{edith@cohenwang.com}

\icmlkeywords{self-similar epochs, microbatches, Jaccard similarity,
  Stochastic Gradient Descent, Metric Embeddings}

\vskip 0.3in
] 

\printAffiliationsAndNotice{\icmlEqualContribution}

\ignore{
\author{
  Eliav Buchnik\thanks{Contact authors}\\
  Tel Aviv University\\ Google Research, Israel\\
  \texttt{eliavbuh@gmail.com}
  \And
  Edith Cohen$^*$\\
  Google Research USA\\
  Tel Aviv University\\
  \texttt{edith@cohenwang.com}\\
  \And
  Avinatan Hassidim\\
  Google Research, Israel\\
  \And
  Yossi Matias\\
  Google Research, Israel
  }
}
\ignore{
\author{
  Eliav Buchnik\thanks{Contact authors}\\
  Tel Aviv University\\ Google Research, Israel\\
  \texttt{eliavbuh@gmail.com}
  \and
  Edith Cohen$^*$\\
  Google Research USA\\
  Tel Aviv University\\
  \texttt{edith@cohenwang.com}\\
  \and
  Avinatan Hassidim\\
  Google Research, Israel\\
  \and
  Yossi Matias\\
  Google Research, Israel
  }}


\newcommand{\fix}{\marginpar{FIX}}
\newcommand{\new}{\marginpar{NEW}}

\begin{abstract}
  Optimization of machine learning models is commonly performed through stochastic gradient updates on 
  randomly ordered training examples. This practice means that
  sub-epochs comprise of independent
  random samples of the training data that may not preserve informative
  structure present in the full data.  We hypothesize that the
  training can be more effective with {\em self-similar}  arrangements that potentially allow each epoch to
  provide benefits of multiple ones.
We study this for  ``matrix factorization'' -- the common task of learning metric
embeddings of entities such as queries, videos,  or words from 
example pairwise associations.
We construct arrangements that 
preserve the weighted Jaccard similarities of 
rows and columns and experimentally observe
training acceleration of 3\%-37\% on synthetic and
recommendation datasets.  Principled arrangements of training examples emerge
as a novel and potentially powerful enhancement to SGD that merits further exploration.


\end{abstract}
 
 \section{Introduction}
 
 Large scale machine learning models are commonly trained on data of the
 form of associations between entities.
The goals are to obtain a model that generalizes (supports inference of
associations not present in the input data) or obtain metric representations
of entities that capture their associations and can be used 
in downstream tasks. 
Examples of such data are images and their labels \cite{imagenet:CVPR2009}, 
 similar image pairs \cite{facenet:cvpr2015},
 text documents and occurring terms
 \cite{BERRY95,DUMAIS95,DEERWESTER90},  users and watched or rated videos
\cite{Koren:IEEE2009},  pairs of co-occurring words
\cite{Mikolov:NIPS13}, and pairs of nodes in a
graph that co-occur in short random walks
\cite{deepwalk:KDD2014}. 
This setting is fairly broad: entities can be of one or multiple
types and example associations used for training can be raw  or
preprocessed by
reweighing raw frequencies
\cite{SaltonBuckley1988,DEERWESTER90,Mikolov:NIPS13,PenningtonSM14:EMNLP2014}
or adding negative examples when the raw data includes only positive
ones \cite{Koren:IEEE2009,Mikolov:NIPS13}.

The optimization objective of the model parameters has the general form  of 
a sum over example associations.  In modern 
 applications the number of terms can be huge and
the de facto method is stochastic gradient descent (SGD)
\cite{SGDbook:1971,Koren:kdd2008,Salakhutdinov:ICML2007,Gemulla:KDD2011, 
  Mikolov:NIPS13}. With SGD, gradient updates computed over
stochastically-selected minibatches of training examples are 
performed over multiple epochs.
The extensive practice and theory of SGD optimization introduced numerous tunable hyperparameters and 
extensions aimed to improve quality and efficiency.  These include
tuning the learning rate also per-parameter 
\cite{Duchi_ADAGRAD:JMLR2011} and altering the 
distribution of training examples by
gradient magnitudes
\cite{AlainBengioSGD:2015,ZhaoZhang:ICML2015}, cluster structure
\cite{clusterSGD:AISTATS2017}, and diversity criteria
\cite{DPPSGD:UAI2017}.  Another popular method, Curriculum Learning
\cite{curriculumlearning:ICML2009}, alters the distribution of examples (from easy to hard) in
the course of training.  
In this work we motivate and explore the potential benefits of tuning the
{\em arrangement} of training examples -- a  novel optimization method that
is combinable with those mentioned above.

The baseline practice, which we refer to as {\em independent}
arrangements, forms the training order by drawing
examples (or 
minibatches of examples) independently at random according to prescribed probabilities
(that may correspond to the frequency of the association in the training set). 
This i.i.d practice is supported by optimization theory as it bounds the
variance of stochastic gradient updates. 
Independent arrangements, however, have a potential drawback:
Informative structure that is present in the
full data is not recoverable from the independent random samples that comprise
sub-epochs. 
We demonstrate this point for a common ``matrix factorization'' task
where the data is pairwise
associations such as users watches of videos for
recommendation tasks \cite{Koren:IEEE2009} and
term co-occurrences
for word embeddings \cite{Mikolov:NIPS13}.  The rows and columns
represent entities
and entries are example associations (viewer,video)  or (word,word).
 \begin{figure}[h]
   \begin{center}
     \begin{tabular}{cc}
 \includegraphics[width=0.2\textwidth]{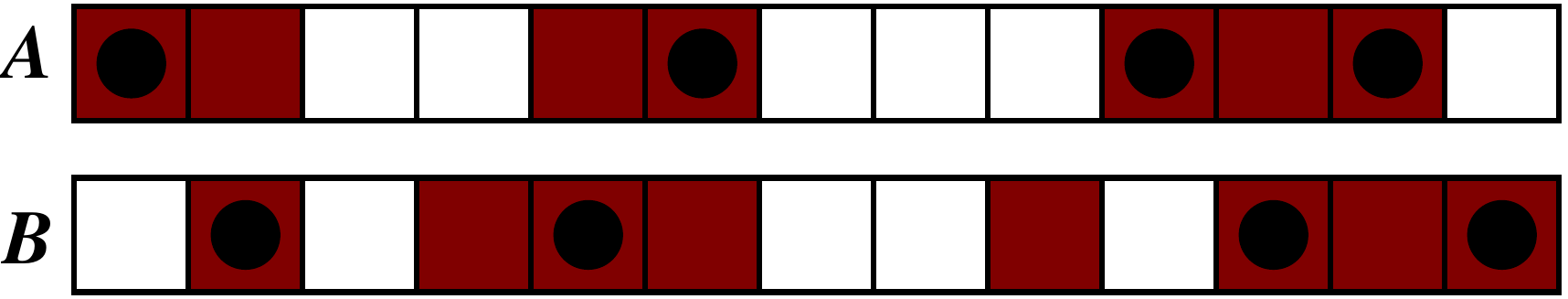} &
                                                      \includegraphics[width=0.2\textwidth]{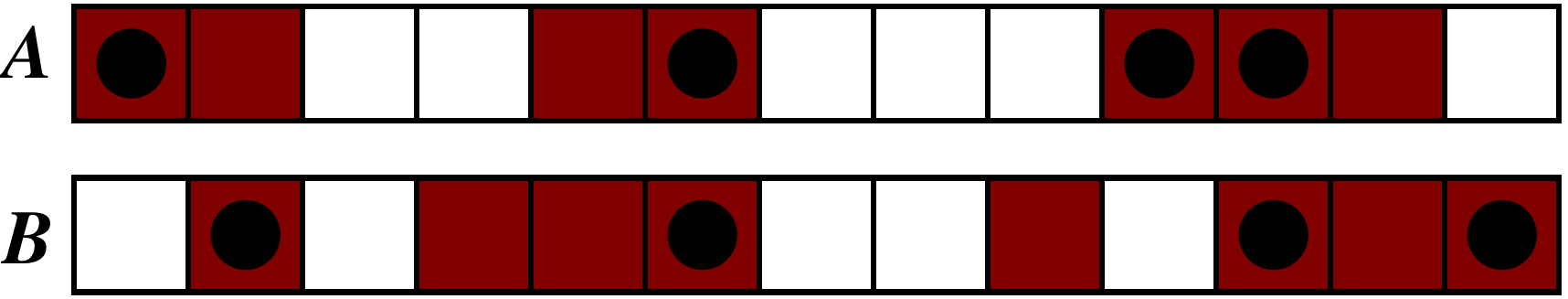}\\
       Independent arrangement & Coordinated arrangement
   \end{tabular}                             
\end{center}  
\caption{Two rows ($A$ and $B$), where entries in red are positive and
  equal.  The Jaccard similarity is $J(A,B)=1/2$ (5 common
 positive columns out of 10 that are positive for at least one).  A sub-epoch includes a random sample from each row. With
  independent arrangements the 4 samples from each row are unlikely to align on the
  common columns resulting in empirical Jaccard similarity of $0$.  With
our coordinated arrangement the samples align and
the empirical Jaccard similarity is $1/2$.}
\label{indcoo:fig}
\end{figure}
In such data  (see Figure~\ref{indcoo:fig}), the
similarity of two rows (or columns)
is indicative of the similarity of the corresponding entities that our
model is out to capture. For example, two
videos with overlapping sets of viewers are likely to be similar.
While the target similarity we seek is typically more complex and in
particular reflects higher order relations (sets of similar but not
overlapping viewers), this ``first order'' similarity is nonetheless
indicative.
In a random sample of matrix entries, however, two similar
rows will have dissimilar samples: The expected
empirical weighted
Jaccard similarity on the sample is much lower than the respective
similarity in the data, and this happens even at the extreme where the
sample is a large fraction of the full data (say half an epoch) and we are considering two
identical rows (!).
For our training this means that sub-epochs rapidly lose this important information that is present
in the full dataset.  
We hypothesize that this may impact the effectiveness of
training:  An arrangement that is more ``self-similar'' in the
sense that information is preserved to a higher extent in sub-epochs 
may allow a single epoch to provide benefits of multiple ones and for the training to converge
faster. 

We approach this by designing
{\em coordinated} arrangements that preserve in
expectation in sub-epochs the weighted Jaccard similarities of
rows and columns.
Our design is inspired by the theory of coordinated weighted
sampling~\cite{KishScott1971,BrEaJo:1972,multiw:VLDB2009,sdiff:KDD2014} which are
related to MinHash sketches~\cite{ECohen6f,Broder:CPM00}. In
coordinated sampling the 
goal is to select samples of entries of vectors that can provide more
accurate estimates of relations between the vectors than independent samples.  In our
application here we will construct arrangements of examples where
subsequences look like coordinated samples.

We specify our coordinated arrangements by a distribution on
randomized subsets  of example associations which we refer to as {\em
  microbatches}.   Our training sequence consists of 
independent microbatches and thus retains the traditional
advantages of i.i.d training  at the coarser microbatch level.
Note that our microbatches are designed so that the probability that each example is placed in a microbatch is equal 
to its prespecified baseline marginal probability.  Therefore, the only difference
between coordinated and independent arrangements is in the ordering.


In some applications or training regimes smaller 
microbatches, which allow for more independence,  can be more effective.
Our coordinated microbatches are optimized in size to preserve expected
similarities.  Microbatch sizes can be naively decreased by random
partitions  --  but this break down the similarity approximation and
more so for similar pairs, which are exactly the ones for which 
the benefits of preserving similarities are larger.
We show how Locality Sensitive Hashing (LSH) maps can be used to
decrease microbatch sizes
in a targeted way that  compromises more the  less
similar pairs. 
We explore LSH maps that leverage coarse
available proxies of entity similarity:  The weighted Jaccard similarity of
the row and column vectors or angular similarity of an embedding obtained by a weaker model.

We design efficient generators of coordinated and LSH-refined
microbatches and study the effectiveness of different arrangements through experiments on
synthetic stochastic block matrices and on recommendation data sets.
We use the popular
Skip Gram with Negative Sampling (SGNS) loss objective \cite{Mikolov:NIPS13}.
We observe consistent
training gain of 12-37\% on blocks and of 3\%-12\% on our real data sets
when using coordinated arrangements.



The paper is organized as follows.
Section~\ref{prelim:sec} presents necessary background on the loss
objective we use in our experiments 
and working with minibatches with one-sided gradient
updates and selection of negative examples. 
In Section~\ref{arrange:sec} we present our
coordinated microbatches and in Section~\ref{cooprop:sec} we establish
their properties. Our LSH refinements are presented in
Section~\ref{LSH:sec} and our experimental results are reported in
Sections~\ref{experiments:sec} and~\ref{selection:sec}.
 We conclude in Section~\ref{conclu:sec}.
 



\ignore{
Latent semantic analysis of text documents 
based on Singular value decomposition (SVD) that essentially reduce 
sparse large word occurrence vectors to dense embedding vectors 
\cite{BERRY95,DUMAIS95,DEERWESTER90}.  
 Netflix prize for developing recommendation engine of movies 
 based on watch history.    The winning method ~\cite{Koren:IEEE2009}

 More to mention:

Alternating minimization \cite{AltMin:1984}

 Bayesian matrix factorization using MCMC \cite{AhnWelling:KDD2015}.
}


\section{Preliminaries} \label{prelim:sec}

Our data has the form of associations
between a {\em focus} entity from a set $F$ and a {\em context} entity
from a set $C$.  The focus and context entities
can be of different types (users and videos) or two roles of the
same type or even of the same set (as in word embeddings).
We use $\kappa_{ij}$ as the association strength between focus $i$ and
context $j$.  In practice, the association strength can be derived from
frequencies in the raw data or from an associated value (for example,
numeric rating or watch time).

An embedding is a set of vectors 
$\vecf_i,\vecc_j \in \Re^d$ that is trained to minimize a loss
objective that encourages $\vecf_i$ and $\vecc_j$ to be ``closer''
when $\kappa_{ij}$ is larger.
Examples of positive associations $(i,j)$ are drawn with probability proportional to $\kappa_{ij}$.
Random associations are then used as negative examples
\cite{HuKorenV:2008} that provide an ``antigravity'' 
effect that prevents all embeddings from collapsing into the same vector. 
The weight
\begin{equation} \label{negweights:eq}
n_{ij} := \lambda \|\kappa_{\cdot j} \|_1   \|\kappa_{i
  \cdot}\|_1 / \|\kappa\|_1
\end{equation}
of a negative example $(i,j)$ is proportional to
the product of its column sum $\|\kappa_{\cdot j} \|_1$  by its row sum
$\|\kappa_{i  \cdot}\|_1$.  The hyperparameter $\lambda$ specifies a
ratio of negative to positive examples.


Our design applies to objectives of the general form
\begin{equation}\label{genform:eq}
 L :=\sum_{ij} \kappa_{ij} L_+(\vecf_i,\vecc_j) +  \sum_{ij }
 n_{ij} L_{-}(\vecf_i,\vecc_j)
 \end{equation}
  and can also accomodate hidden parameters as in
  \cite{Bromley:NIPS1994,ChopraHL:CVPR2005}.
For concreteness, we focus here on Skip Gram with Negative Sampling (SGNS) 
 \cite{Mikolov:NIPS13}. 
The SGNS objective is designed to maximize the log likelihood of these examples.
The probability of positive and negative examples are respectively
modeled using
\begin{align*}
 p_{ij} &=&\sigma(\vecf_i\cdot \vecc_j)= \frac{1}{1+\exp(-\vecf_i
  \cdot \vecc_j)}\ \\
   1-p_{ij} &=& \sigma(-\vecf_i\cdot \vecc_j)= \frac{1}{1+\exp(\vecf_i
  \cdot \vecc_j)}\ .
\end{align*}
The likelihood function, which we seek to maximize,  can then be expressed as
$\Prod_{ij} p_{ij}^{\kappa_{ij}} \Prod_{ij}
  (1-p_{ij})^{n_{ij}}\ .$
We equivalently can minimize the negated log likelihood that turns the
objective into a sum of the form \eqref{genform:eq}:
  \[ L:= -\sum_{ij} \kappa_{ij} \log p_{ij}-  \sum_{ij } n_{ij} \log(1-p_{ij})\ .\]
(using $L_+(\vecf_i,\vecc_j)  := -\log \sigma(\vecf_i\cdot\vecc_j)$ and
$L_-(\vecf_i,\vecc_j) := -\log \sigma(-\vecf_i\cdot \vecc_{j})$.)

The optimization is performed by
random initialization of the embedding vectors followed by stochastic gradient updates.
The stochastic gradients are computed for minibatches of examples that
include $b$ positive examples, where $(i,j)$ appears with frequency 
$\kappa_{ij}/\|\kappa\|_1$ and a set of $b\lambda$ negative examples.

\subsection{One-sided updates}
We work with {\em one-sided} updates, where each minibatch
updates only its focus or only its context embedding vectors, and  accordingly say that minibatches are {\em designated} for focus or context updates.
One-sided updates are used with alternating minimization
\cite{AltMin:1984} and decomposition-coordination
approaches \cite{Cohen:JOPT1980}.  For our purposes, one-sided updates facilitate our
coordinated arrangements (intuitively, because we need to separately
preserve column and row similarities)
and also  allow for precise minibatch-level matching of each
positive update of a parameter with a corresponding set of negative
updates as a means to control variance.

Our minibatches are constructed from a set $P$ of $b$ positive
examples and matched negatives.
Our marginal probabilities of positive and negative examples (see
Eq.~\ref{negweights:eq}) are equivalent to pairing
each positive example $(i,j)$ (with marginal probability $\kappa_{ij}/ \|\kappa\|_1$) with (i) $\lambda$ negative 
examples of the form $(i,j')$ where $j'$ is a random context entities 
(selected proportionally to the column sum $\|\kappa_{\cdot j'}\|/\|\kappa\|_1$ and (ii) $\lambda$
negative examples of the form $(i',j)$ where $i'$ are 
random focus entities $i'$ (selected proportionally to their row sums 
$\|\kappa_{i' \cdot }\|_1/\|\kappa\|_1$).
With one-sided updates, we pair each positive example $(i,j)\in P$ with $\lambda$
negative examples selected according to the respective designation.
To form a focus-updating minibatch, we generate a random set of
$\lambda$ context vectors $C'$. For each positive example $(i,j)\in P$ we generate $\lambda$ negative examples
$(i,j')$ for $j'\in C'$.  The focus embedding $\vecf_i$ is updated to
be closer to $\vecc_j$ but at the same time repealed (in expectation)
from $C'$ context vectors. With learning rate $\eta$, the combined
update to $\vecf_i$ due to positive example $(i,j)$ and matched
negatives is
\begin{equation*}
 \Delta \vecf_i = - \eta
\nabla_{\vecf_i}\left( L_+(\vecf_i , \vecc_j) + \sum_{j'\in
    C'} L_{-}(\vecf_i, \vecc_{j'}) \right)\ .
\end{equation*}
Symmetrically,  to form a context-updating minibatch we draw a random set
of focus vectors $F'$ and generate respective negative examples. Each positive
example $(i,j)\in P$ yields an
update of context vector $\vecc_j$ by
{\small $ \Delta \vecc_j = - \eta \nabla_{\vecc_j}\left(L_+(\vecf_i,\vecc_j) +
    \sum_{i'\in F'} L_-(\vecf_{i'},\vecc_{j}) \right).$}
All  updates are combined and applied at the end of the minibatch.

\section{Arrangement Schemes} \label{arrange:sec}

Arrangement schemes determine
how examples are organized.
At the core of each scheme is a distribution
$\mathcal{B}$ over subsets of positive
examples which we call {\em microbatches}.  Our microbatch distributions have the property
that the marginal probability of each example $(i,j)$ is always equal
to $\kappa_{ij}/\|\kappa\|_1$ but subset probabilities vary across
schemes. Moreover, 
within a scheme we may have
different distributions
$\mathcal{B}_f$ for focus and $\mathcal{B}_c$ for context
designations.

Minibatches formation for focus updates is specified in
Algorithm~\ref{makeMB:alg}  (the construction for context updates is symmetric).  The input is a microbatch distribution $\mathcal{B}_f$, minibatch size parameter $b$, and a parameter $\lambda$ that determines the ratio of negative to positive training examples.
We draw independent microbatches until we have
a total of $b$ or more positive examples and then select negative
examples as described above.
When training,  we alternate between focus and context updating
minibatches to maintain balance between the total number of examples
processed with each designation.
\begin{algorithm2e}[h]
  {\small
    \caption{Minibatch construction (Focus updates) \label{makeMB:alg}}
\DontPrintSemicolon  
\KwIn{$\mathcal{B}_f$, $b$, $\lambda$
\tcp*[h]{Microbatch distribution, size, negative sampling}}
$P,N \gets \emptyset$\;
\lRepeat{$|P|\geq b$}{$X\sim \mathcal{B}_f$; $P\gets P\cup X$\;}
$C' \gets$ $\lambda$ contexts selected iid by column weights\;
\ForEach{example pair $(i,j)\in P$}{
  \ForEach{$j'\in C'$}{$N\gets N\cup \{(i,j')\}$}
  }}
\Return{$P\cup N$}  
\end{algorithm2e}

The baseline independent arrangement method (\ind) can be placed in
this framework using microbatches that consist
of a single positive example $(i,j)$ selected with probability
$\kappa_{ij}/\|\kappa\|_1$ (see Algorithm~\ref{MBIndGlobal:alg}).
Our coordinated microbatches (\coo) have different distributions
for focus and context updates.
Algorithm~\ref{MBCooGlobal:alg} generates focus microbatches
(the generator for  context designation is symmetric).
These microbatches have the form of a set of positive examples with a
shared context.  In the instructive special case of $\kappa$ with all-equal positive 
entries focus microbatches include all positive entries in some column
and context microbatches include all positive entries in a raw.

We preprocess $\kappa$ so that we can efficiently draw $j$ with
probability $\|\kappa_{
  \cdot j}\|_\infty/\sum_h \|\kappa_{\cdot h}\|_\infty$ and construct an 
index that 
for context $j$ and value $T$ efficiently  returns  all entries $i$
with $\kappa_{ij}\geq T$.   The preprocessing is linear in the 
sparsity of $\kappa$ and with it the microbatch 
generator amounts to drawing a context $j$ 
(an$O(1)$ operation), $u\sim U[0,1]$ and then query the index with
$j$ and $T=u \|\kappa_{
  \cdot j}\|_\infty$.
The preprocessing cost for microbatch generation is often dominated
by the preprocessing done to generate $\kappa$ from raw data.

\begin{algorithm2e}[h]
  {\small 
\caption{\ind\ microbatches\label{MBIndGlobal:alg}}
\KwIn{$\kappa$}
Choose $(i,j)$ with probability 
$\kappa_{ij}/\|\kappa\|_1$\; 
\Return{$\{(i,j)\}$}
}
\end{algorithm2e}  
\begin{algorithm2e}[h]
{\small 
\DontPrintSemicolon 
\caption{\coo\  microbatches (Focus updates) \label{MBCooGlobal:alg}}
\KwIn{$\kappa$}
\tcp{Preprocessing:}
\ForEach{context $j$}{
$M_j \gets \max_i   \kappa_{ij}$ \tcp*[h]{Maximum entry for context 
  $j$}\; 
  Index column $j$ so that we can return for each $t\in (0,1]$, 
$P(j,t) := \{i \mid \kappa_{ij}\geq t M_j\}$.\;}
\tcp{Microbatch draw:}
Choose a context $j$ with probability $\frac{M_j}{\sum_h M_h}$\; 
Draw $u \sim U[0,1]$\; 
\Return{$\{(i,j) \mid i\in P(j, u)\}$}
}
\end{algorithm2e}
\section{Properties of \coo\ Arrangements} \label{cooprop:sec}

We establish that \coo\ arrangements produce the same marginal
distribution on training examples as the baseline \ind\ arrangements. 
We then highlight two properties of coordinated arrangements that are beneficial to accelerating convergence: 
A micro-level property that
makes gradient updates more effective by moving embedding vectors of similar entities 
closer and a macro-level property of
preserving expected similarity in sub-epochs.

\paragraph{Marginal distribution}
We show that the occurrence frequencies of examples $(i,j)$  in \coo\
microbatches (of either designation) is
$\propto \kappa_{ij}$.  
\begin{lemma}
The inclusion probability of a positive example $(i,j)$ in
a coordinated microbatch with focus designation (Algorithm~\ref{MBCooGlobal:alg}) is
$\kappa_{ij}/ \sum_h \|\kappa_{\cdot h}\|_{\infty}$, where
the notation $\|\kappa_{\cdot 
  h}\|_\infty$ is the maximum entry in column $h$.
Respectively, the inclusion probability of $(i,j)$ in
a microbatch with context designation is $\kappa_{ij}/\sum_h 
\|\kappa_{h\cdot}\|_\infty$, where $\|\kappa_{h\cdot}\|_\infty$ is the
maximum entry at row $h$. 
\end{lemma}
\begin{proof}
Consider focus updates (apply a symmetric argument for context updates). 
The example $(i,j)$ is selected when first 
context $j$ is selected, which happens with probability 
$\|\kappa_{\cdot j}\|_{\infty} / \sum_h \|\kappa_{\cdot h}\|_{\infty} $ and then we have 
$u \leq \kappa_{ij} / \|\kappa_{\cdot j}\|_{\infty} $ for independent $u\sim U[0,1]$, which happens with 
probability $\kappa_{ij} / \|\kappa_{\cdot j}\|_{\infty} $.  Combining,  the probability that 
$(i,j)$ is selected is the product of the 
probabilities of these two events which is 
$\kappa_{ij}/\sum_h \|\kappa_{\cdot h}\|_{\infty} $. 
\end{proof}
Our arrangements consist of both focus and context microbatches
that balance the total number of examples in each designation.
Therefore, each example appears in the same frequency with each designation.

\paragraph{Alignment of corresponding examples}
Our \coo\ microbatches maximize the co-placement probability of 
{\em corresponding} pairs of examples (examples with shared context or
focus) (for arrangements that respect the marginal distribution):
\begin{lemma} \label{sync:lemma}
  If a focus-designation  \coo\ microbatch includes an example $(i,j)$ and 
$\kappa_{ij} \leq \kappa_{i'j}$ then it also includes the example
$(i',j)$.  Symmetrically with context-designation, $(i,j)$ being
included and
$\kappa_{ij} \leq \kappa_{ij'}$ implies that $(i,j')$ is also included.
\end{lemma}
We argue that this property provides some
implicit regularization that encourages embeddings of entities 
with corresponding examples to be closer.
In particular, an aligned pair of updates on corresponding examples
tends to pull the embedding vectors closer.
 A pair of entities with higher Jaccard 
similarity has a larger fraction of corresponding examples and benefit more from
alignment.
Interestingly, the benefit is there even when embedding vectors are
random, as is the case early in training.
In particular, the SGNS loss term for a positive example is 
$L_+(\vecf,\vecc) = -\log \sigma(\vecf,\vecc) = -\log\left(
  \frac{1}{1+\exp(-\vecf\cdot\vecc)}\right).$
The gradient with respect to $\vecf$  is 
$\nabla_{\vecf}(L_+(\vecf,\vecc))=-\vecc 
\frac{1}{1+\exp(\vecf\cdot\vecc)}$ and the respective update of
$\vecf' \gets \vecf + \eta \frac{1}{1+\exp(\vecf\cdot\vecc)} \vecc$
clearly increases 
$\cossim(\vecf,\vecc)$.
\begin{wrapfigure}{l}{0.2\textwidth}
  \begin{center}
    \includegraphics[width=0.2\textwidth]{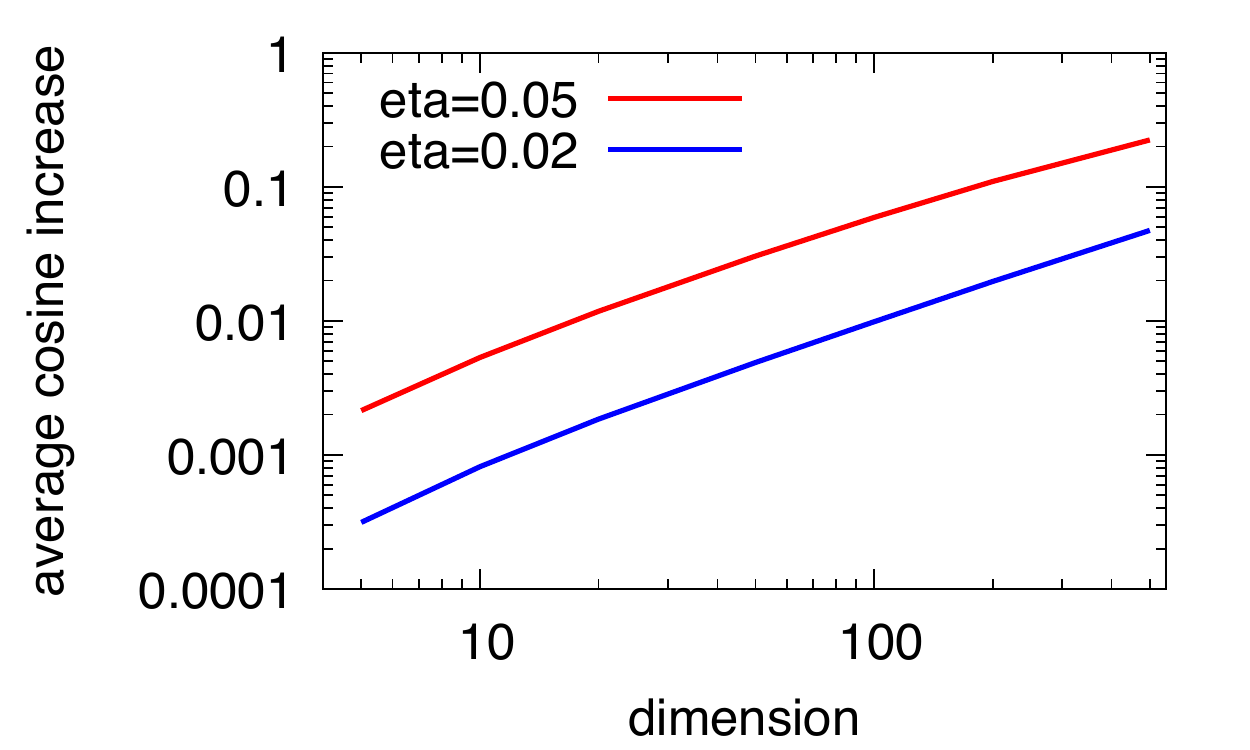}
\end{center}  
\caption{{\small Expected increase in $\cossim(\vecf_1, \vecf_2)$ as a
    function of dimension for $\vecf_i \sim  
    \mathcal{N}^d$  after gradient update to same random context $\vecc \sim \mathcal{N}^d$}}
\label{gradmove:fig}
\end{wrapfigure}
Consider two focus entities $1,2$ and corresponding examples
$(1,j)$ and $(2,j)$. When the two examples are in the same
focus-updating minibatch (where $\vecc_j$ is fixed)
both $\cossim(\vecf_1,\vecc)$ and 
$\cossim(\vecf_2,\vecc)$ increase but
a desirable side effect is that in expectation 
 $\cossim(\vecf_1,\vecf_2)$ increases as well. 
 The updates are aligned also with full gradients but not with
 \ind\  arrangements that on average place corresponding examples half an epoch
 apart. 
 Figure~\ref{gradmove:fig} shows the
expected increase in cosine similarity 
$\E\left[\cossim(\vecf'_1,\vecf'_2) 
    - \cossim(\vecf_1,\vecf_2) \right]\ $ as a function of the dimension
for example learning rates 
$\eta=0.02,0.05$ when the vectors $\vecf_1$, $\vecf_2$, and $\vecc$ are 
independently drawn from a product distribution $\mathcal{N}(0,1)^d$
of independent Gaussians.

\paragraph{Preservation of Jaccard similarities}
 We establish that \coo\ arrangements preserve in
 expectation  Jaccard similarities of pairs of rows and columns.
 The weighted Jaccard similarity of two vectors $\vecv$ and $\vecu$
 is defined as
 \begin{equation} \label{Jaccard:eq}
   J(\vecv,\vecu) = \frac{\sum_i
     \min\{v_i,u_i\}}{\sum_i
     \max\{v_i,u_i\}}\ .
\end{equation}

 \begin{lemma}
   Consider a set of focus updating microbatches and let
   $X_{ij}$ be the random variable that is the multiplicity of example $(i,j)$.
Then for any two rows $i,i'$,  the expectation of the empirical weighted  Jaccard
   similarity on $X$ (when defined)  is equal to the weighted Jaccard similarity on
   $\kappa$:
     \begin{align*} 
\E\left[ J(X_{i\cdot}, X_{i'\cdot})  \mid  \sum_j
    \max\{X_{i',j},X_{i,j}\}  > 0 \right] 
=  J(\kappa_{i,\cdot}, \kappa_{i',\cdot})
  \end{align*}
\ignore{
  
  \begin{align*} 
&\E\left[ \frac{\sum_j \min\{X_{i',j},X_{i,j}\}}{ \sum_j \max\{X_{i',j},X_{i,j}\}} \mid  \sum_j
    \max\{X_{i',j},X_{i,j}\}  > 0 \right]\\
=& \frac{\sum_j
     \min\{\kappa_{i',j},\kappa_{i,j}\}}{\sum_j
     \max\{\kappa_{i',j},\kappa_{i,j}\} }
  \end{align*}}
  A symmetric claim holds for context updating microbatches.
\end{lemma}   
\begin{proof}
We consider a single microbatch and its contributions to the numerator and
  denominator of the empirical similarity  $J(X_{i\cdot}, X_{i'\cdot})$. 
From Lemma~\ref{sync:lemma}, the possible contributions are $(0,0)$,
$(0,1)$ or $(1,1)$.
Therefore, $J(X_{i\cdot}, X_{i'\cdot})$ (if defined) is simply the 
average of the contributions to the numerator
over microbatches that contributed to the denominator.  The
expectation of $J(X_{i\cdot}, X_{i'\cdot})$  (when defined)
is therefore equal to the probability of a contribution to the numerator
in a single microbatch given that there was a contribution to the
denominator.
If the shared context in the microbatch is $j$,
the probability of contribution to the  denominator is
   $\max\{\kappa_{i',j},\kappa_{i,j}\}/\|\kappa_{\cdot j}\|_\infty$
   and to the numerator is 
   $\min\{\kappa_{i',j},\kappa_{i,j}\}/\|\kappa_{\cdot j}\|_\infty$.
   The probability over the random draw of context $j$
 of a contribution to the denominator and numerator respectively is
   $\sum_j \max\{\kappa_{i',j},\kappa_{i,j}\}/\sum_j  \|\kappa_{\cdot 
     j}\|_\infty$ and 
   $\sum_j \min\{\kappa_{i',j},\kappa_{i,j}\}/\sum_j  \|\kappa_{\cdot 
     j}\|_\infty$.   Since a contribution to the numerator is made only
   if there was one to the denominator, the expectation we seek is the
ratio $J(\kappa_{i,\cdot}, \kappa_{i',\cdot})$.
  \end{proof}  

\section{Refinement using LSH Maps} \label{LSH:sec}
\ignore{
Placement of $(i,j)$ and $(i',j)$ in the same focus updating
microbatch results in pulling $\vecf_i$ and $\vecf_{i'}$ closer
together (see the micro-level property highlighted in Section~\ref{cosinegain:sec}).  This is helpful when
the entities $i$ and $i'$ are similar in that they have a close target
embeddings.  Otherwise, the update is anyhow countered by other updates
and the placement have undesirable effect that it increases the microbatch size and may increase variance of the stochastic gradients.  
In particular, since a large microbatch is processed by consecutive same-designation minibatches, it increases the effective minibatch size to microbatch size.
This suggests that it would be useful to tune the quality of
co-placements so as to decrease unhelpful ones while retaining
as many helpful ones as we can.
}
We provide methods to partition our \coo\ microbatches so that they
are smaller and of higher quality in the sense that a larger fraction
of corresponding example pairs are between entities with higher
similarity.  To do this we 
use locality sensitive hashing (LSH) to compute randomized maps of
entities to keys.   Each map is represented by  a vector $\mathbf{s}$ of 
keys for entities such that similar entities are more likely to obtain
the same  key. 
We use these maps to refine our basic microbatches by partitioning them according to keys.

Ideally, our LSH modules would correspond to the target similarity,
but this creates a chicken-and-egg problem.
Instead,  we can use  LSH modules that are available at the start of
training and provide some proxy of the target similarity. For example,
a partially trained or a weaker and cheaper to train model.
We consider two concrete LSH modules based on  {\em
  Jaccard} and on {\em Angular} LSH.
The modules generate maps for either focus or context entities which are
applied according to the microbatch designation.
We will specify the map generation for focus entities, as maps for context entities
can be symmetrically obtained by reversing roles.

Our Jaccard LSH module is outlined in Algorithm~\ref{lsh:alg}. 
The probability that two
focus entities $i$ and $i'$ are mapped to the same key (that is, $s_i
= s_{i'}$) is equal to the weighted Jaccard
similarity of their association vectors $\kappa_{i\cdot}$ and $\kappa_{i'\cdot}$
(For context updates the map is according to the vectors $\kappa_{\cdot j}$):
\begin{lemma}\cite{multiw:VLDB2009}
  $$\Pr[s_i = s_{i'}] = J(\kappa_{ij}, \kappa_{i'j})$$
\end{lemma}
\ignore{
\begin{proof}
Notice that the LSH map is according to coordinated bottom-1 samples of  association
vectors.
\end{proof}
The Jaccard similarity is a coarse proxy that is based only on
first order (1-hop) relations (similarity of the context vectors of two focus
entities) whereas we expect our target embedding
to capture higher order relations.
In a text corpus, for example, the words "doctor" and "physician" have
the same associations to words "ambulance" and "hospital", which
contribute to their first order similarity. The words
"bachelor party" and "bridesmaid" may not have strong first order
similarity but have strong higher order similarity: For example,  "bachelor party" is strongly associated to
"groom", "groom" is strongly associated with "bride", and "bride" is
strongly associated with "bridesmaid." }

Our angular LSH module is outlined in Algorithm~\ref{lshcos:alg}.  Here we input an explicit
``coarse'' embedding
$\tilde{\vecf}_i, \tilde{\vecc}_j$ that we expect to be lower quality
proxy of our target one.
Each LSH map is obtained by drawing a random vector and then mapping each entity $i$ to
the sign of a projection of $\tilde{\vecf}_i$  on the random  vector.
The probability that two focus entities have the
same key depends on the angle between their coarse embedding vectors:
\begin{lemma}\cite{GoemansWilliamson:JACM1995}
  $$\Pr[s_i = s_{i'}] = 1- \frac{1}{\pi} \cos^{-1}  \cossim(\tilde{\vecf}_i , \tilde{\vecf}_{i'})\ ,$$  where
  $\cossim(\vecv,\vecu):= \frac{\vecv\cdot\vecu}{\|\vecv\|_2
    \|\vecu\|_2}$ is the cosine of the angle between the two vectors.
\end{lemma}

Multiple LSH maps can be applied to decrease microbatch sizes
and increase the similarity level of entities placed in the same
microbatch: With $r$ independent maps the
probability that two entities are microbatched together is
$\Pr[s_i = s_{i'}]^r$ -- thus the probability decreases faster when
similarity is lower.
The number of LSH maps we apply can be set statically or  adaptively
to obtain microbatches that are at most a certain size (usually the
minibatch size).
For efficiency, we precompute a small number of LSH maps
in the preprocessing step and randomly draw from that set.  The
computation of each map is linear in the sparsity of $\kappa$.

\begin{algorithm2e}[h]
  {\small
\DontPrintSemicolon
\caption{Jaccard LSH map: Focus \label{lsh:alg}}
\ForEach(\tcp*[h]{i.i.d Exp distributed}){context $j$}{Draw $u_j \sim \Exp[1]$}
\ForEach(\tcp*[h]{assign LSH bucket key}){focus $i$}{$s_i \gets \arg\min_j u_j/\kappa_{ij}$}
\Return{$\mathbf{s}$}
  }
  \end{algorithm2e}
\begin{algorithm2e}[h]
  {\small
\DontPrintSemicolon
\caption{Angular LSH map: Focus \label{lshcos:alg}}
\KwIn{$\{\tilde{\vecf}_i\}$ \tcp*[h]{coarse $d$ dimensional embedding}}
Draw $r \sim S_d$ \tcp*[h]{Random vector from the unit sphere}\;
\ForEach(\tcp*[h]{assign LSH bucket key}){focus $i$}{$s_i \gets \text{\bf{sign}}(r \cdot \tilde{\vecf}_i)$}
\Return{$\mathbf{s}$}
  }
\end{algorithm2e}

\begin{figure*}[h]
\center 
\includegraphics[width=0.23\textwidth]{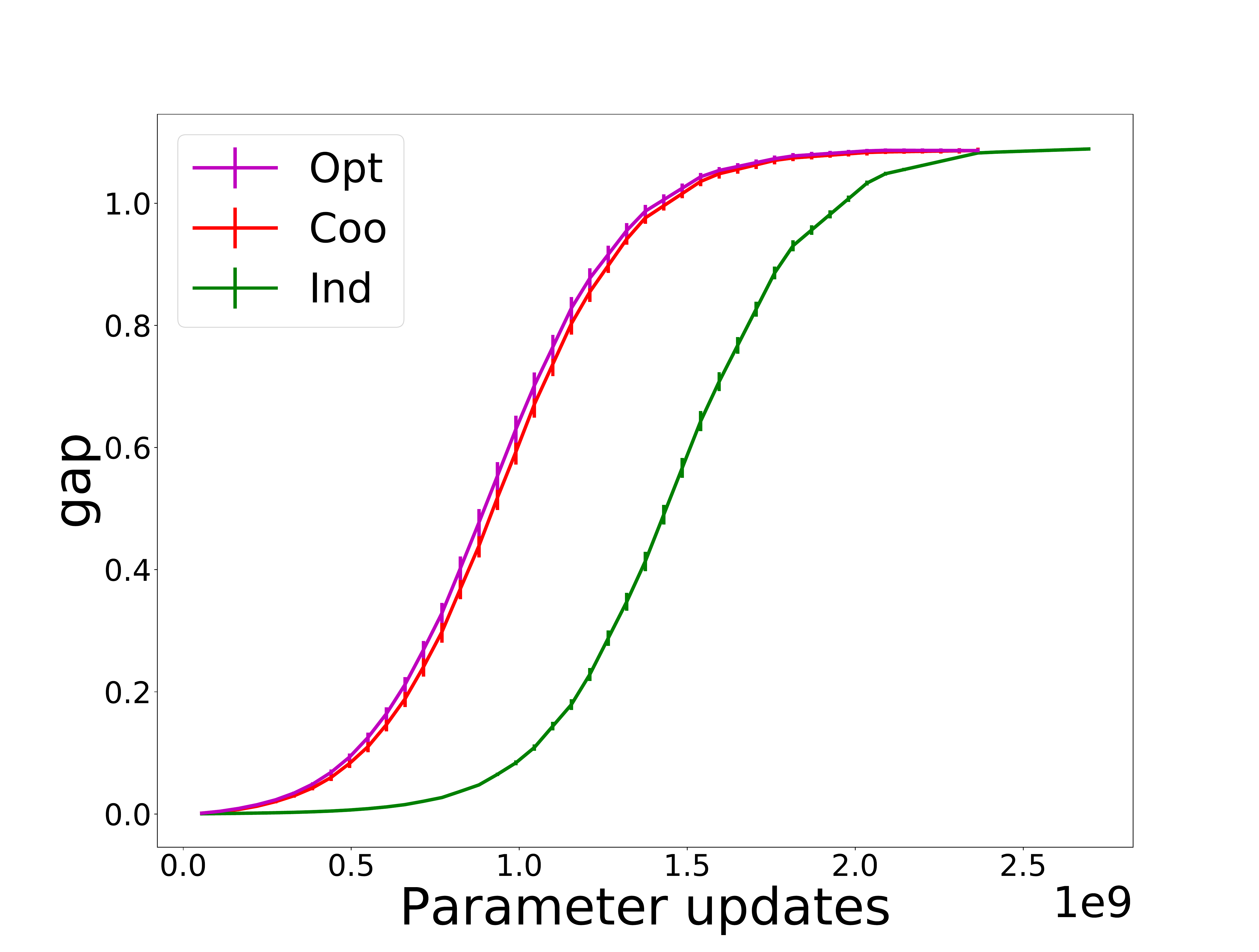}
\includegraphics[width=0.23\textwidth]{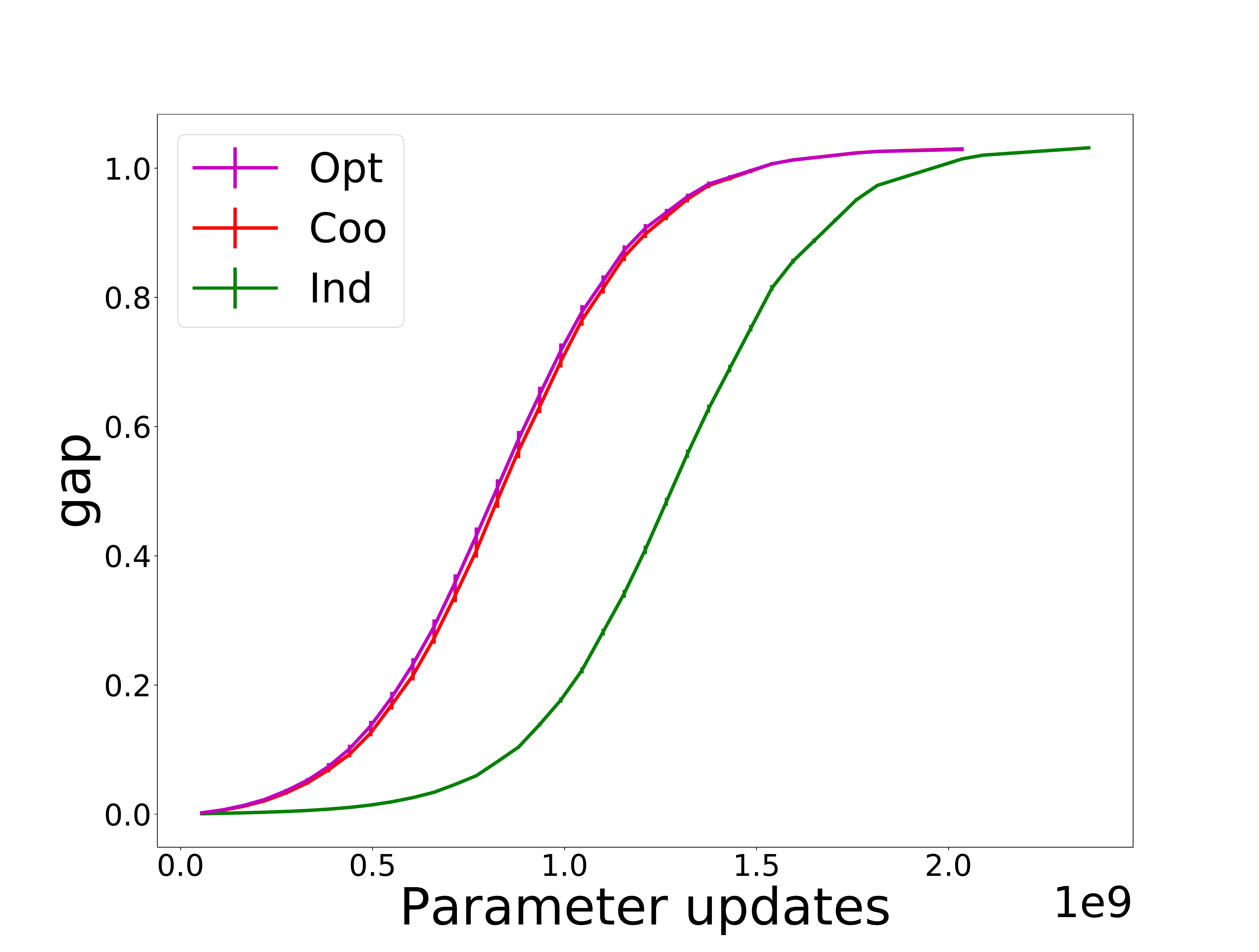}
\includegraphics[width=0.23\textwidth]{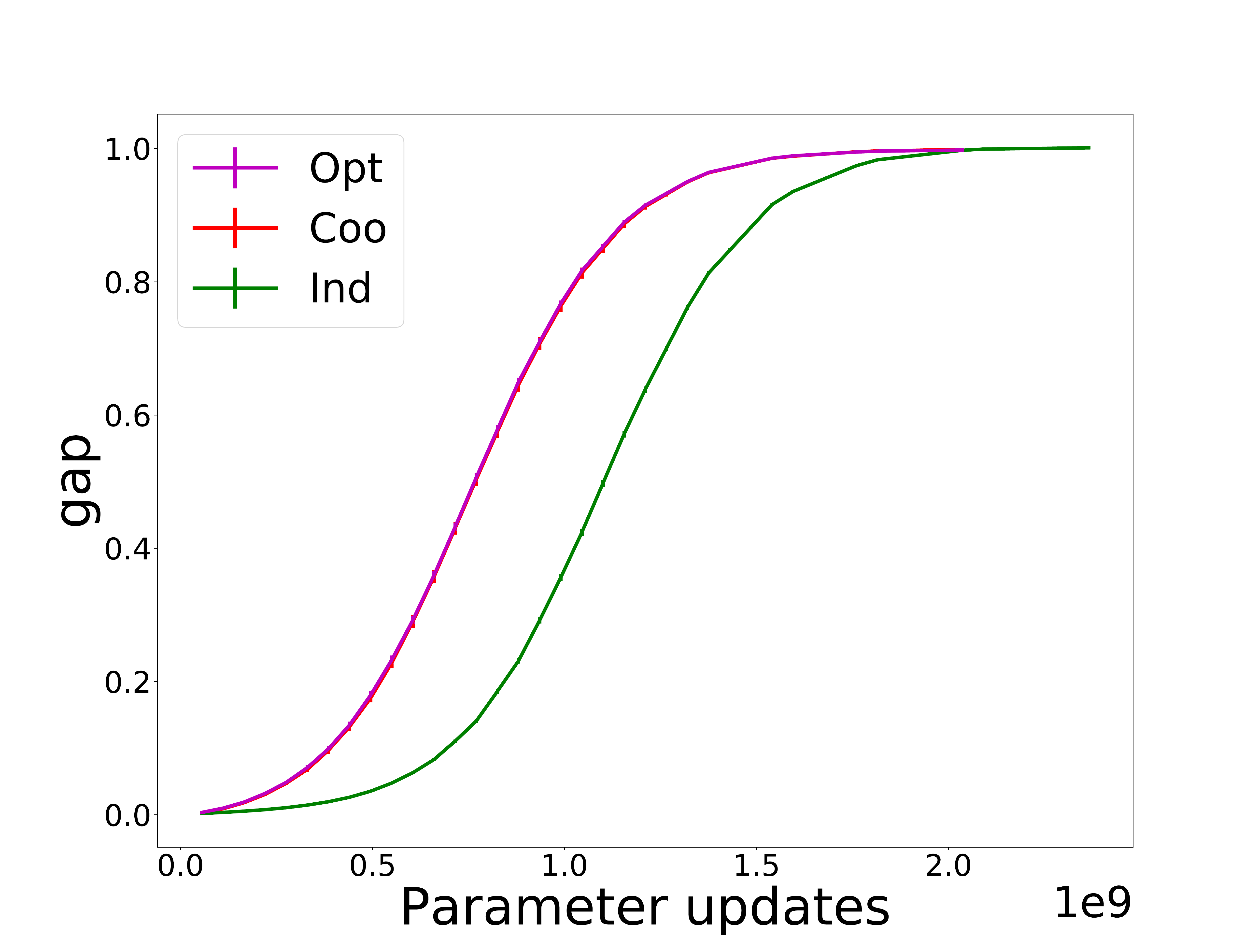}
\includegraphics[width=0.23\textwidth]{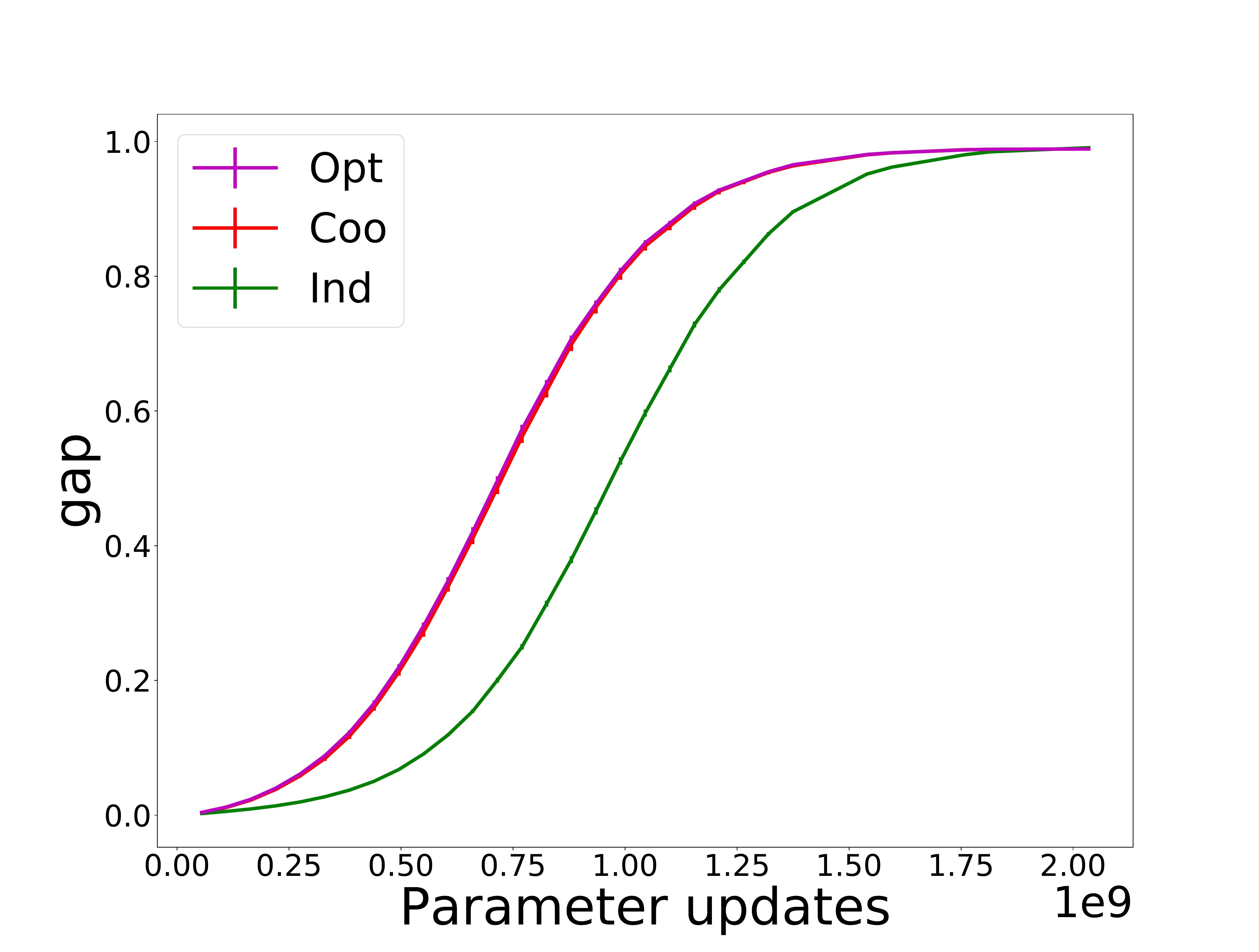}
\ignore{
\\
\includegraphics[width=0.23\textwidth]{early__block_f_c_gap__synthetic_10000_10000_10_0p70_10000000_batch_64_dim_50_agg.pdf}
\includegraphics[width=0.23\textwidth]{early__block_f_c_gap__synthetic_10000_10000_20_0p70_10000000_batch_64_dim_50_agg.pdf}
\includegraphics[width=0.23\textwidth]{early__block_f_c_gap__synthetic_10000_10000_50_0p70_10000000_batch_64_dim_50_agg.pdf}
\includegraphics[width=0.23\textwidth]{early__block_f_c_gap__synthetic_10000_10000_100_0p70_10000000_batch_64_dim_50_agg.pdf}
}
\includegraphics[width=0.23\textwidth]{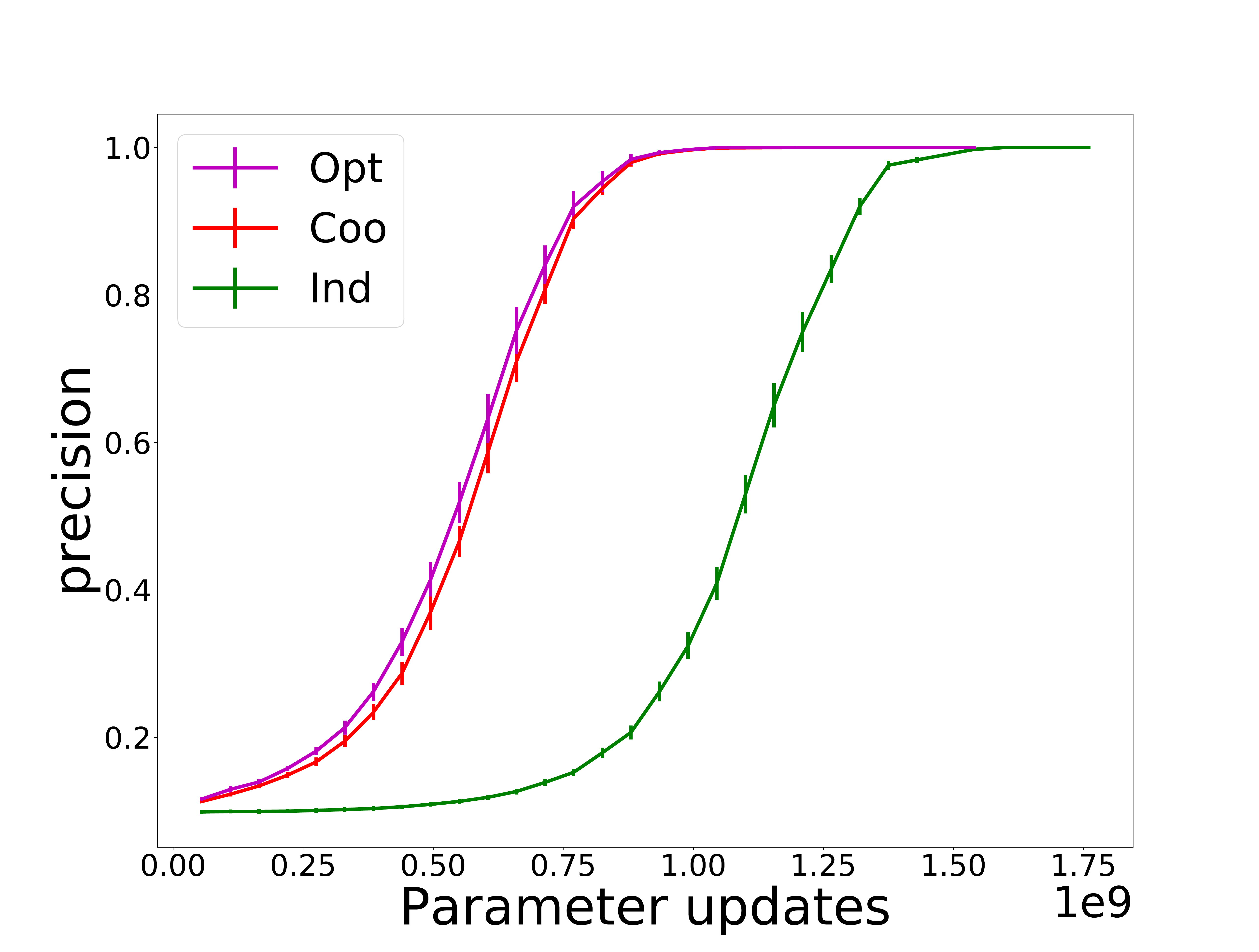}
\includegraphics[width=0.23\textwidth]{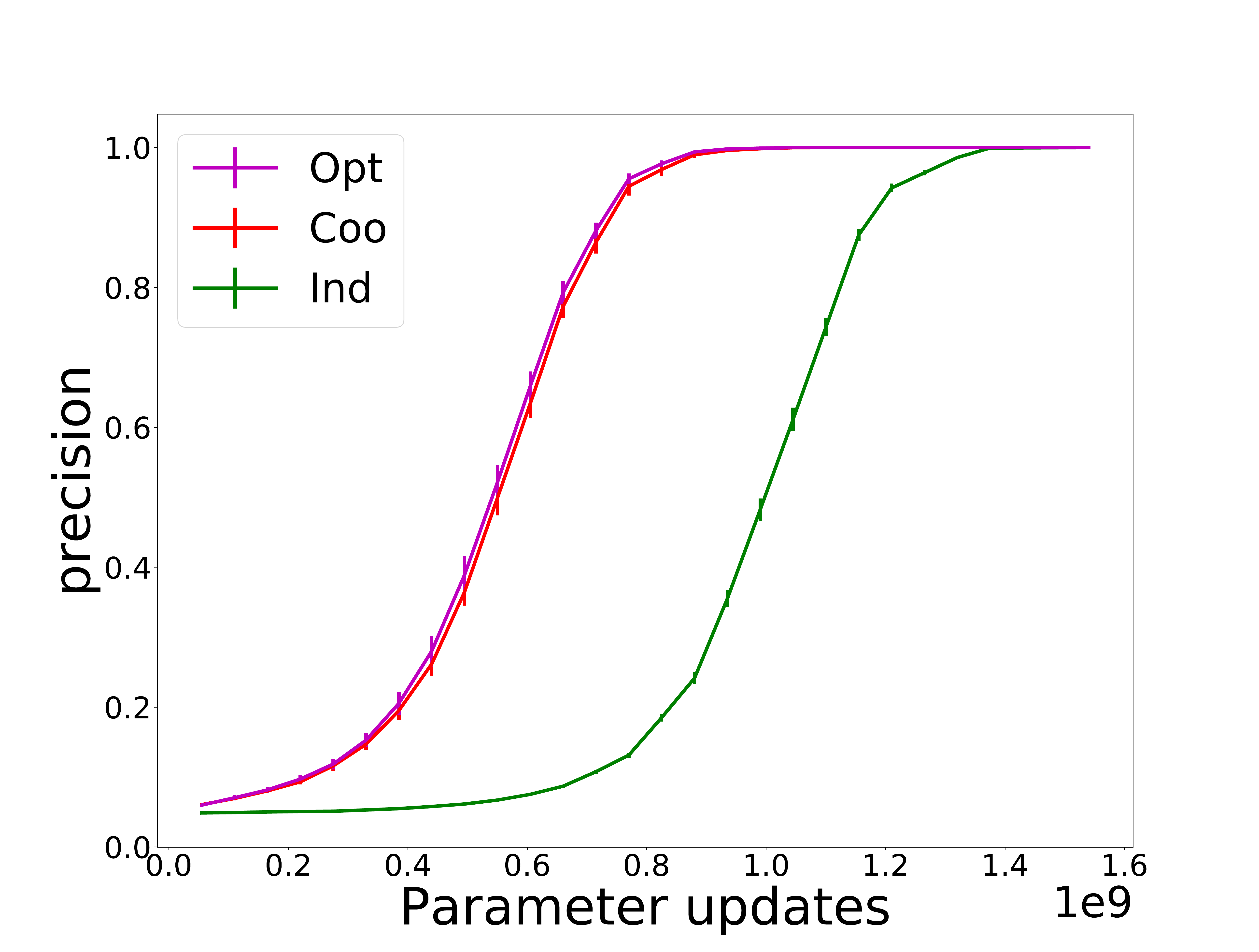}
\includegraphics[width=0.23\textwidth]{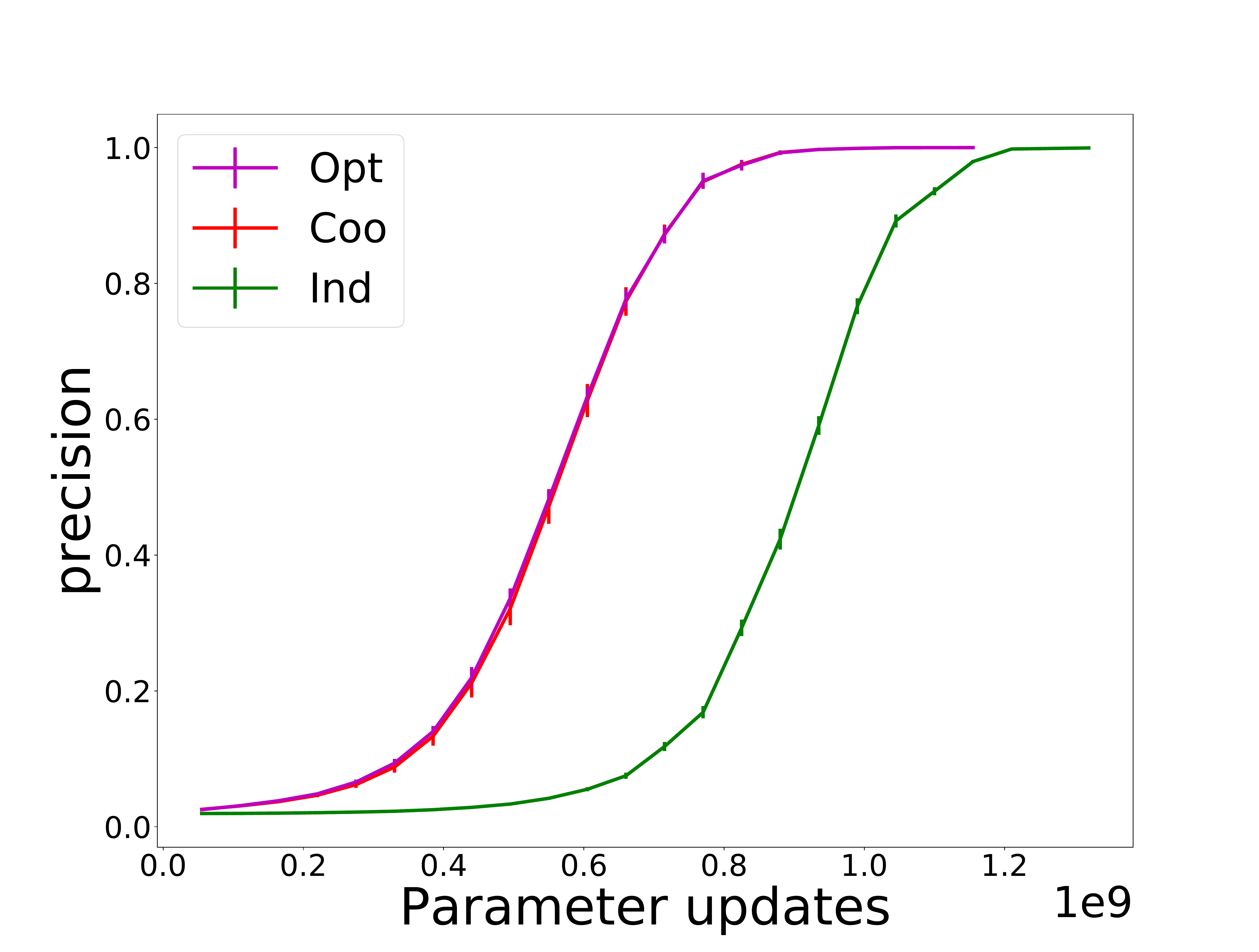}
\includegraphics[width=0.23\textwidth]{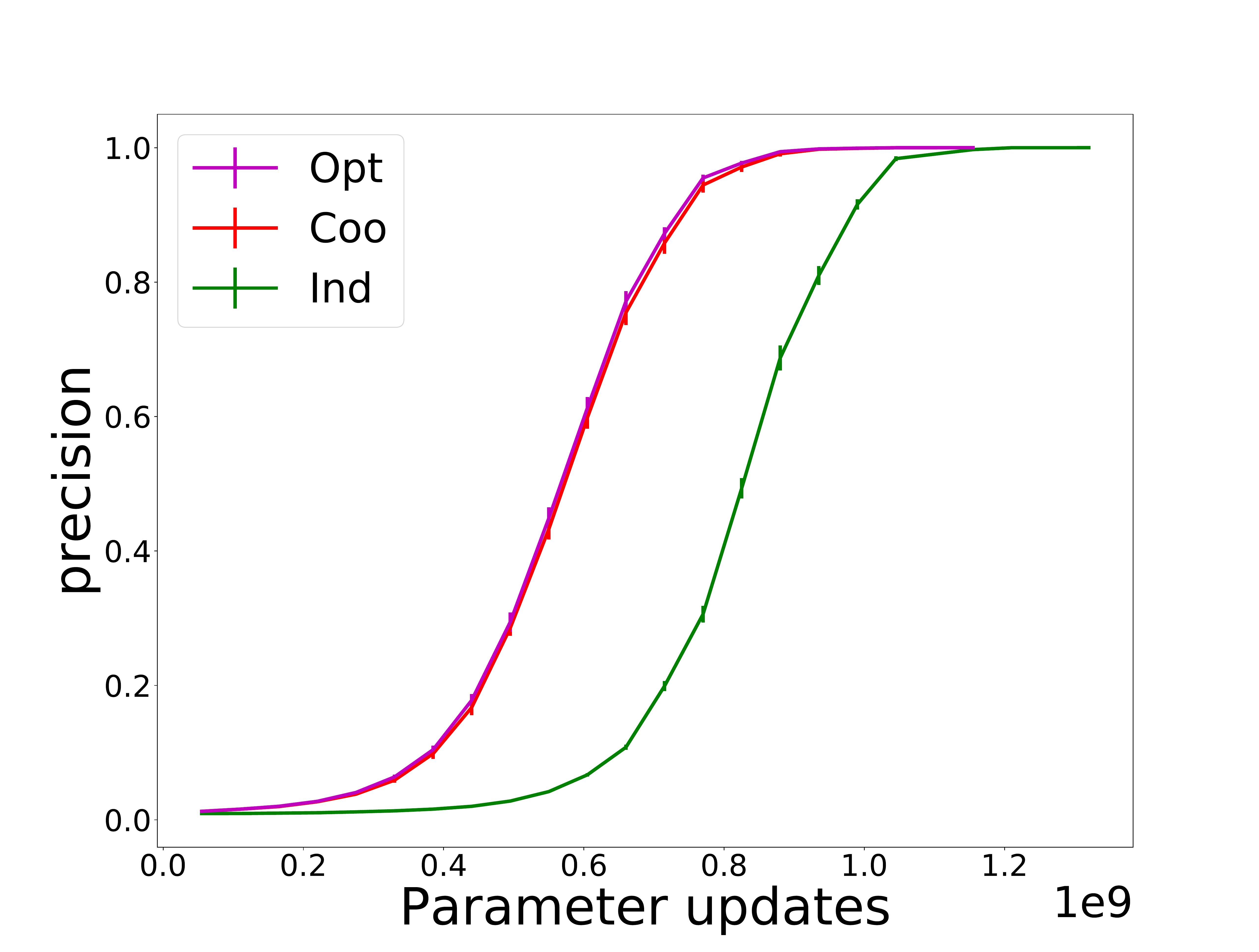}
\ignore{
\\
\includegraphics[width=0.23\textwidth]{early__block_f_c_precision_best_context_synthetic_10000_10000_10_0p70_10000000_batch_64_dim_50_agg.pdf}
\includegraphics[width=0.23\textwidth]{early__block_f_c_precision_best_context_synthetic_10000_10000_20_0p70_10000000_batch_64_dim_50_agg.pdf}
\includegraphics[width=0.23\textwidth]{early__block_f_c_precision_best_context_synthetic_10000_10000_50_0p70_10000000_batch_64_dim_50_agg.pdf}
\includegraphics[width=0.23\textwidth]{early__block_f_c_precision_best_context_synthetic_10000_10000_100_0p70_10000000_batch_64_dim_50_agg.pdf}
 }
\caption{Cosine gap (top) and Precision at $k=10$ (bottom) with \coo\ and \ind\ arrangements and
  \opt\ refinements in the course of training ($d=50$, $b=64$) for stochastic blocks 
  with $B\in\{10,20,50,100\}$ (left to right). }\label{blocks_gap_precision:fig}
\end{figure*}

\ignore{
\begin{figure}[t] 
\center
\caption{Precision at $k=10$
 with \coo\ and \ind\ arrangements and
  \opt\ ($d=50$, $b=64$) for stochastic blocks 
  with $B\in\{10,20,50,100\}$. }\label{blocks_precision:fig}
\end{figure}
}

\section{Arrangement Methods Experiments} \label{experiments:sec}

We trained embeddings with different arrangement methods:
The baseline independent arrangements (\ind) as in 
Algorithm~\ref{MBIndGlobal:alg},  coordinated arrangements (\coo)  as
in Algorithm~\ref{MBCooGlobal:alg}, and some  (\coolsh) arrangements
with Jaccard LSH.   We also experimented with 
tunable arrangements (\mix) that start with \coo\ (which reaps much of
its benefit earlier in training)
and switch to \coolsh\ or to \ind.   Finally, as another baseline we
also trained using the more standard \ind\ with two-sided updates.
The results were similar or slightly inferior to one-sided \onlyinproc{\ind.}\notinproc{\ind\ and
are reported in Appendix~\ref{twosided:sec}.}

\subsection{Stochastic blocks data} \label{stochasticblocks:sec}

We generated data sets using the stochastic blocks model
\cite{CondonKarp:2001}.   This synthetic data allowed us to explore the effectiveness
of different arrangement methods as we vary the number and size
of blocks.  The simplicity and symmetry of this data (parameters,
entities, and associations) allowed us to compare different
arrangement methods while factoring out 
optimizations and methods geared for asymmetric data such as per-parameter learning
rates or alterting the distribution of examples.  The blocks data
binary similarity allowed us to explore the
limits of LSH refinements by refining \coo\ microbatches according to
ground truth similarity (partitioning \coo\ microbatches by block
membership) (\opt).


The parameters for the generative model are the  
dimensions $n\times n$ of the matrix, the number of (equal size) blocks $B$, the  
number of interactions $r$,  and the in-block probability $p$.  
The rows and columns are partitioned to consecutive groups of $n/B$, 
where the $i$th part of rows and $i$th part of  
columns are considered to belong to the same block.  
 We generate the matrix by 
initializing the associations to be $\kappa_{ij}=0$. We then  
draw $r$ interactions independently as follows. 
We select a row index $i\in [n]$ uniformly at random. 
With probability $p$, we select (uniformly at random) a column $j\in [m]$
that is in the 
same block as $i$ and otherwise (with probability $1-p$) we select a 
column $j\in [n]$ that is outside the block of $i$.  We then increment $\kappa_{ij}$.   
 The final $\kappa_{ij}$ is the number of times $(i,j)$ was drawn.
 In our experiments we set $n=10^4$, $r=10^7$,$p=0.7$
 and $B\in \{10, 20, 50, 100\}$.

 \subsection{Implementation and methodology}
 We implemented our methods in Python using  the TensorFlow library 
\cite{tensorflow2015-whitepaper}. We used the word embedding implementation
of~\cite{Mikolov:NIPS13,word2vec:github}
except that we used our methods to specify
minibatches. The 
implementation included a default bias parameter associated with
context embeddings and we 
trained embeddings with and without the bias parameter.  The
relative performance of arrangement methods was the same but the overall 
performance was significantly better when the bias parameter was used.  We
therefore report results with bias parameters.\notinproc{ Details on training and prediction with bias are provided in Appendix~\ref{biasparam:sec}. }
We used a fixed learning rate to facilitate a more accurate comparison
of methods and trained with $\eta=0.005$ to
$\eta=0.15$. We observed similar relative performance and report results
with $\eta=0.02$.  We worked with minibatch size
parameter values $b\in\{4,64,256\}$ (recall that $b$ is the number of
positive examples and $\lambda=10$ negative examples are matched with
each positive example), and embeddings dimension
$d\in\{5,10,25,50,100\}$. \notinproc{ Hyperparameter sweeps
  of $d$ and $b$ and the learning rate $\eta$ are reported
  respectively in
  Appendix~\ref{dimanalysis:sec}, Appendix~\ref{minibatchsize:sec},
  and Appendix~\ref{learningrate:sec}.}

\subsection{Quality measures}
We use two measures of the quality of an embedding with respect to the blocks ground
truth.  The first is the  {\em cosine gap} which measures average quality and is defined as the difference in the average cosine similarity between positive examples and negative examples. We generate a
sampled set $T_{+}$ of same-block pairs $(i,j)$ as positive test
examples and a sampled set $T_{\_}$ of pairs that are not in the same
block as negative test examples and compute
\begin{equation} \label{cosinegap:eq}
\frac{1}{|T_+|}\sum_{(i,j)\in T_+} \cossim(\vecf_i,\vecc_j) -
\frac{1}{|T_-|}\sum_{(i,j)\in T_-} \cossim(\vecf_i,\vecc_j)\ .
\end{equation}
We expect a good embedding to have high cosine similarity for same-block pairs and 
low cosine similarity for out of block pairs.
The second measure we use, {\em precision at $k$}, is focused on the quality of the top predictions and is appropriate for recommendation tasks.  For each sampled representative entity we compute the entities with top $k$ cosine similarity and consider the average fraction of that set that are in the same block.

\begin{table}[h]
\scriptsize 
\center 
\begin{tabular}{cc|cc|cc|cc} 
\#blocks   &        & \multicolumn{2}{c}{$0.75$} & \multicolumn{2}{c}{$0.95$} & \multicolumn{2}{c}{$0.99$} \\
$B$ & peak&\%gain &$\times10^9$& \%gain& $\times 10^9$& \%gain& $\times 10^9$\\
  \hline
  \multicolumn{8}{c}{Cosine Gap Quality }\\
\hline
10&1.09&31.07&1.58&24.80&1.89&\textcolor{blue}{{\bf 19.06}}&2.17\\  

20&1.03&29.56&1.40&23.74&1.72&\textcolor{blue}{{\bf 18.13}}&1.97\\

50&1.00&25.52&1.22&20.46&1.55&\textcolor{blue}{{\bf 14.73}}&1.80\\

100&0.99&20.66&1.09&15.27&1.40&\textcolor{blue}{{\bf 12.11}}&1.64\\

%
  \hline
    \multicolumn{8}{c}{Precision at $k=10$ Quality }\\
\hline
  10&1.00&43.64&1.12&37.70&1.24&\textcolor{blue}{{\bf 37.41}}&1.38\\

20&1.00&40.80&1.03&36.16&1.14&\textcolor{blue}{{\bf 33.47}}&1.23\\  

  50&1.00&33.89&0.92&31.37&1.04&\textcolor{blue}{{\bf 25.35}}&1.09\\

  100&1.00&26.67&0.84&23.37&0.94&\textcolor{blue}{{\bf 20.00}}&1.02\\

\end{tabular}
\caption{Training gain of \coo\ with respect to \ind\
baseline for  stochastic 
   blocks ($d=50$, $b=64$).  Peak is maximum quality for \coo.  We report the training  for \ind\ to 
 reach 75\% , 95\%, and 99\% of peak with respective 
 percent reduction in training with \coo.}
\label{tab:coo_gap_precision_blocks}
\end{table}

\subsection{Stochastic blocks results} \label{stochasticblocksexp:sec}

Our results were consistent for different dimensions and minibatch
sizes and we report
representative results for $d=50$ and $b=64$ and for the methods
\coo, \ind, and the reference method \opt.
Results for the
cosine 
gap quality measure and for the precision (at $k=10$) are reported in 
Figure~\ref{blocks_gap_precision:fig} and
Table~\ref{tab:coo_gap_precision_blocks}.
The figures show the increase in
quality in the course of training for the different methods
\ignore{ and also zoom on the early part
of training}. The $x$-axis in these plots shows the {\em amount of training}
in terms of the {\em total number of gradient updates performed}. 
The tables report the amount of additional training needed for \ind\ to obtain
the performance of \coo.
We observe that across all block sizes $B$ and for the two
quality measures, \coo\ arrangement resulted in significantly faster
convergence than the  \ind\ arrangements.  The gains were larger with
larger blocks.  Much of the gain of \coo\ arrangements over \ind\ was realized
earlier in training and then maintained.
 The \opt\ method provided only very modest
 improvement over \coo\ and only for larger blocks. This improvement
bounds that possible by any \coolsh\ method on this data and indeed
\coolsh\ results (not shown) were between \coo\ and \opt.

\subsection{Recommendation data sets and results} \label{recexp:sec}

We performed experiments on two recommendation data sets, \movielens\
and \amazon.
The \movielens\ dataset \cite{movielen1m} contains $10^6$ reviews by  $6\times 
10^3$ users of $4\times 10^3$ movies.  
The \amazon\ dataset \cite{SNAP}
contains $5\times 10^5$ fine food reviews of $2.5 \times 10^5$ users 
on $7.5 \times 10^3$ food items. 
Provided review  scores were $[1\text{-}5]$ and we
preprocessed the matrix by taking $\kappa_{ij}$
to be $1$ for  review score that is at least $3$ and $0$
otherwise. We then reweighed entries in the \movielens\ dataset by dividing the value by the sum 
of its row and column to the power of $0.75$. 
This is standard processing that retains only positive ratings and 
reweighs to prevent domination of frequent entities.

We created a test set $T_+$ of positive examples by
sampling $20\%$ of the non zero entries 
with probabilities proportional to 
$\kappa_{ij}$.
The remaining examples were used for training.   
As negative test examples $T_-$ we used random zero entries.
We measured quality using the cosine  gap 
and precision at $k=10$ over users with at least 20 nonzero entries.
We used 5 random splits of the data to test
and training sets and 
10 runs per split.  \ignore{ We set hyperparameters and \mix\ switch points
using a separate split. }
The results are reported in Figure~\ref{realdata:fig} and
Table~\ref{tab:recommendations}.  We show performance for \coo\ and
\ind\ arrangements and also for a \mix\ method that
started out with \coo\ arrangements and switched to \ind\ arrangements
at a point determined by a hyperparameter search.  The \mix\ method
was often the best performer.
We observe consistent gains of 3\%-12\% that indicate
that arrangement tuning is an effective tool also on these more
complex real-life data sets.

\begin{figure}[t] 
  \center
  \includegraphics[width=0.23\textwidth]{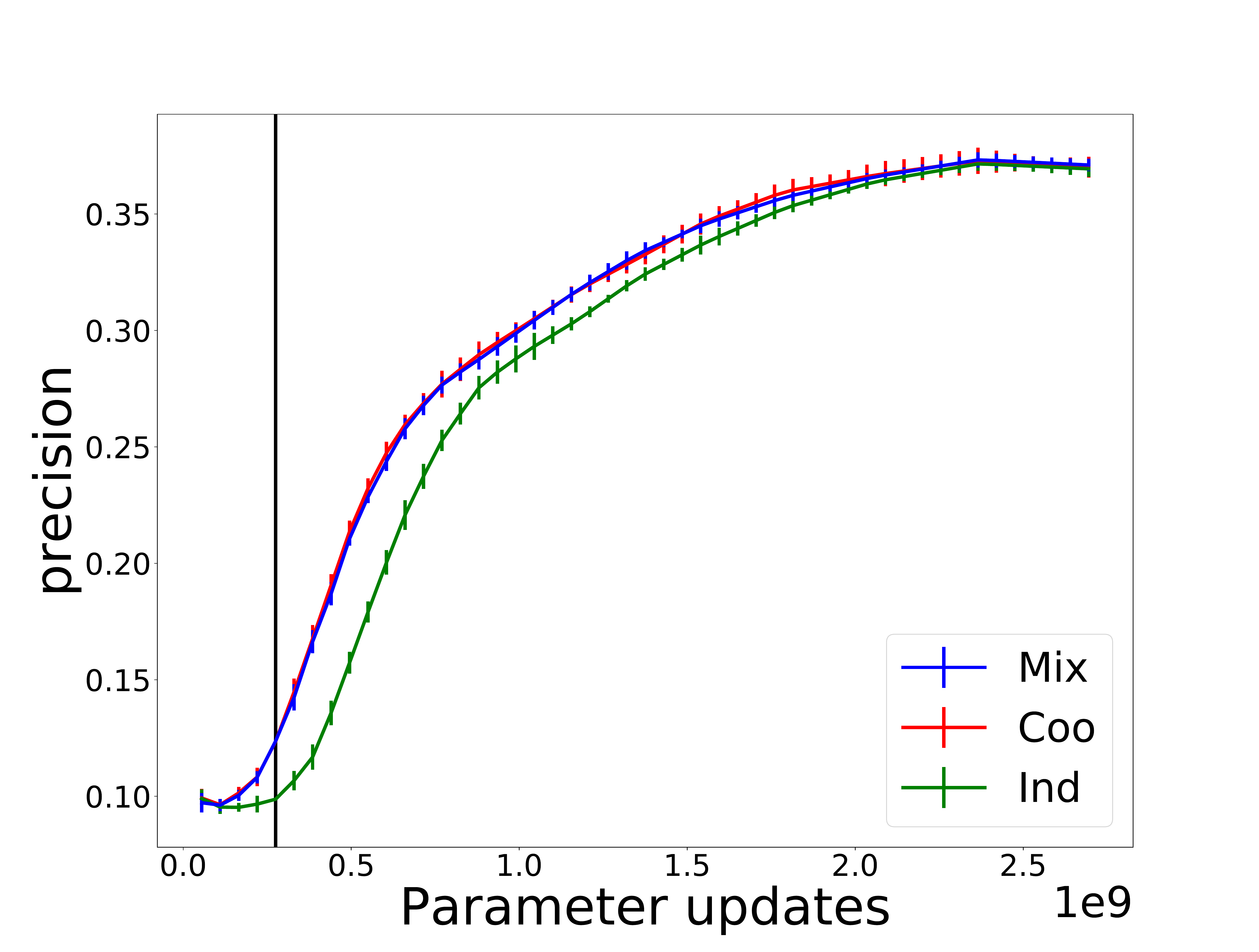}
\includegraphics[width=0.23\textwidth]{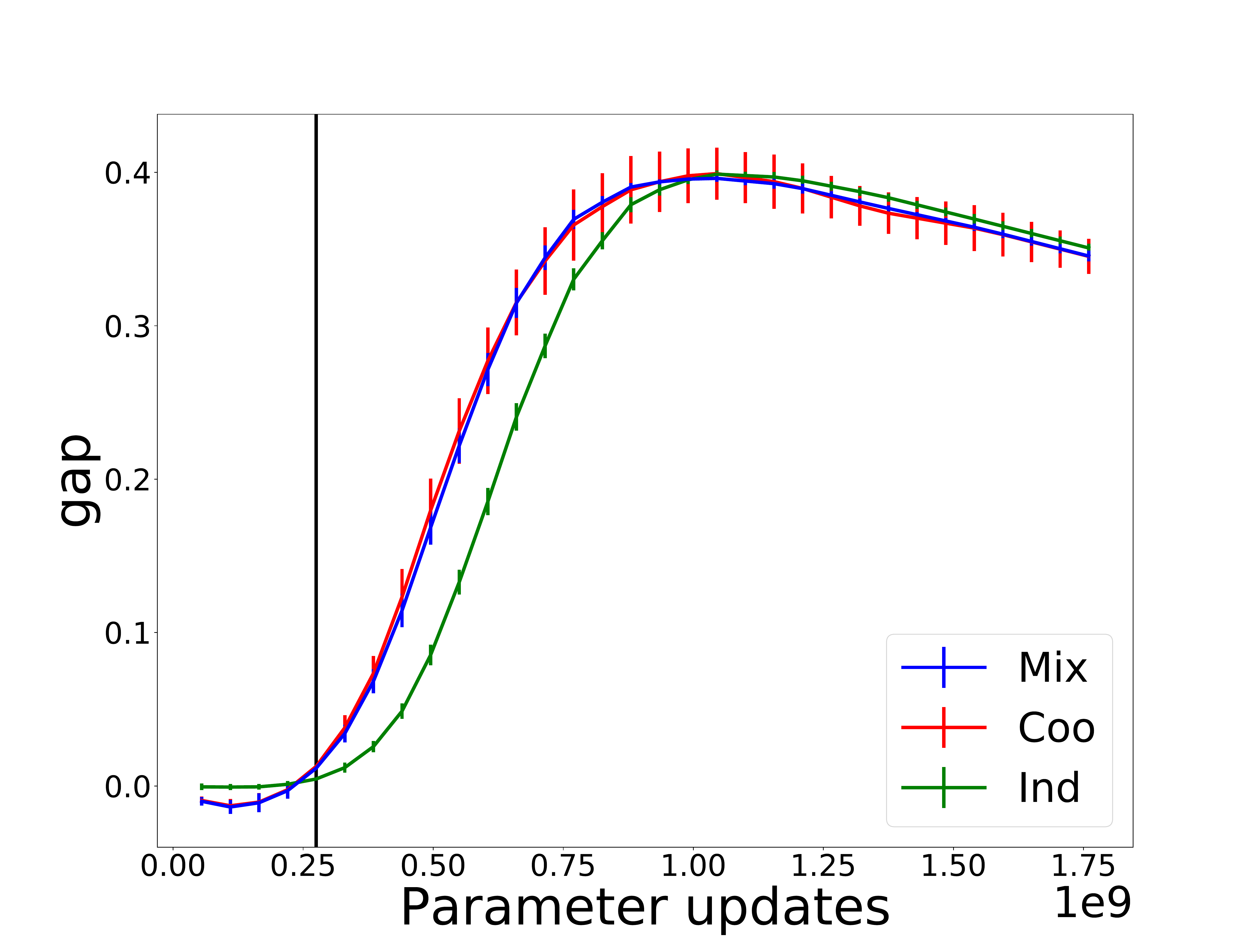}
\includegraphics[width=0.23\textwidth]{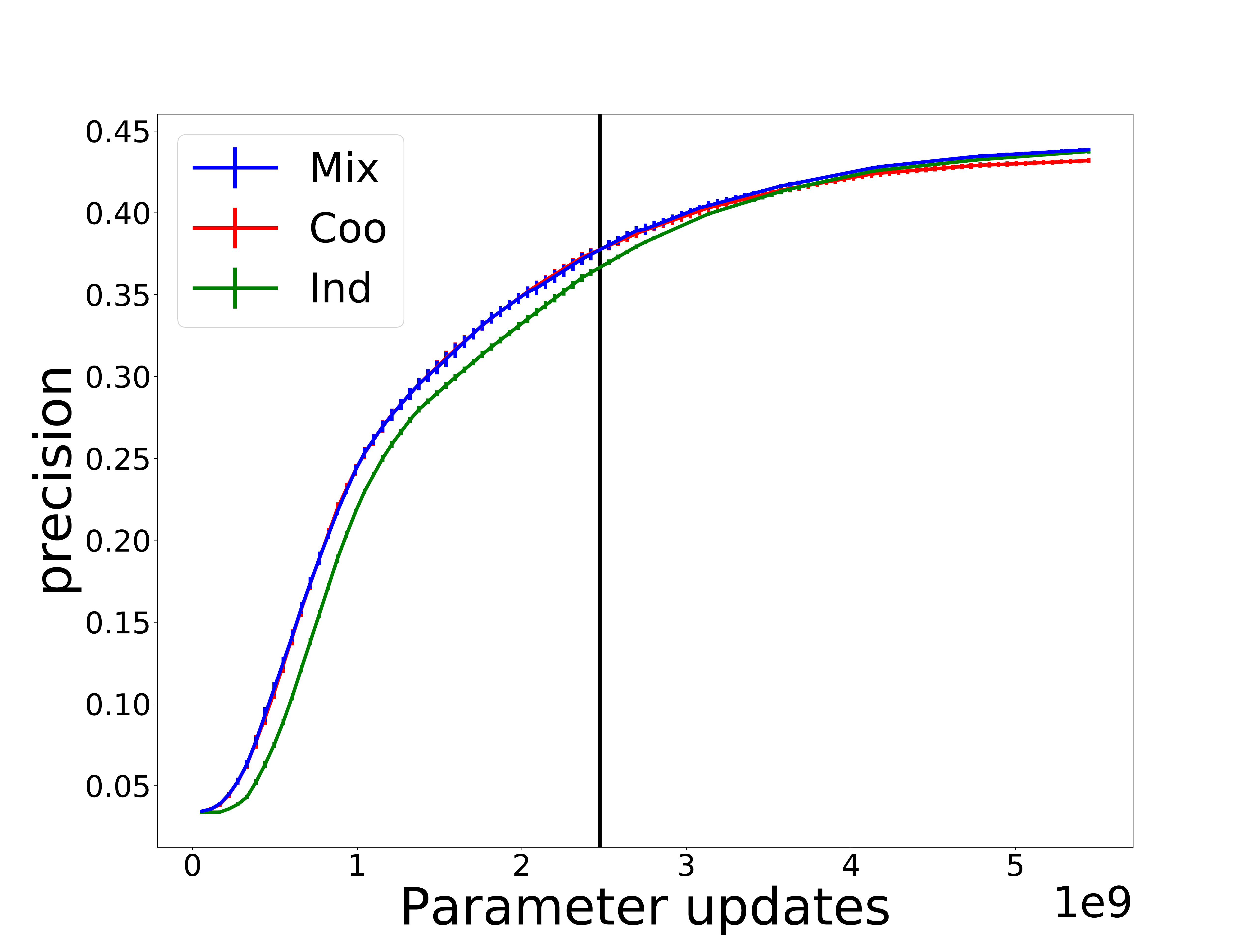}
\includegraphics[width=0.23\textwidth]{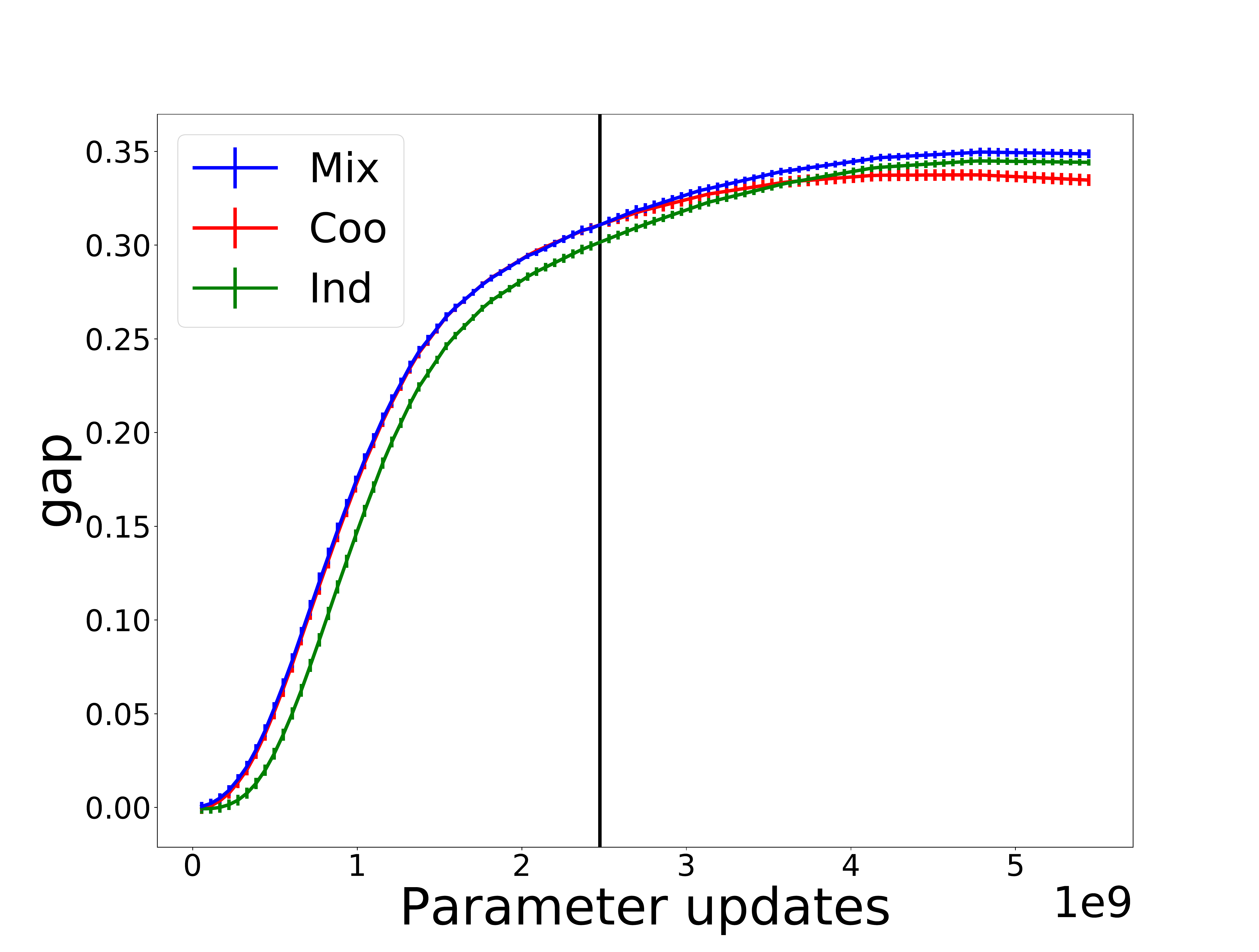}
\caption{Precision at $k=10$ (left) and cosine gap (right) in the
  course of training
  with different arrangement methods on \movielens (top) and \amazon
  (bottom) datasets.
($d=50$, $b=64$). The
  vertical lines indicate the switch point of \mix\ (from \coo\ to
  \ind).}
\label{realdata:fig}
\end{figure}

\begin{table}[]
\scriptsize 
\center 
\begin{tabular}{l|cc|cc|cc} 
   &   \multicolumn{2}{c}{$0.75\times$ peak} &
                                                            \multicolumn{2}{c}{$0.95\times$
                                                  peak} &
                                               \multicolumn{2}{c}{$0.99\times$ peak} \\
     &\%gain &$\times10^9$& \%gain& $\times 10^9$& \%gain& $\times 10^9$\\
\hline 
 \multicolumn{7}{c}{\amazon\ cosine gap:  Gain over \ind\ (peak=$0.35$) }\\
\hline 
Mix@4.5M  & 9.8 &1.56 & 10.29&3.12&\textcolor{blue}{{\bf 12.55}}&3.94\\
Coo &9.48&1.56&7.35&3.12&0&3.94\\
\hline\hline 
\multicolumn{7}{c}{\movielens\ cosine gap: Gain over \ind\ (peak=$0.40)$}\\
\hline 
Mix@0.25M  & 11.94&0.68&7.45&0.82&\textcolor{blue}{{\bf 6.08}}&0.92\\
  Coo  & 11.94	& 0.68 &	4.97&0.82&\textcolor{blue}{{\bf 6.08}}&0.92\\
\hline\hline
 \multicolumn{7}{c}{\amazon\ precision: Gain over \ind\ (peak=$0.44$)}\\
\hline 
Mix@4.5M  &10.76&1.80&3.00&3.40&\textcolor{blue}{{\bf 4.05}}&4.53\\
\hline\hline 
\multicolumn{7}{c}{\movielens\ precision:  Gain  over \ind\ (peak=$0.37)$}\\
\hline 
Mix@0.25M  & 10.24&0.85&4.31&1.66&\textcolor{blue}{{\bf 3.24}}&2.05\\
\end{tabular}
\caption{\amazon\ and \movielens, cosine gap and precision, training
  gain over \ind\ baseline  ($b=64$, $d=50$).}
\label{tab:recommendations}
\end{table}








\section{Example Selection Experiments} \label{selection:sec}

\begin{figure}[h]
  \center
\includegraphics[width=0.21\textwidth]{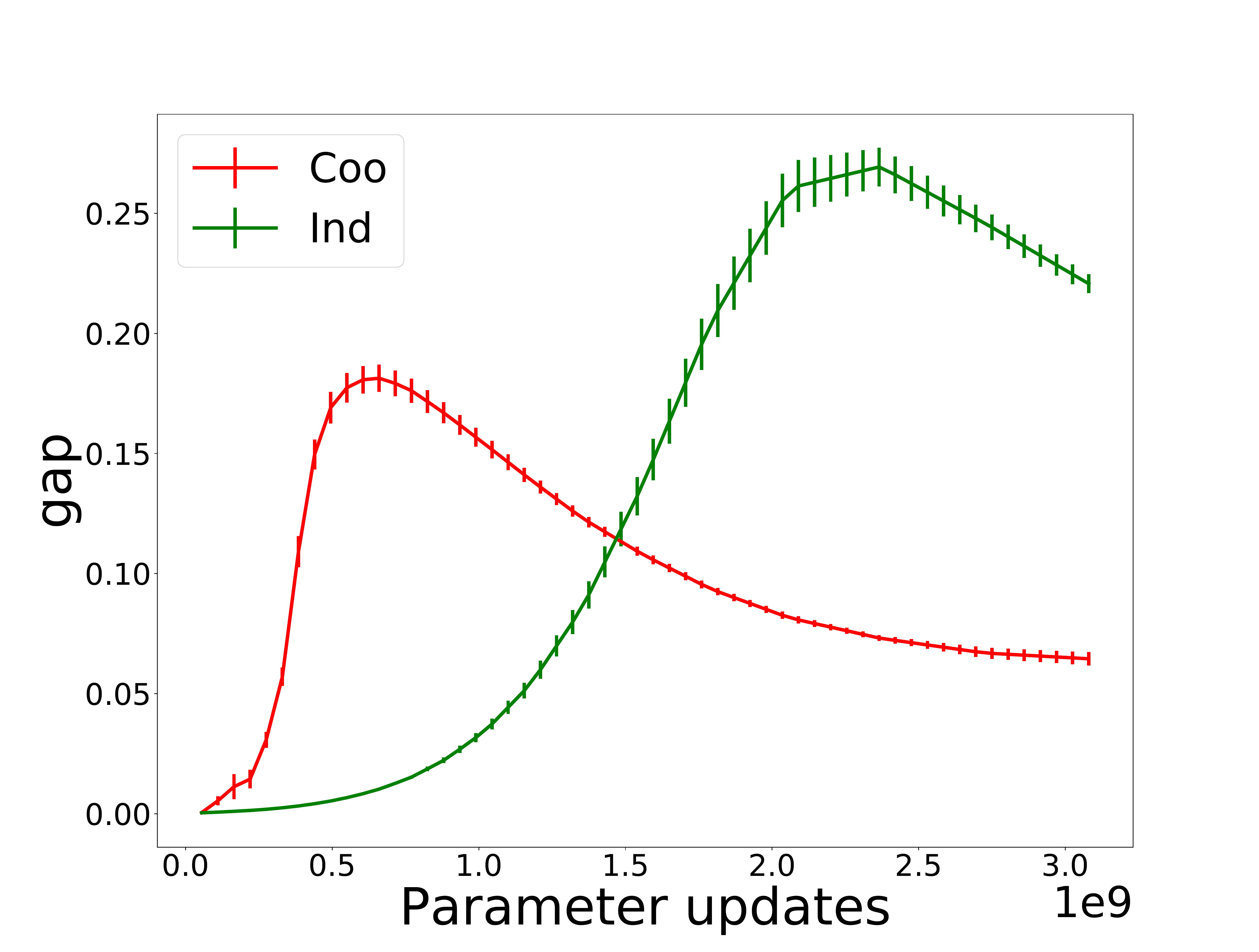}
\includegraphics[width=0.21\textwidth]{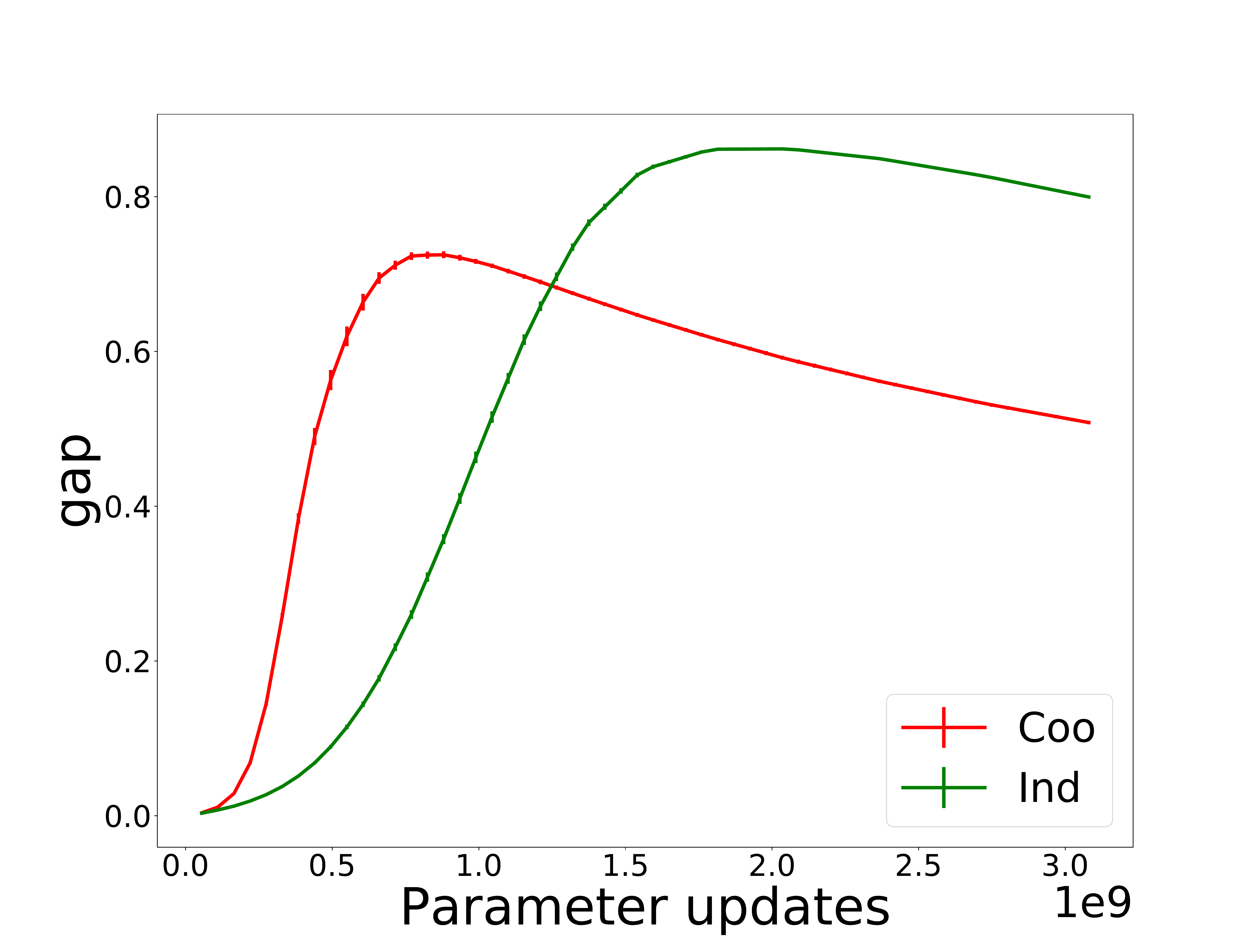}
\includegraphics[width=0.21\textwidth]{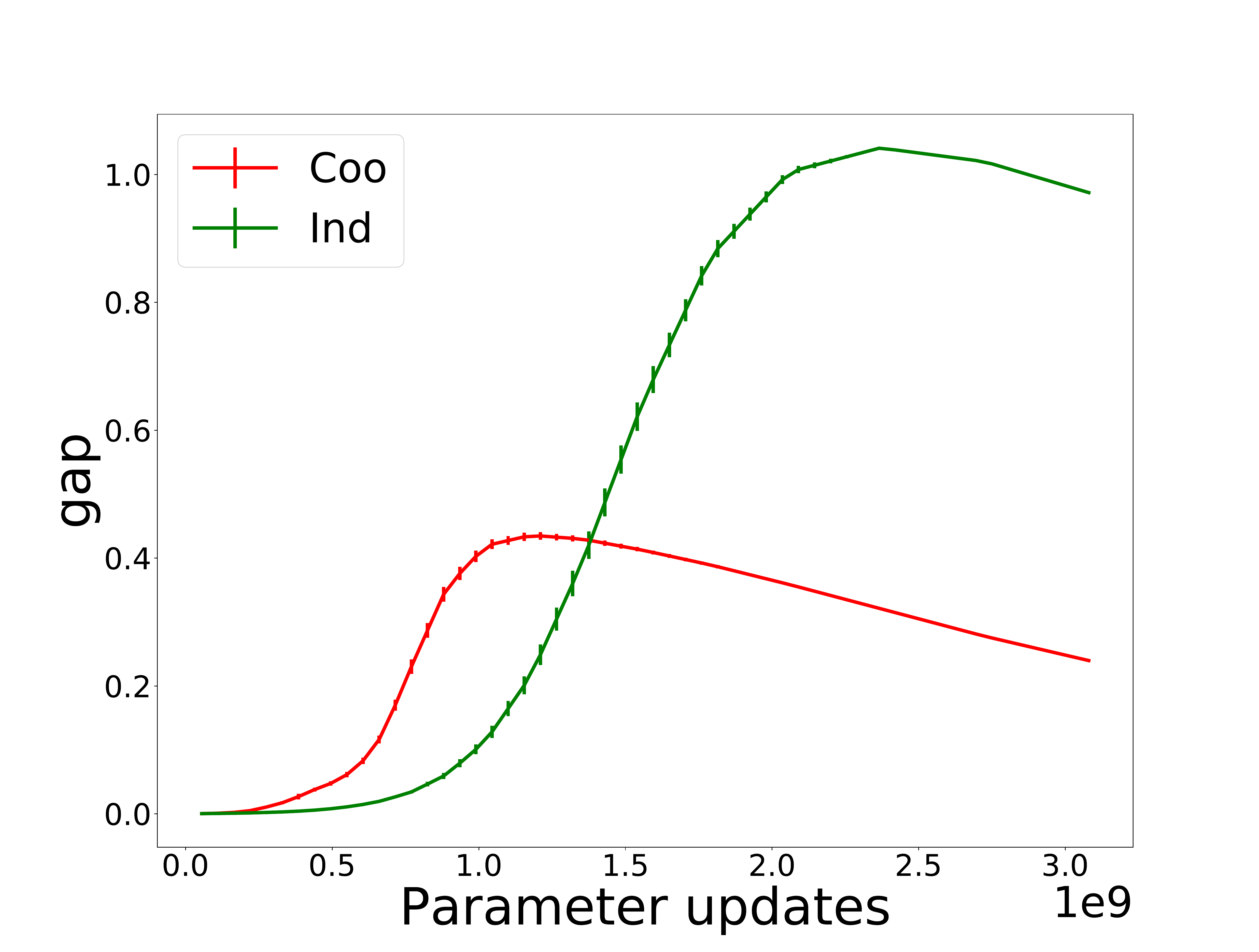}
\includegraphics[width=0.21\textwidth]{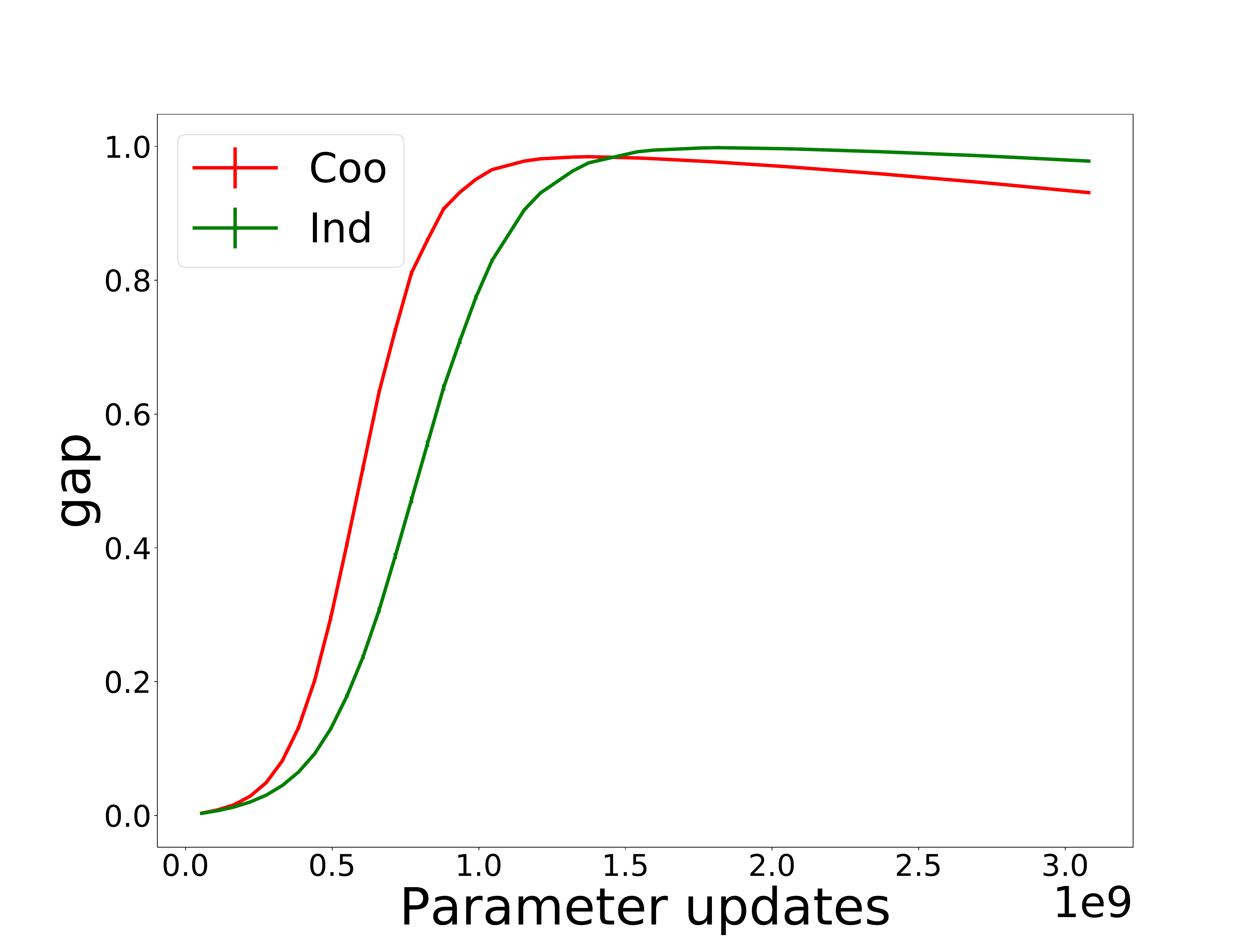}
  \ignore{
\includegraphics[width=0.23\textwidth]{synthetic_10000_10000_10_0_70_10000000_t_size_5_train_selection_scale_0_set_0_budget_b_size_4_lr__step_pol_4_0.pdf}
\includegraphics[width=0.23\textwidth]{synthetic_10000_10000_100_0_70_10000000_t_size_5_train_selection_scale_0_set_0_budget_b_size_4_lr__step_pol_4_0.pdf}
\includegraphics[width=0.23\textwidth]{synthetic_10000_10000_10_0_70_10000000_t_size_20_train_selection_scale_0_set_0_budget_b_size_4_lr__step_pol_4_0.pdf}
\includegraphics[width=0.23\textwidth]{synthetic_10000_10000_100_0_70_10000000_t_size_20_train_selection_scale_0_set_0_budget_b_size_4_lr__step_pol_4_0.pdf}
}
\caption{Training with selected examples (sub-epoch with \ind\
  and \coo\ arrangements).  Stochastic blocks datasets and minibatch
  size parameter $b=4$.
Left to right:  $(T=5,B=10)$;   $(T=5,B=100)$, $(T=20,B=10)$, $(T=20,B=100)$. 
\label{budgetR:fig}}
\end{figure}

In this section we empirically explore 
the qualities of the information contained in
sub-epochs when using \coo\ and \ind\ arrangements.
We select a small set of training examples that corresponds to 
sub-epochs with different arrangements and then train to
convergence using multiple epochs on only the selected examples.
We sampled  $T=5,10,15,20$ examples from each row (for focus
 updates) and symmetrically from each
 column (for context updates) of the association matrix.  
 To emulate a sub-epoch of \ind\ arrangement, we select $T$ independent examples
 from each row $i$ by selecting
 a column $j$ with probability $\kappa_{ij}/\|\kappa_{i\cdot}\|_1$.
 To emulate a sub-epoch with \coo\ arrangement 
 we repeat the following $T$ times. 
We draw $u_j\sim \Exp[1]$ for each column and select for each row $i$
the column $\arg\max_j \kappa_{ij}/u_j$.   Clearly the marginal
distribution of both selections is the same: The probability that 
  column $j$ is selected for row $i$  is equal to 
  $\kappa_{ij}/\|\kappa_{i\cdot}\|_1$.
Symmetric schemes apply to columns. 

 We trained embeddings (with no bias parameter and using \ind\
 arrangements) on these small subsets of examples using identical setups.  Updates were
 according to minibatch designation:
Row samples used for updating
row embeddings and  column samples for updating column embeddings. 
 Representative results are reported in Figure~\ref{budgetR:fig}.
We observe that \coo\  selection
consistently results in faster early training than 
\ind\ selection but
sometimes reaches a lower peak.
We explain the faster early training by the \coo\ selection preserving 
short-range similarities (based on first-hop relations) and the lower peak by 
lost ``long-range'' structure (that reflects longer-range relations as
those captured by longer random walks and metrics such as personalized
page rank). 
Our \coo\ arrangements which use the complete set of examples
retain both the short-range benefits of 
\coo\ selections and the long-range benefits
of \ind\ selections.
\ignore{
 With fewer examples per entity, coordinated selection
also had a higher peak quality than  the respective 
independent selection.  With more examples and larger blocks, 
the coordinated selection peaked lower, due to loss of the multi-hop expander structure.}


    \section{Conclusion} \label{conclu:sec}

    We demonstrated that SGD can be accelerated with principled arrangements of training 
    examples that are mindful of the ``information flow'' through gradient
    updates.
\notinproc{

   We mention some directions for followup work.  Our arrangements respect a specified marginal distribution of
    training examples and hence
 can be combined and explored with methods, such as curriculum learning,  that
 specify the distribution and even modify it in the course of training.
We explored here one design goal of
self similarity, where sub-epochs preserve structural properties of
the full data.  Our case study further focused on pairwise associations
(``matrix factorization''  with  the SGNS objective), and preserving weighted Jaccard
similarities of rows and columns.  It is interesting to explore extensions that include
other loss objectives and deeper networks  and design principled
arrangements suitable to more complex association structures.  

}
\ignore{
showed that arrangements
    can be a powerful optimization knob.
    
We consider embedding computations with stochastic gradients and 
establish that the arrangement of training examples into
minibatches can be a powerful performance knob.
In particular, we introduced coordinated arrangements as a principled method to
accelerate SGD training of embedding vectors.
  Our experiments focused on the popular SGNS loss and our methods were designed for pairwise associations.   In future we hope to explore the use of coordinated
arrangement with other loss objectives, deeper networks, and more complex association structures. 
}

\onlyinproc{\newpage}
\section*{Acknowledgements}
We are grateful to the anonymous ICML 2019 reviewers for many helpful
comments that allowed us to improve the presentation.  This research is partially supported by the Israel Science Foundation (Grant No. 1841/14).
\bibliographystyle{icml2019}
\bibliography{cycle}
\onlyinproc{\end{document}}
\appendix

\section{Embedding Dimension Sweep} \label{dimanalysis:sec}

We trained with different dimensions ($d=5$ to $d=200$) in order to
understand the impact of the choice of dimension on final quality and
convergence.  We report
results for
  \ind\ arrangements  as the relative behavior of arrangement methods
  was similar across dimensions with fixed training rate of $\eta=0.02$.
  The quality in the course of training (cosine gap and precision at
  $k=10$) is reported in
Figure~\ref{dimanalysisRW:fig}. 
For each experiment we report the training amount both in terms of the
total number of parameter updates performed and the total
multiplicity of positive examples processed during training.  Note
that the ratio of the number of updates to the multiplicity of
examples is $\lambda d$, where $\lambda$ is the ratio of negative to positive
examples and $d$ is the dimension.  Therefore, the measures are
linearly related for a fixed dimension but the number of updates per
processed example increases with the dimension.
\begin{figure*}[hbt!]
  \center
{\small \amazon\ dataset}\\
\includegraphics[width=0.22\textwidth]{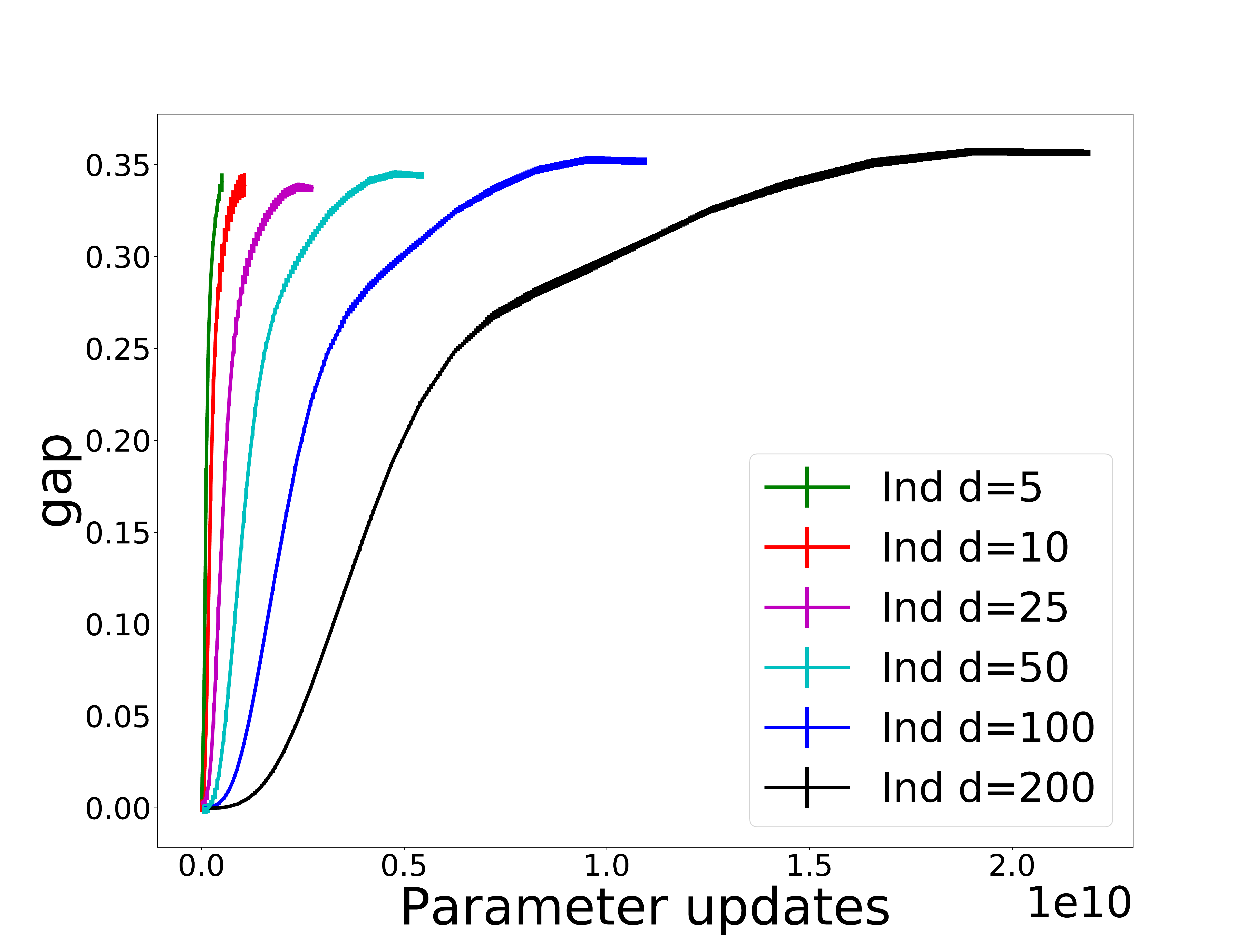}
\includegraphics[width=0.22\textwidth]{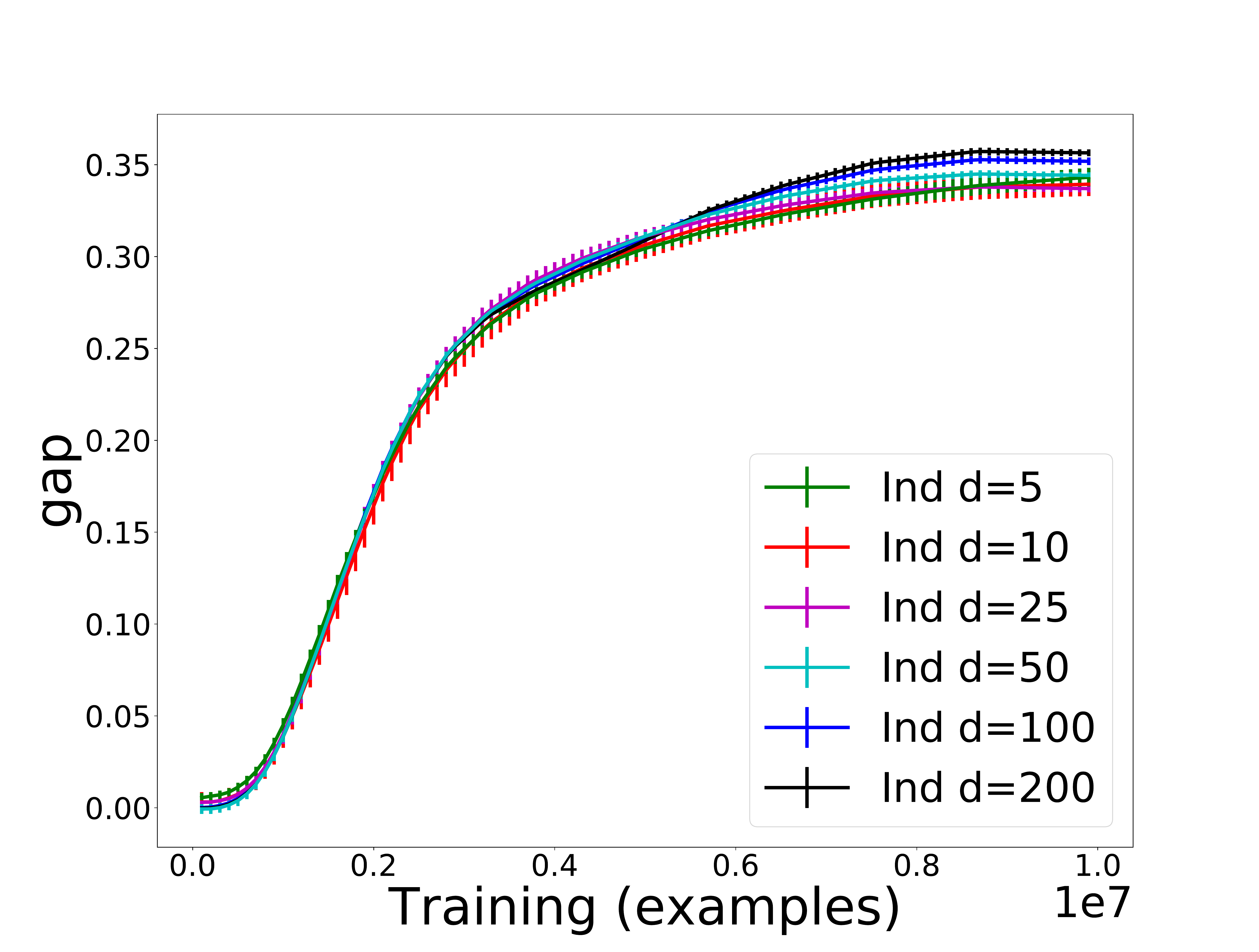}
\includegraphics[width=0.22\textwidth]{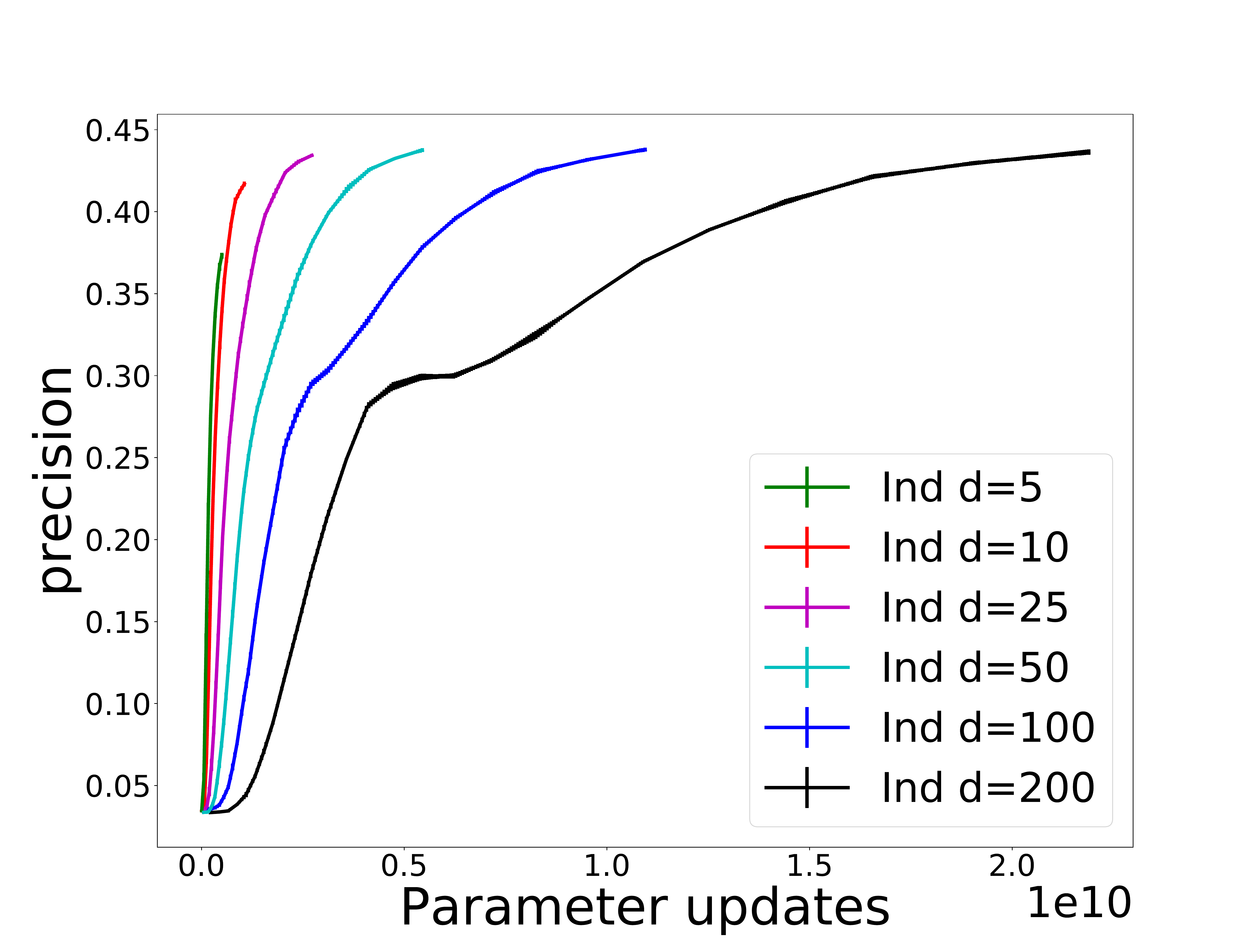}
\includegraphics[width=0.22\textwidth]{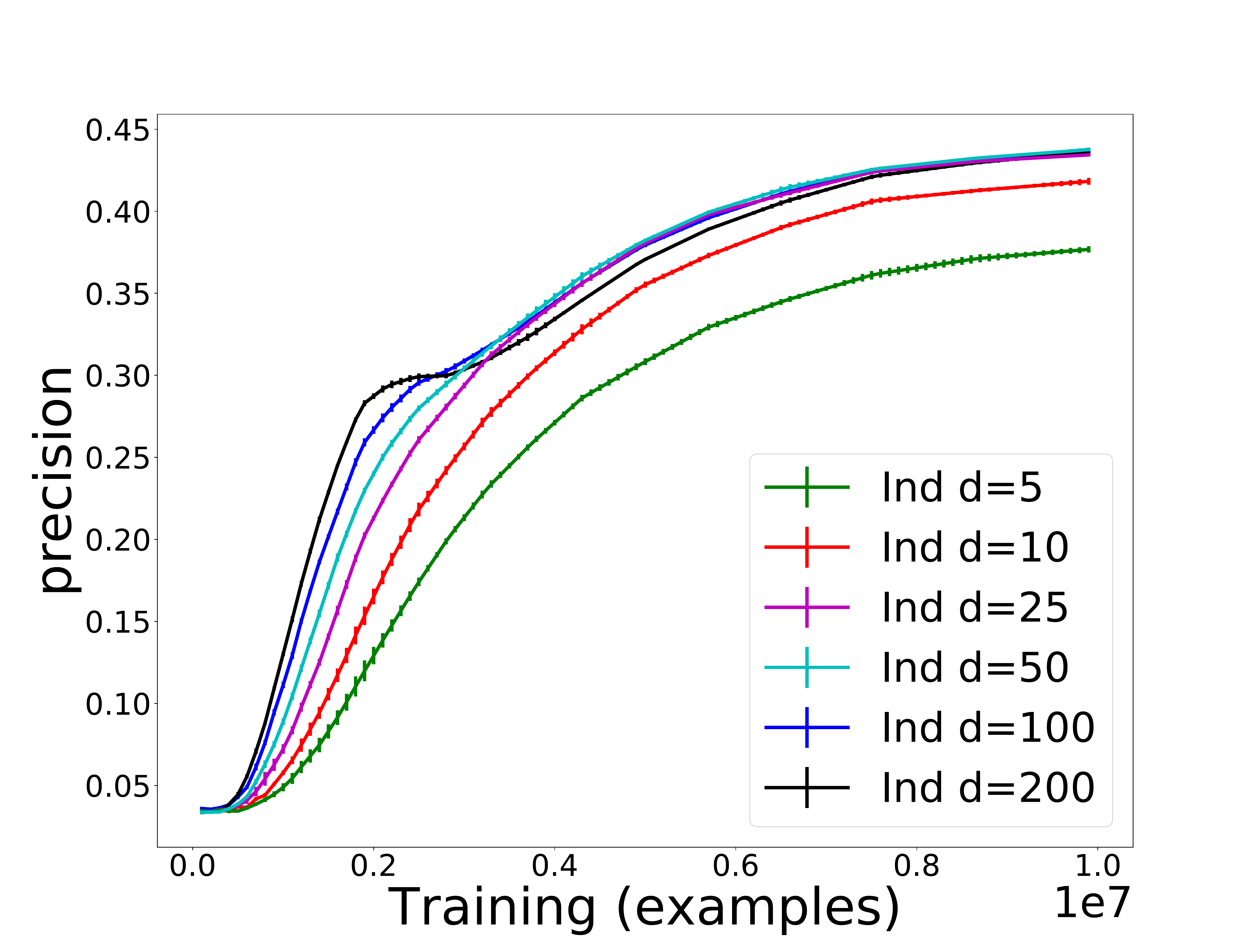}\\
{\small \movielens\ dataset}\\
\includegraphics[width=0.22\textwidth]{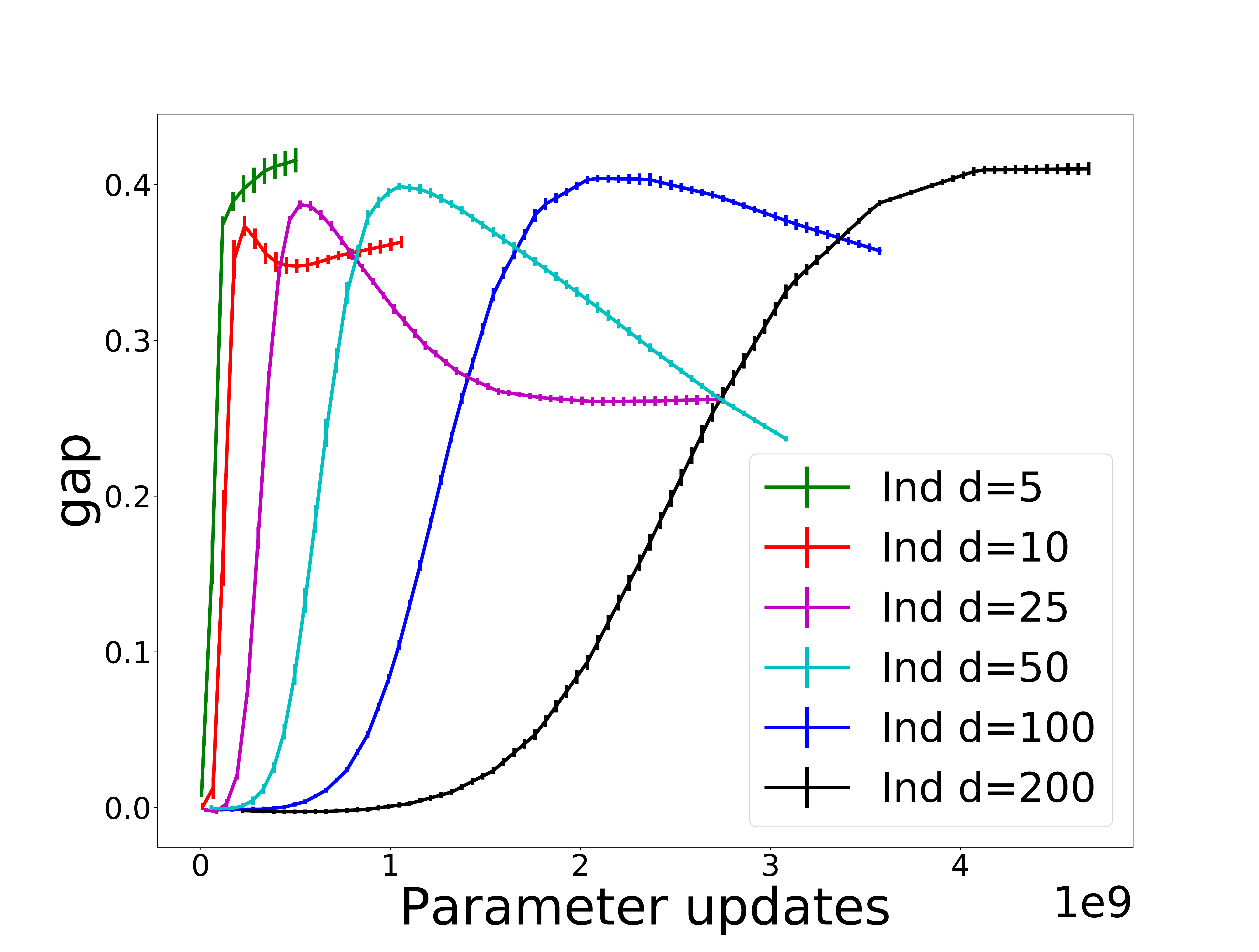}
\includegraphics[width=0.22\textwidth]{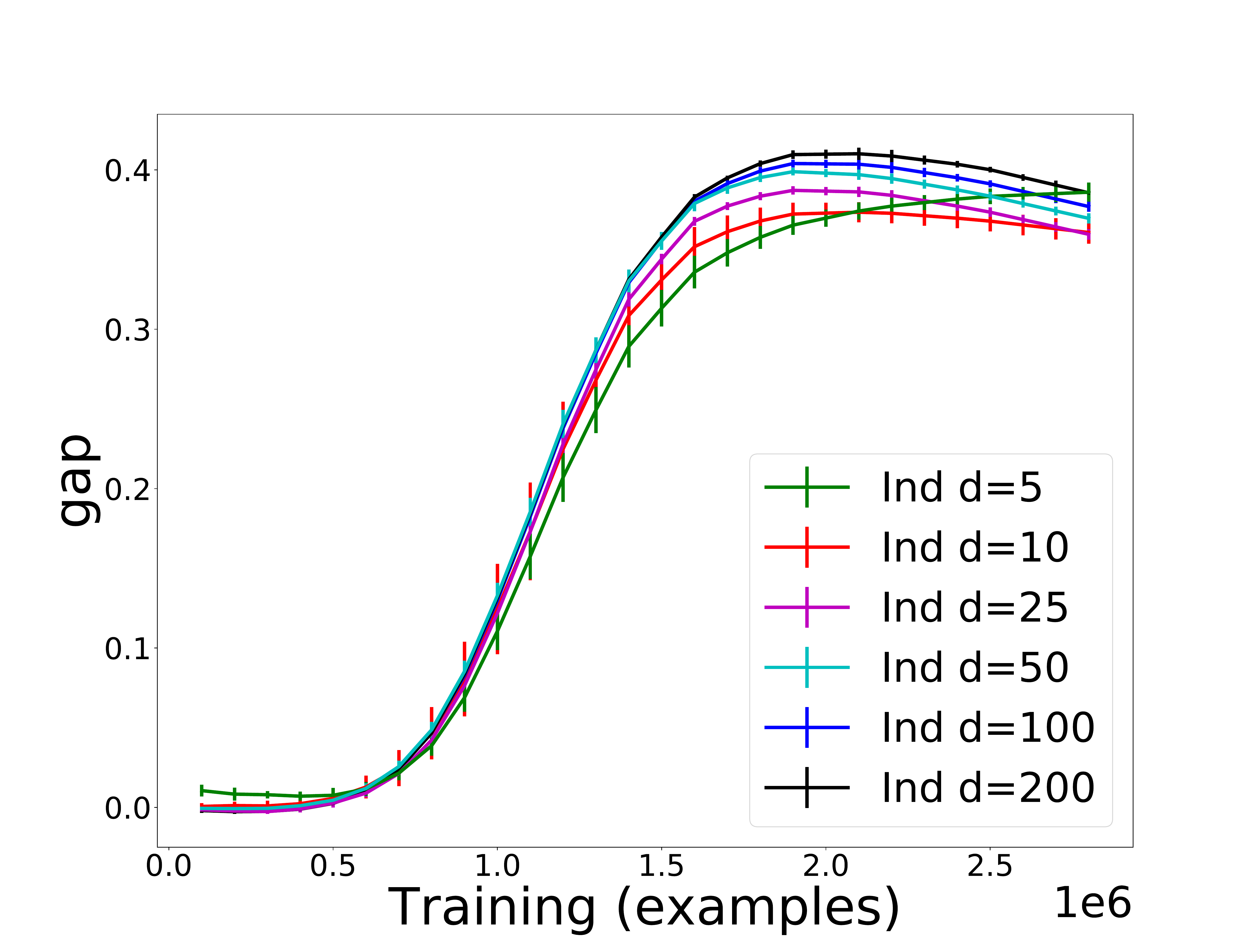}
\includegraphics[width=0.22\textwidth]{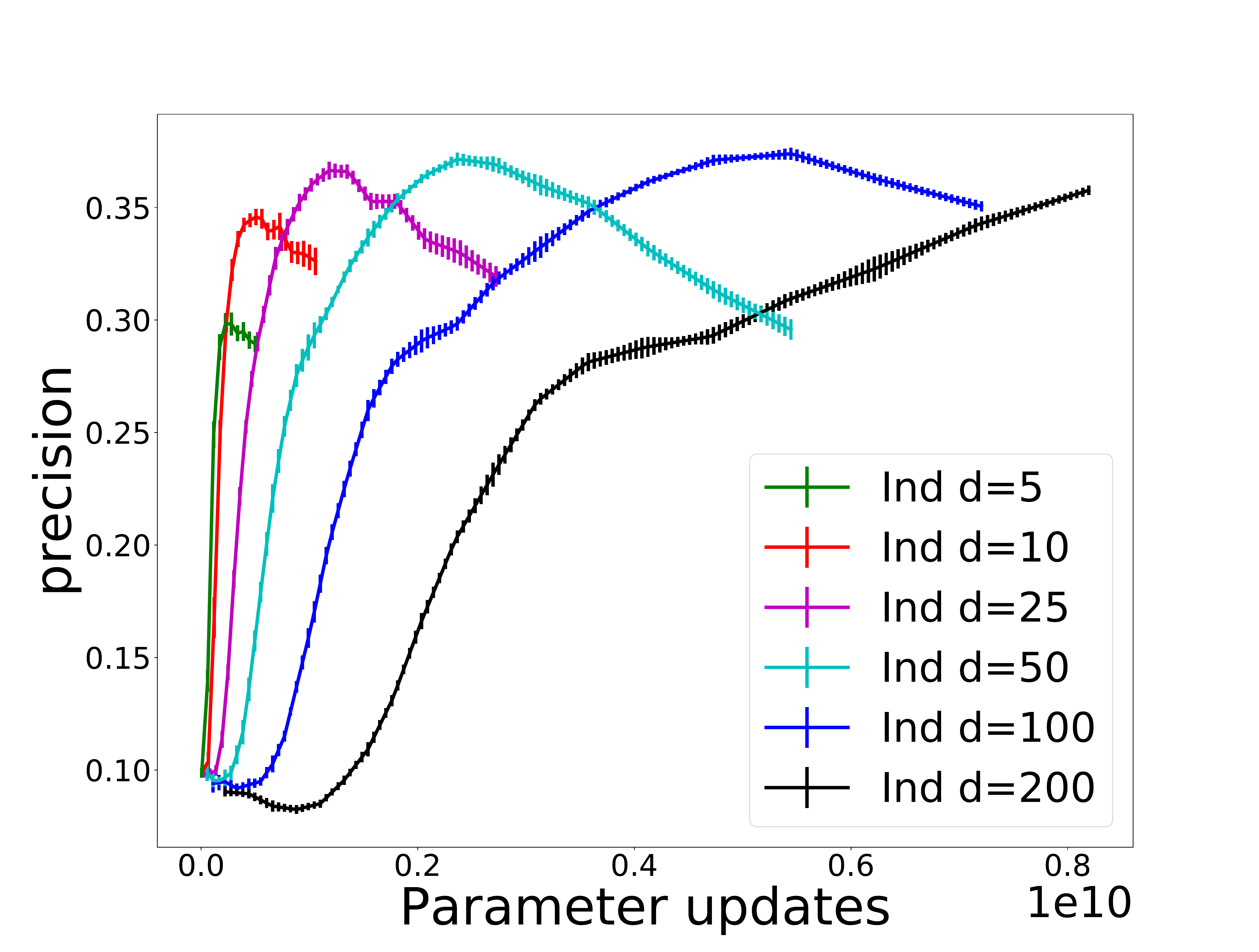}
\includegraphics[width=0.22\textwidth]{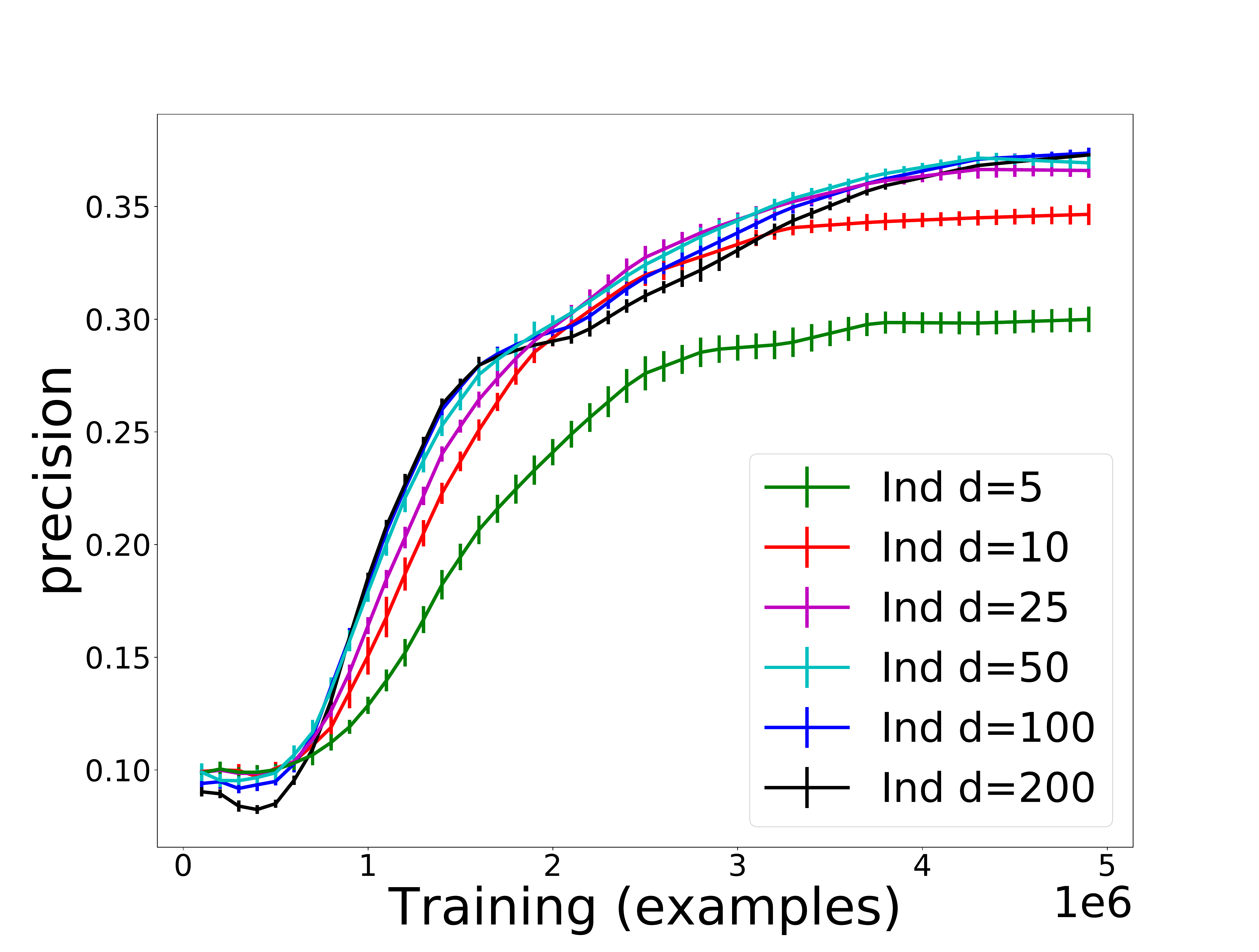}\\
{\small $10^4\times 10^4$ stochastic blocks dataset ($B=10$ blocks)}\\
\includegraphics[width=0.22\textwidth]{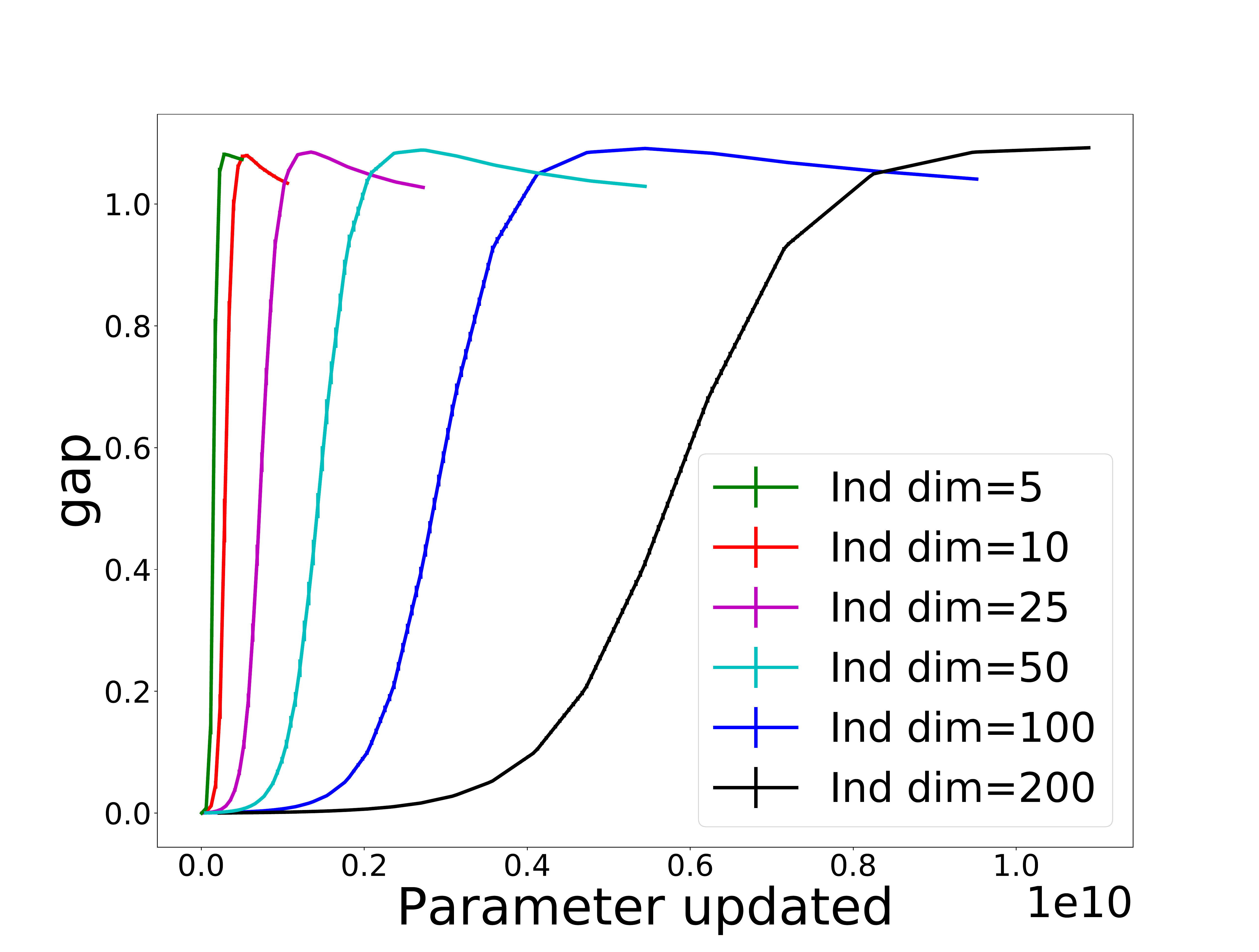}
\includegraphics[width=0.22\textwidth]{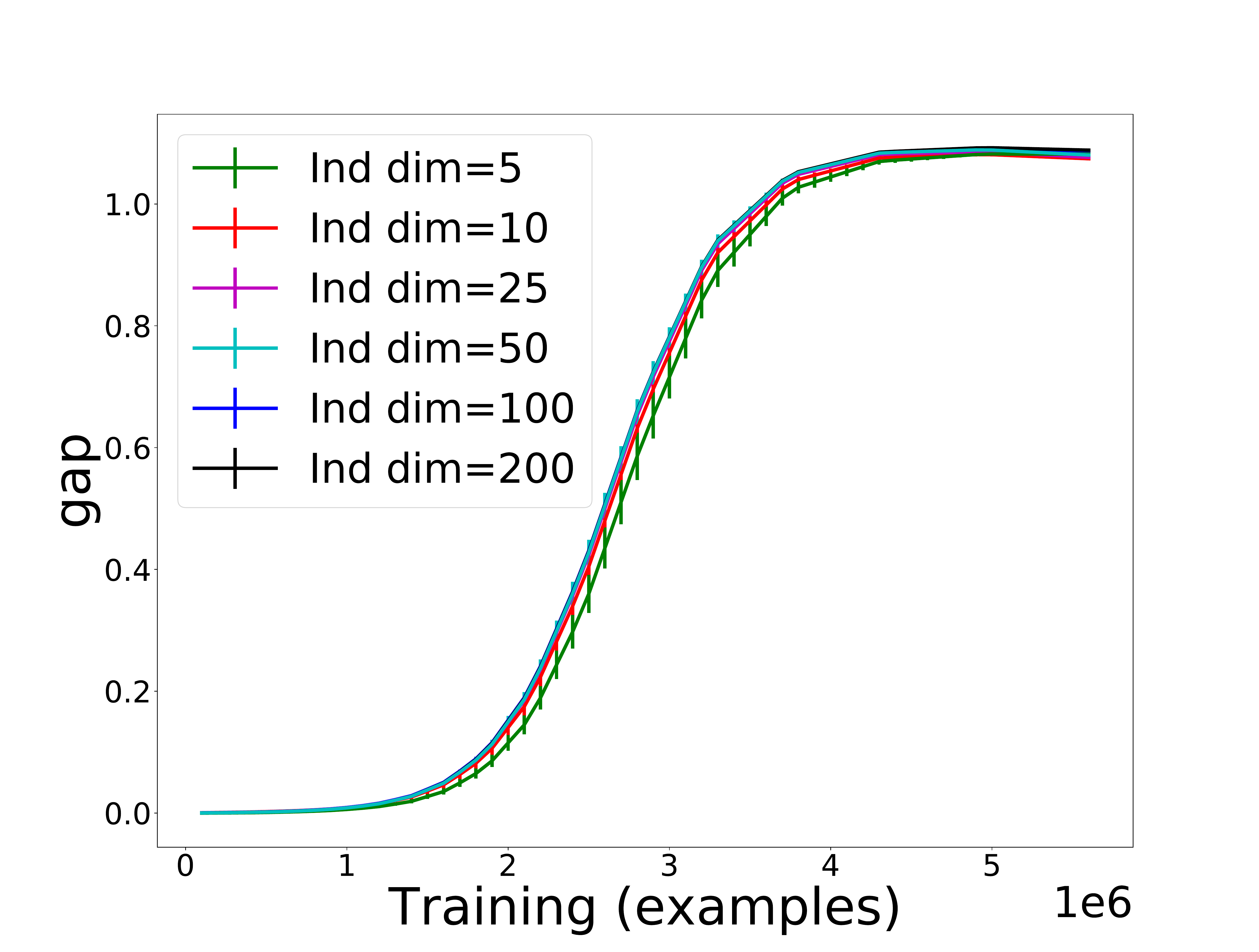}
\includegraphics[width=0.22\textwidth]{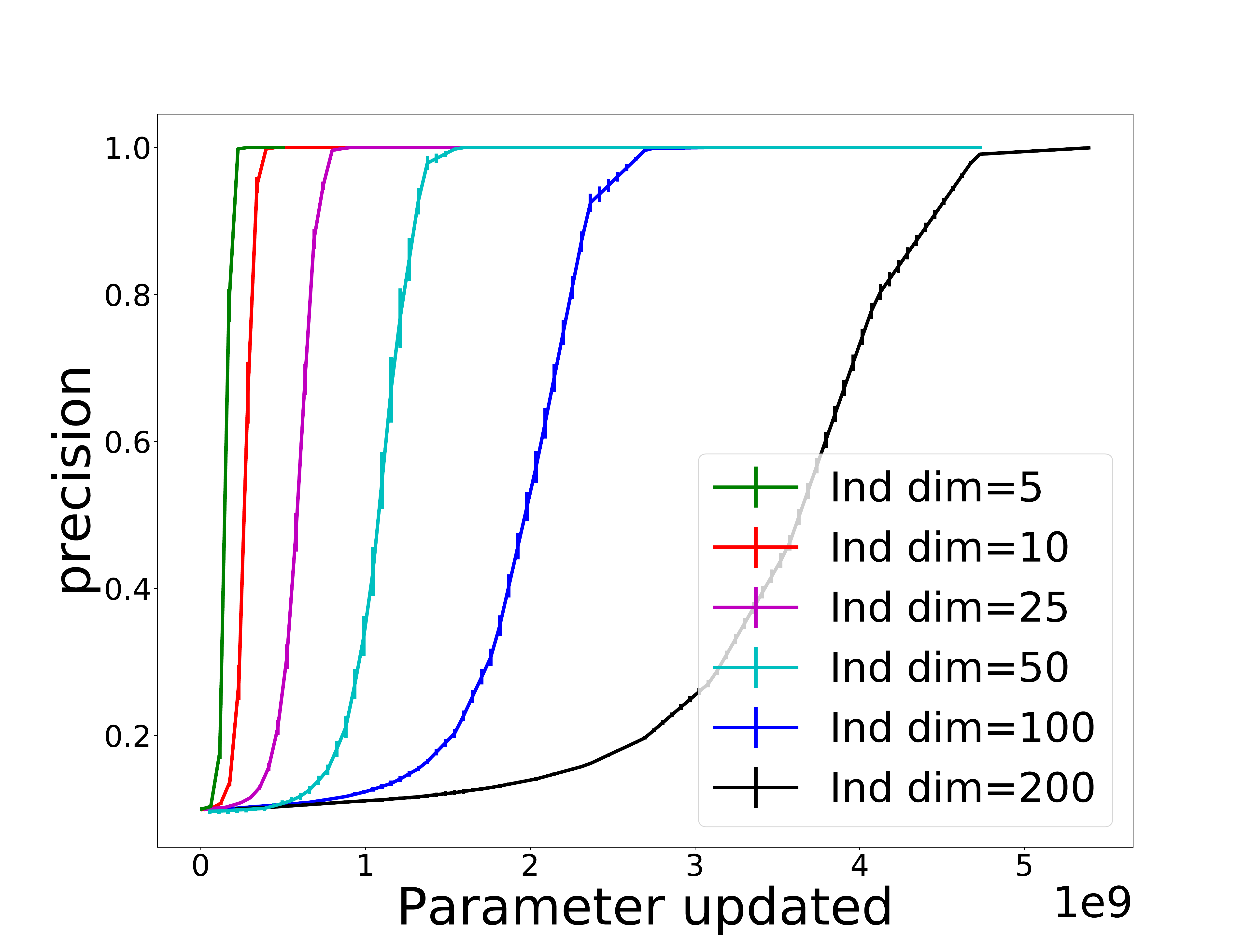}
\includegraphics[width=0.22\textwidth]{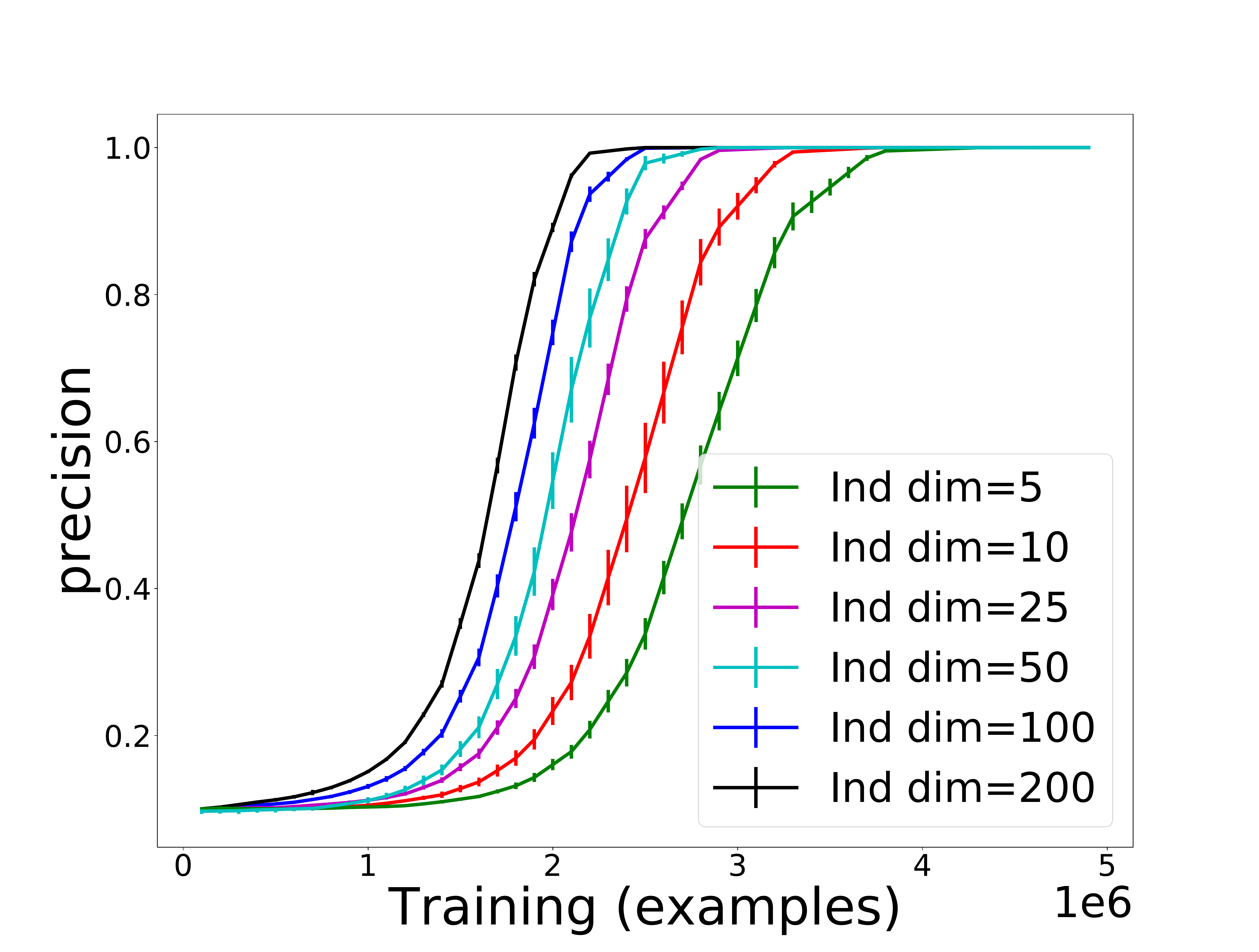}\\
{\small $10^4\times 10^4$ stochastic blocks dataset ($B=100$ blocks)}\\
\includegraphics[width=0.22\textwidth]{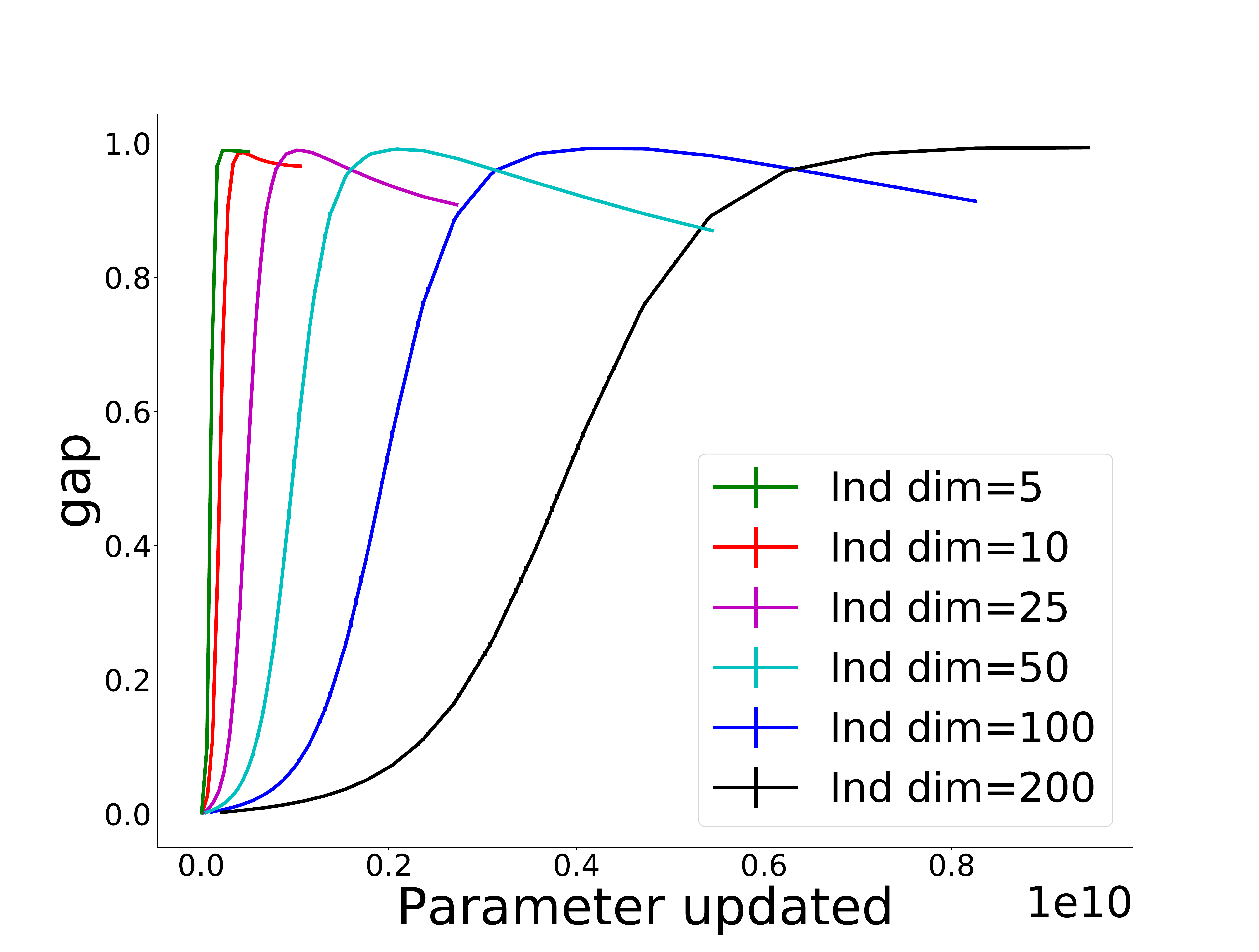}
\includegraphics[width=0.22\textwidth]{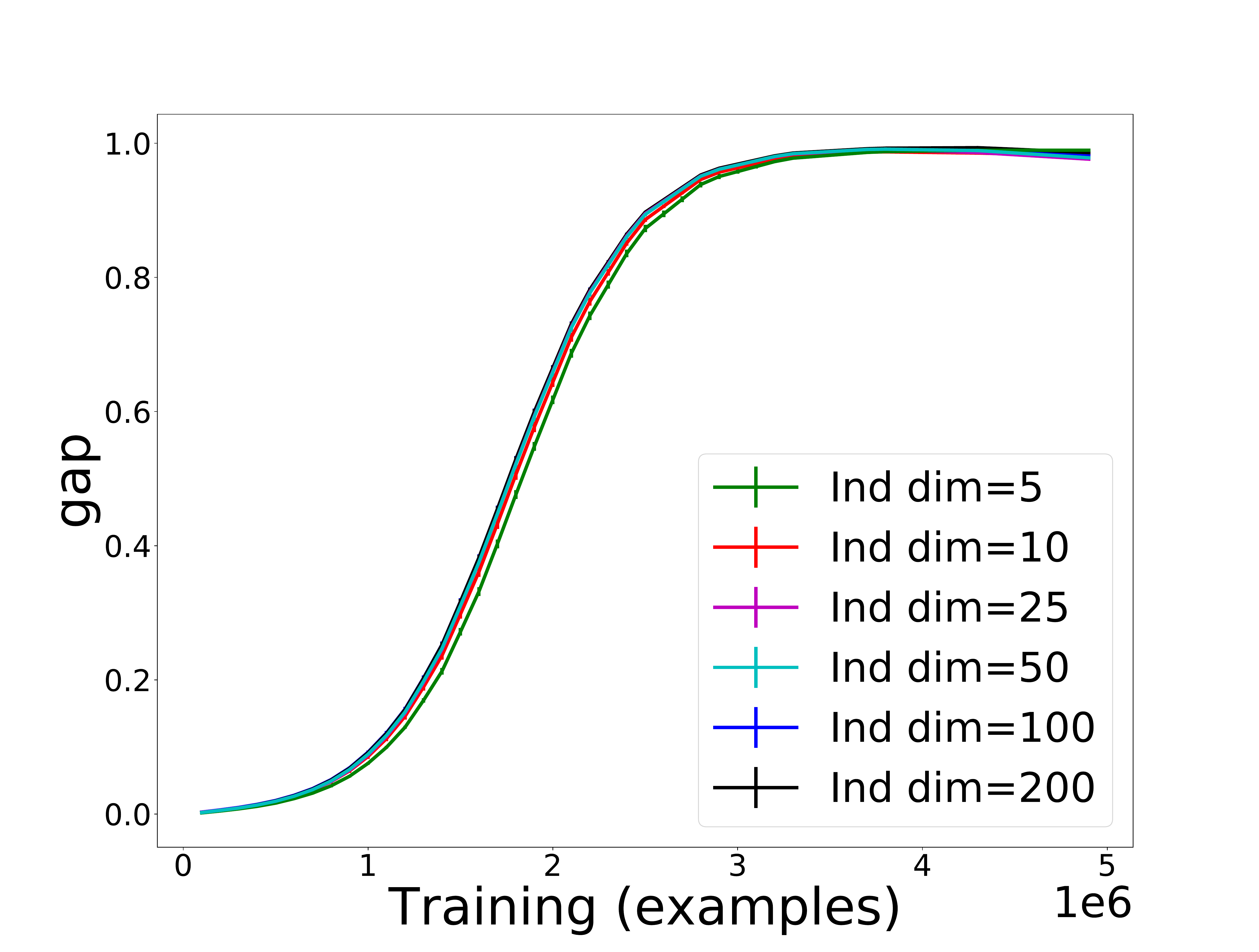}
\includegraphics[width=0.22\textwidth]{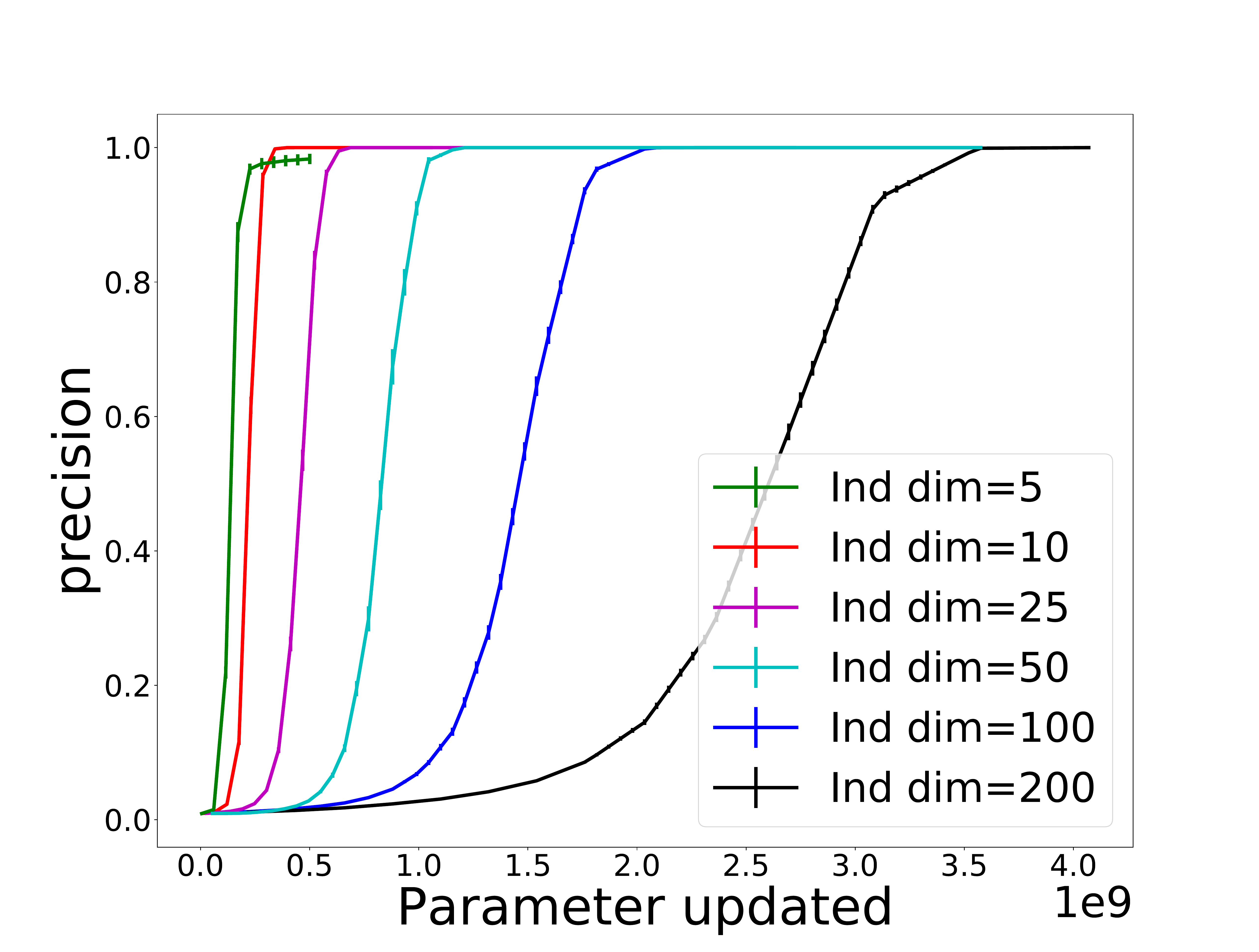}
\includegraphics[width=0.22\textwidth]{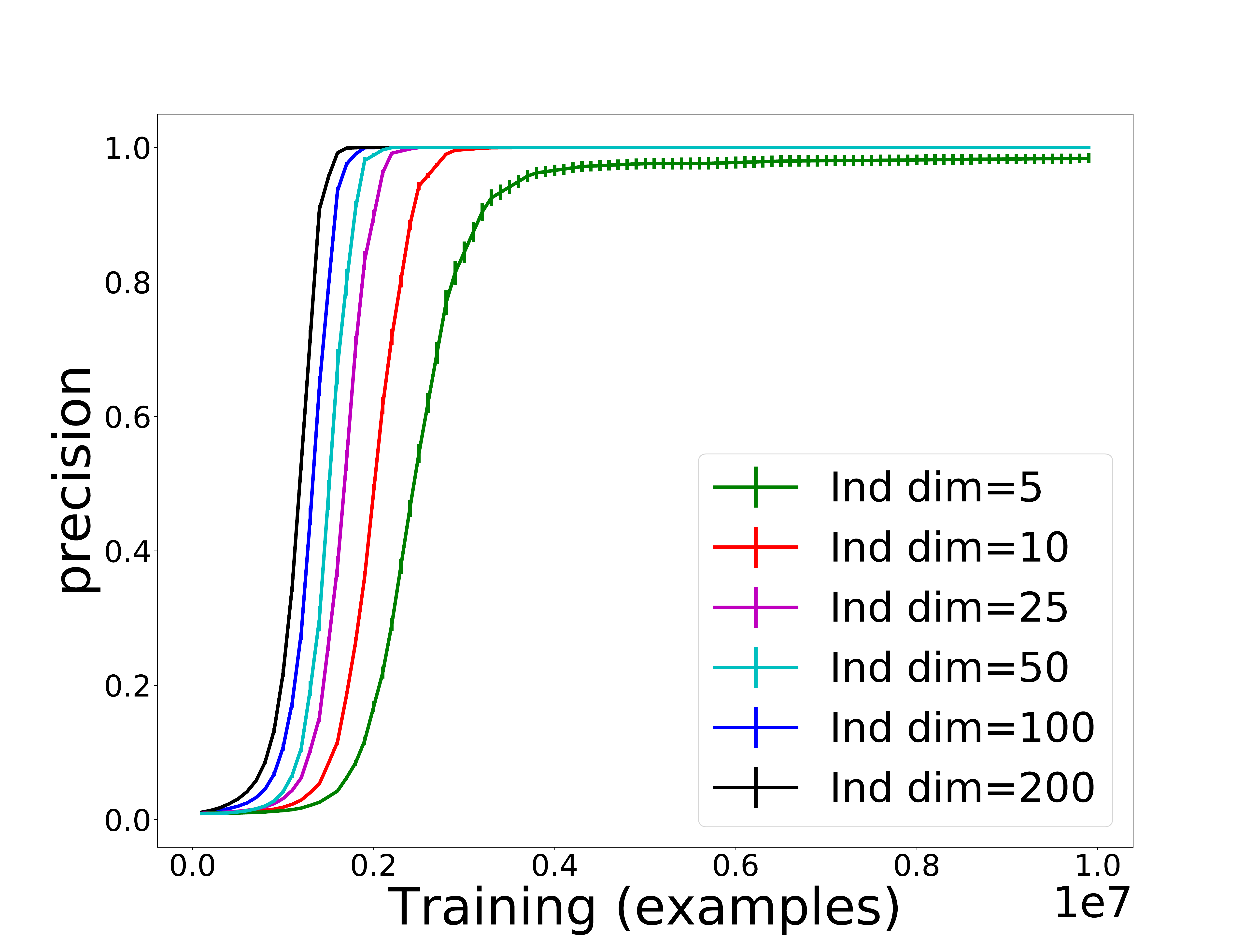}
\caption{Training  (\ind\ with $b=64$) with different
  dimensions. Cosine gap and Precision for $k=10$ for 
 \movielens,  \amazon, and stochastic blocks datasets.  Training is reported in terms of
 number of parameter updates (left) and in terms of processed
 positive examples (right).}\label{dimanalysisRW:fig}
\end{figure*}

On all data sets we can observe
 slightly faster convergence with higher dimension in terms of number of training
 examples but generally slower convergence with dimension when
 considering the number of parameter updates.   On the recommendations data sets and for the precision quality
  measure on the stochastic blocks data we can see that the peak
  quality increases with the dimension.   
 This means that lower
  dimension is more effective in reaching a
  particular lower quality level but may reach lower peak quality.
  This supports methods (such as our LSH based refinements) that can leverage coarser, lower
  quality, but much more efficient to compute embeddings to
  accelerate the training of more complex models.

\newpage
  \section{Minibatch Size Sweep} \label{minibatchsize:sec}

We trained using \ind\ arrangements with minibatches with size
parameters $b=\{1,4,16,64,256\}$ in order to understand how
performance depends on minibatch size.
Recall that the parameter value $b$ is the number of positive
examples, so the actual minibatch size in terms of the total number of
examples is $\lambda b$, where
$\lambda$ (we used $\lambda=10$) is the number of negative examples
per postive one.
When sweeping the minibatch size we used the learning rate $\eta=0.02$
and dimension $d=50$ as in our main experiments.
On our synthetic and real datasets and both measures there was no
performance difference between different parameters.  Therefore, the gains of our arrangement
methods hold also with respect to \ind\ arrangements with very
small minibatches.

  \newpage
\section{Learning Rate Sweep} \label{learningrate:sec}

We trained with a fixed learning rate to facilitate a better
comparison of different arrangement methods and used $\eta=0.02$ in our
reported results. 
Figure~\ref{lrsweep:fig} shows the performance with \ind\ arrangement
for varied fixed learning rates between $\eta=0.01$ and $\eta=0.15$.  As expected, we can see faster
initial training with higher learning rate but (for the cosine gap
measure) lower final quality.  The relative behavior of different
methods, in particular, the effectiveness of \coo\ arrangements
earlier in the training, was similar across learning rates.
Figure~\ref{lr0_1:fig} reports representative results for $\eta=0.01$.
\begin{figure}[htb!]
  \center
{\small Cosine gap}\\
\includegraphics[width=0.22\textwidth]{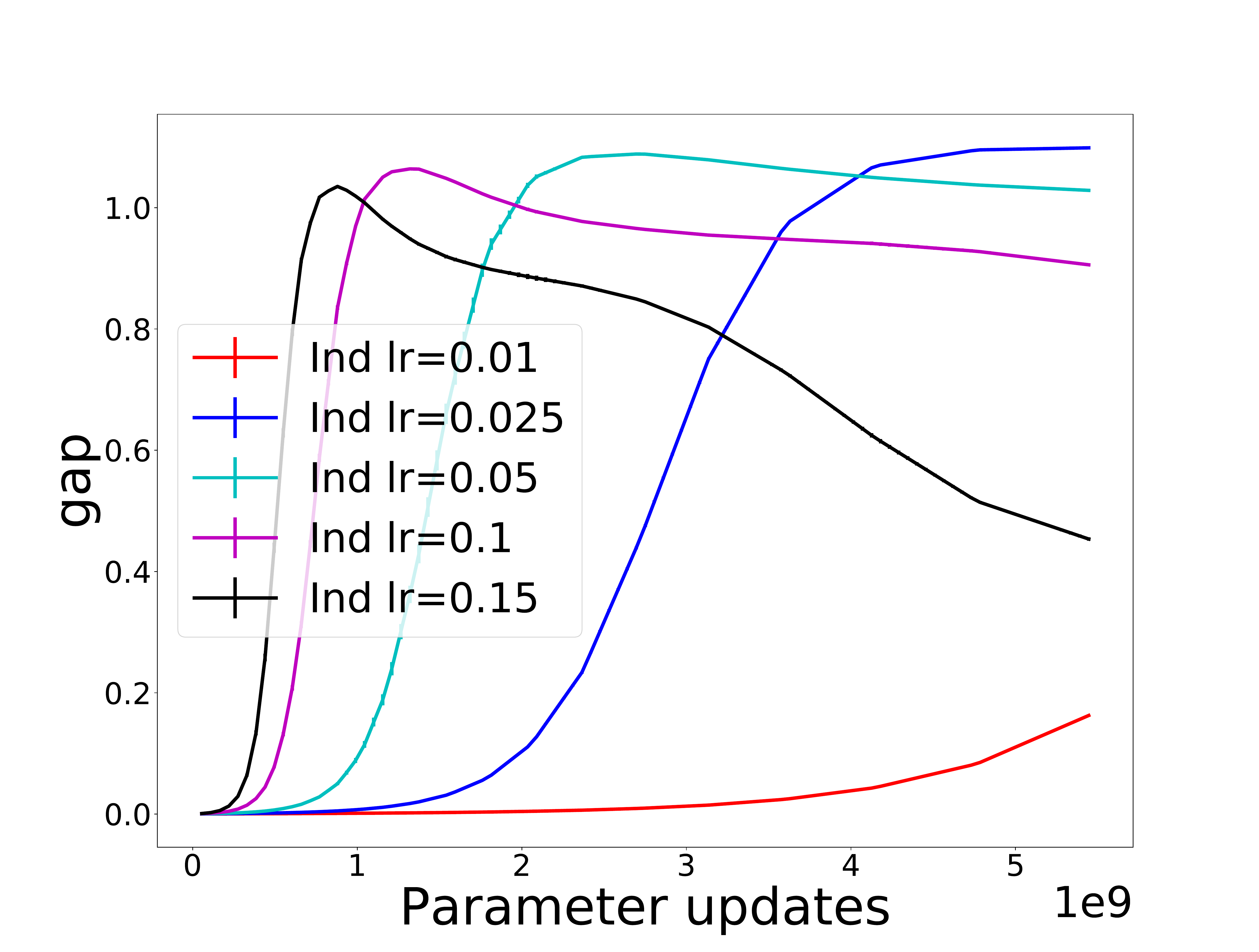}
\includegraphics[width=0.22\textwidth]{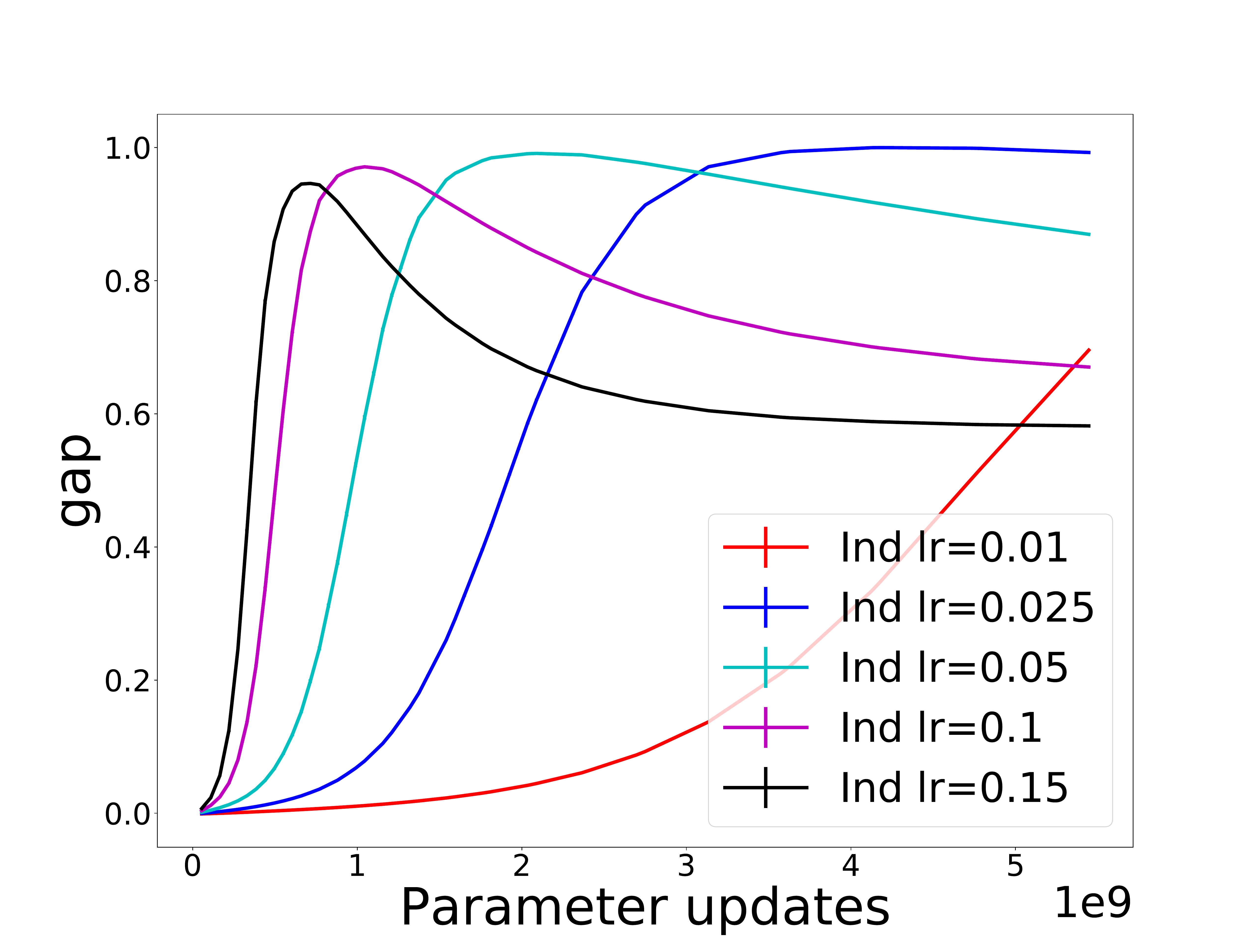}
{\small Precision at $k=10$}\\
\includegraphics[width=0.22\textwidth]{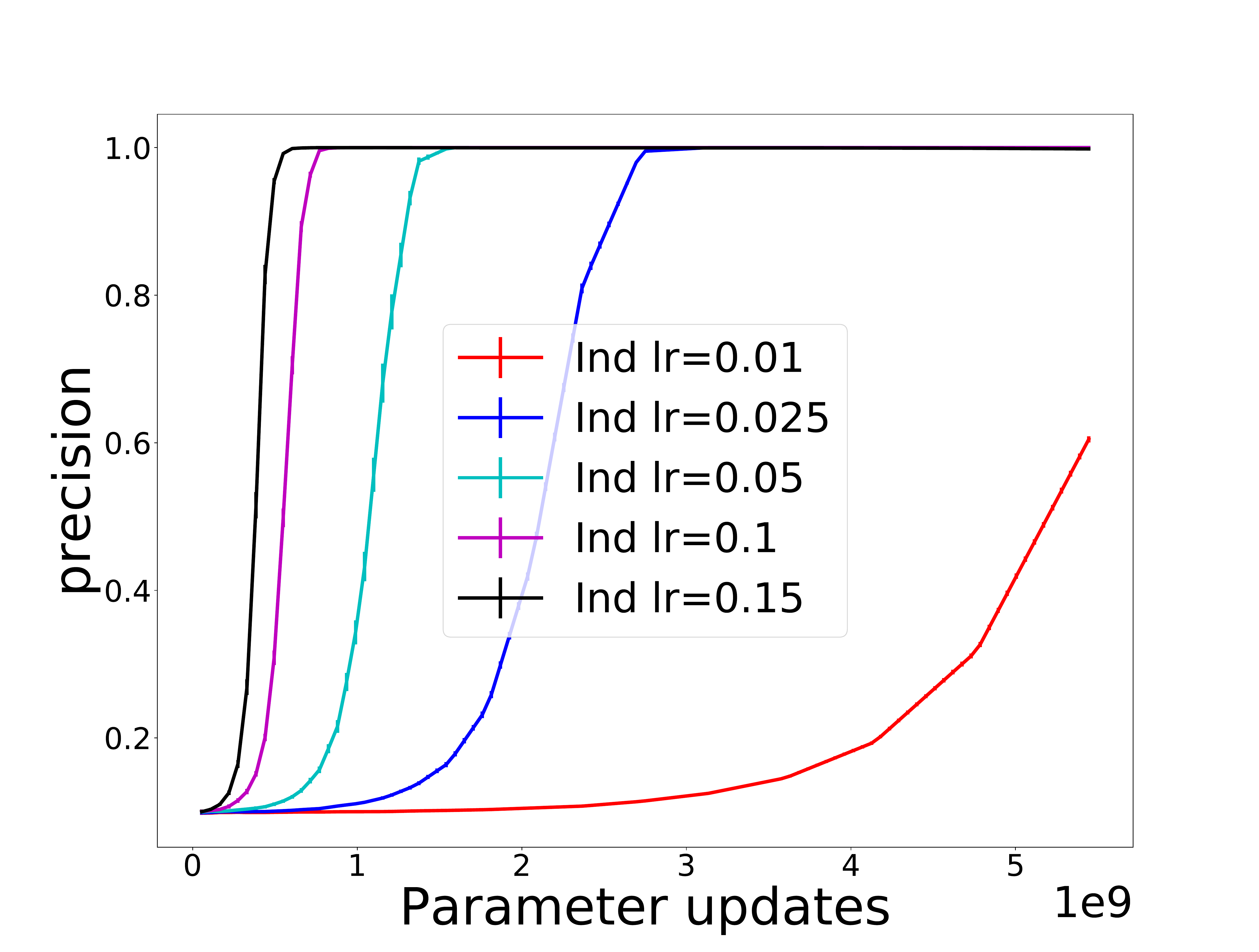}
\includegraphics[width=0.22\textwidth]{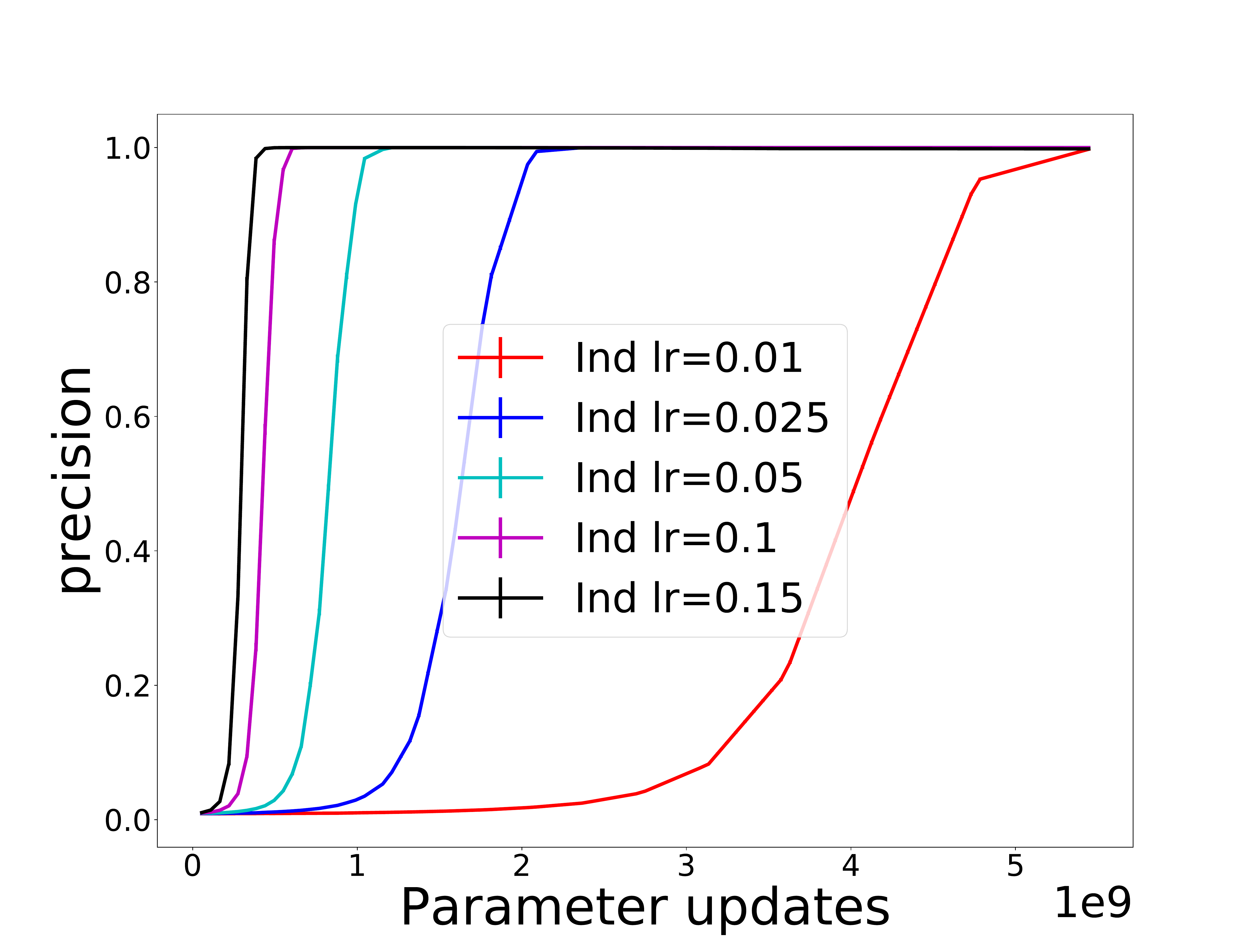}
\caption{Cosine gap and Precision for $k=10$  with learning rate sweep (\ind\ with $b=64$, $d=50$) for
  stochastic blocks with $B=10$ (left) and $B=100$ (right).}\label{lrsweep:fig}
\end{figure}

\begin{figure}[htb]
  \center
{\small Cosine gap}\\
\includegraphics[width=0.22\textwidth]{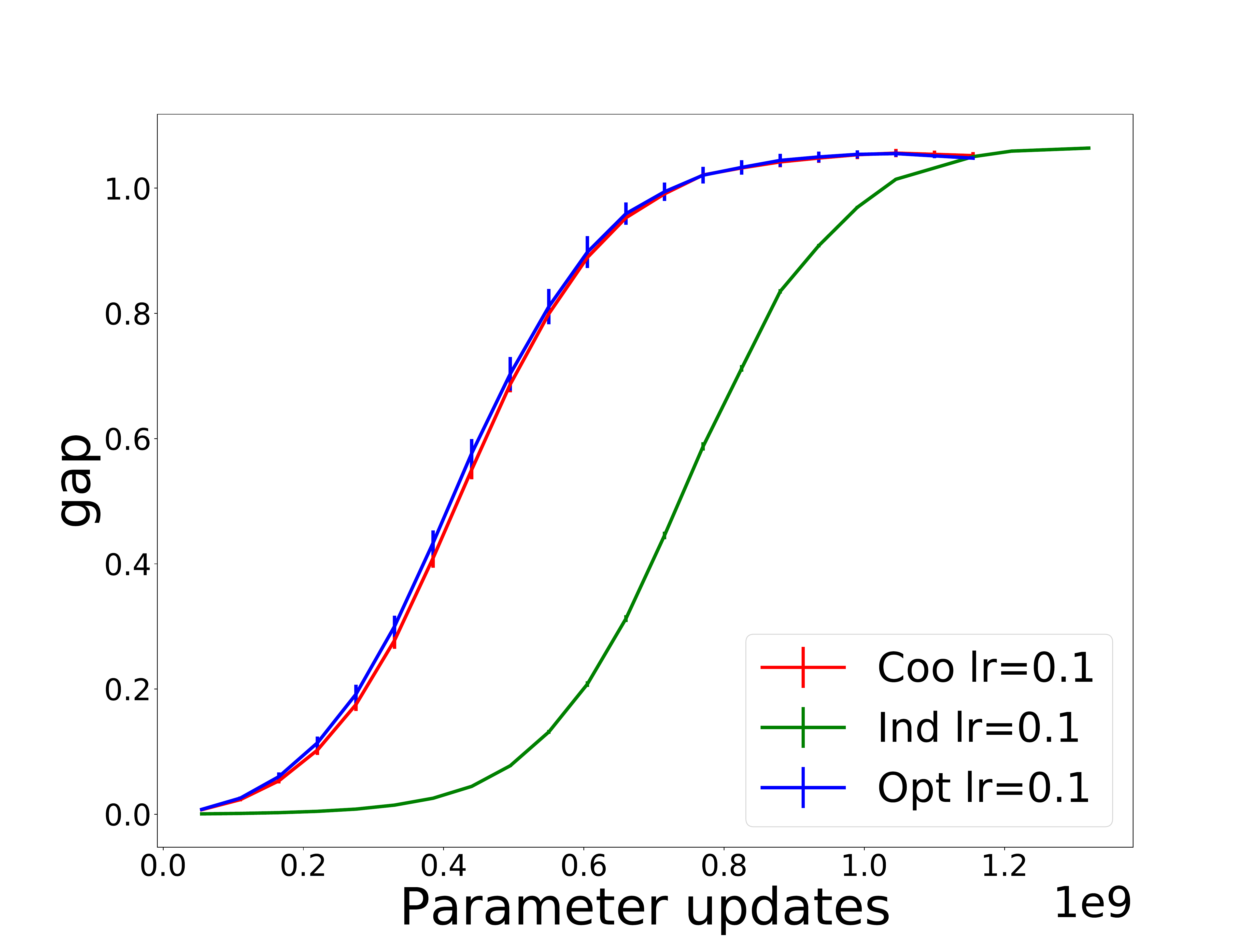}
\includegraphics[width=0.22\textwidth]{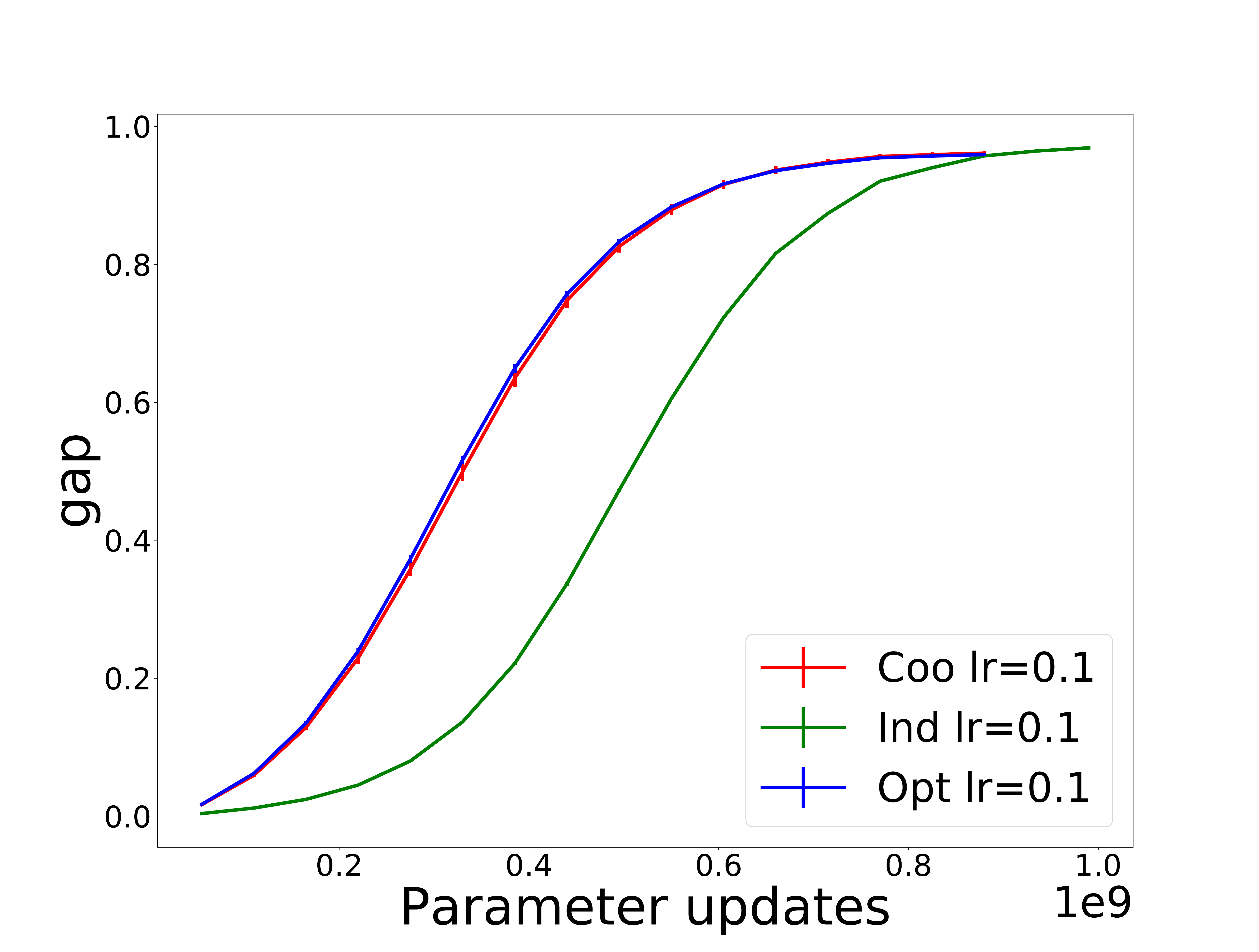}
{\small Precision at $k=10$}\\
\includegraphics[width=0.22\textwidth]{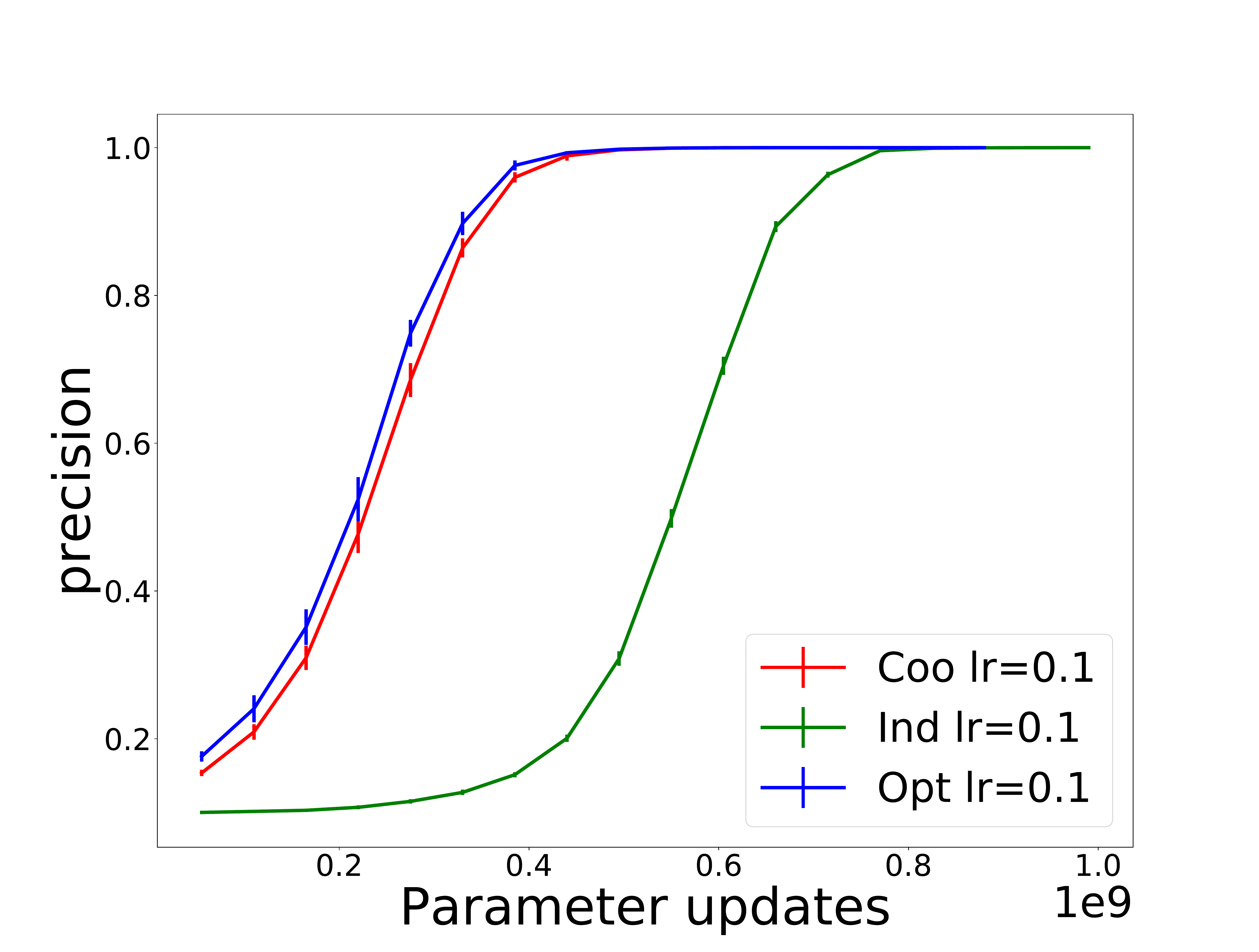}
\includegraphics[width=0.22\textwidth]{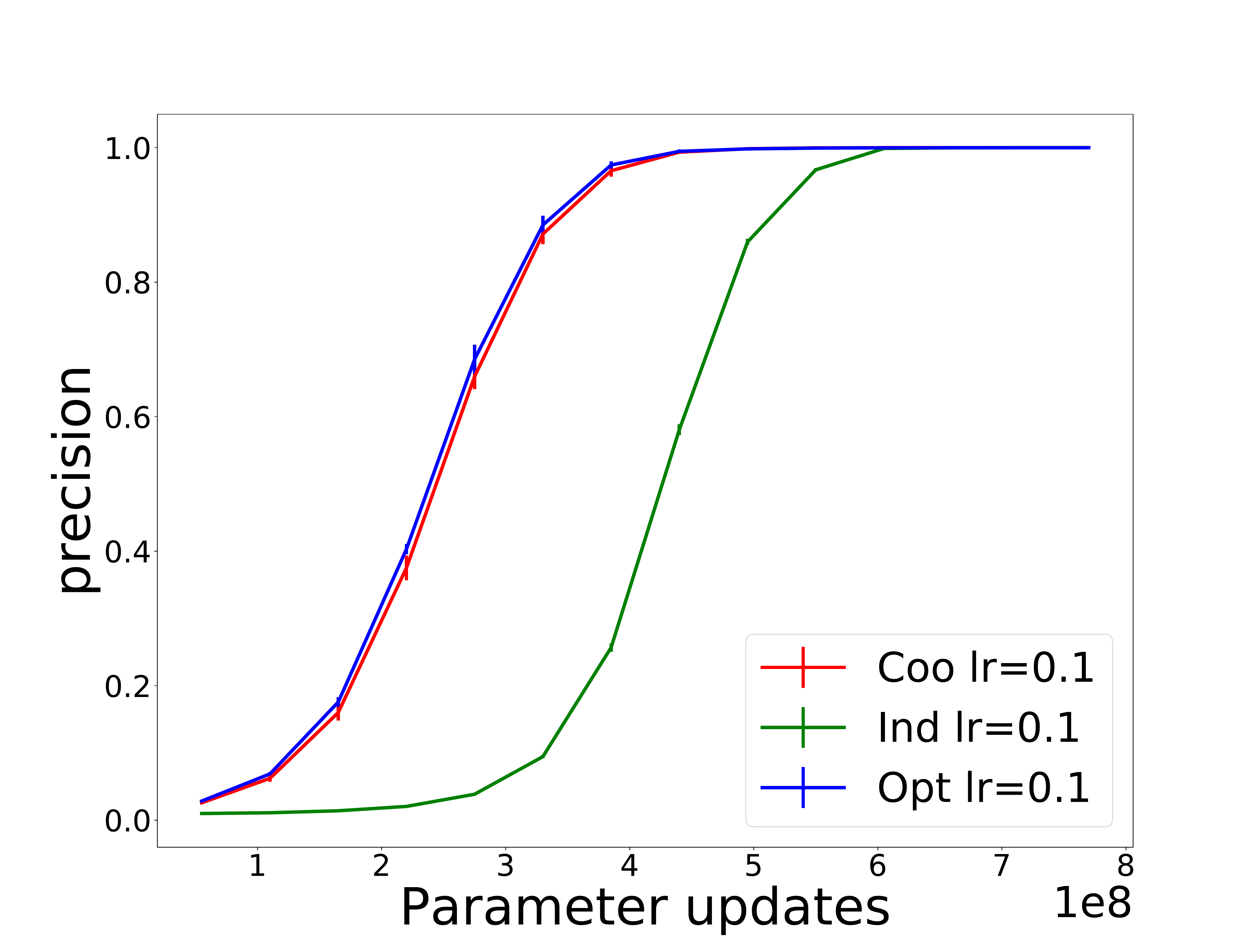}
\caption{Cosine gap and precision with \ind, \coo, and \opt\
  arrangements with learning rate $\eta=0.01$ ($b=64$, $d=50$) for
  stochastic blocks with $B=10$ (left) and $B=100$ (right).}\label{lr0_1:fig}
\end{figure}

  \newpage 

\section{Training with the Bias Parameter} \label{biasparam:sec}

The bias parameter can be viewed as an entry appended to embedding
vectors, effectively increasing the dimension to $d+1$.  The embedding $\vecf_i$ of focus entities is augmented by a
fixed valued entry of $1$ and context embedding vectors $\vecc$ are
augmented with a trainable bias parameter $b_j$.

When training with the bias parameter, we updated all parameters in
one-sided manner as described, except that the bias terms (that
can only be updated for context vectors) were updated in a two-sided
manner.
Specifically, in minibatches with context designations we update the
bias term as the full context vector is updated.  In minibatches
with focus designation (where focus embeddings are updated) we also
update the bias terms (only) of the context vectors.  We maintain the
way we match negative updates to positive updates on a per-parameter
basis.  When a bias term is updated (in a positive example) we balance
it with $\lambda$ ``antigravity'' negative updates (only performed on the bias
parameter) against a set of random focus vectors.  Therefore, focus
designated minibatches use a set of random context vectors for
negative updates of focus embeddings and a set of random focus vectors
for negative updates of bias terms of context embeddings.

When training we packed smaller same-designation microbatches into
minibatches or partitioned
large microbatches into consecutive minibatches.  Without bias terms,
with our one-sided updates the partitioning of a microbatches did not affect the end result as
effectively all updates were independent.  With
bias terms on focus designations the updates are applied after each
minibatch and were not independent.  We found that applying all
updates after each minibatch (instead of only applying them at the
effective end of a microbatch) was generally helpful.

Following practice, when using embedding vectors in our 
cosine gap and precision performance measures we use the vectors without the
bias terms. We confirmed that this practice yields smoother and better
results also on our datasets.  Moreover,  for the precision measure we
show results when seeking the closest context vector to a given focus
vector. This way, the result is the same whether we use embedding
vectors with the bias parameter or not.   We confirmed that working
this way (designating the context entity to be the one selected for a
query focus entity) yields smoother and more accurate results.

\section{Two-sided Training} \label{twosided:sec}

We used one-sided training to facilitate our coordinated arrangements
and negative pairing and reported results with all arrangement
methods, including the baseline \ind, using one-sided training.  For
completeness, we report here respective results also with respect to
the more standard baseline of independent
arrangements with two-sided training.

Results are reported in 
Figure~\ref{2sided_blocks:fig} for stochastic blocks data and in
Figure~\ref{2sided_realdata:fig} for the two recommendation datasets.
The training cost ($x$-axis) is in terms of the total
number of parameter updates performed, noting that
two-sided training performs double the updates per example than one-sided training.
We observe that two-sided \ind\ performs very similarly
and in some cases slightly worse than
one-sided \ind\ and therefore the training gains of our coordinated
arrangements hold also with respect to the two-sided  baseline.

\begin{figure}[t!] 
\center
\includegraphics[width=0.23\textwidth]{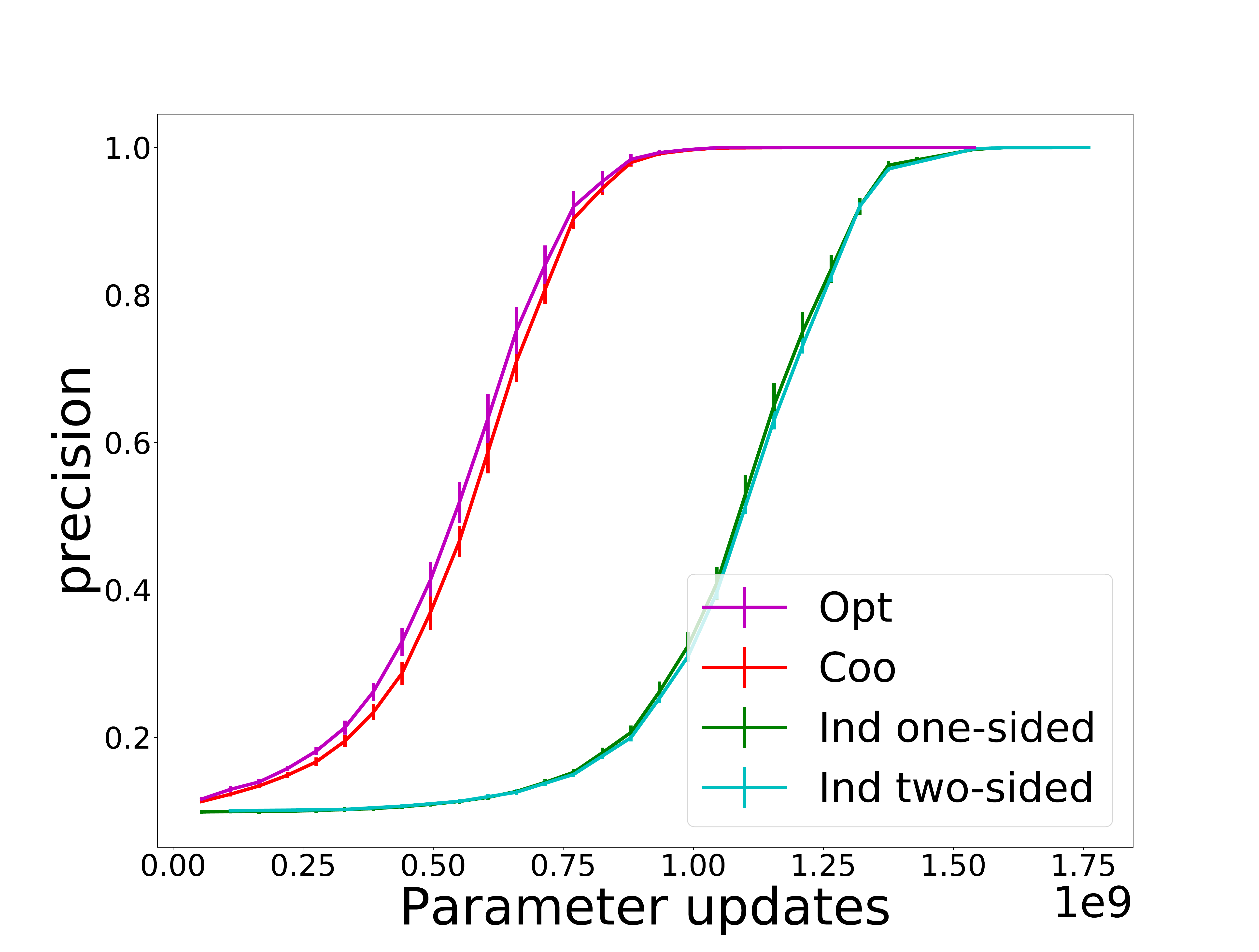}
\includegraphics[width=0.23\textwidth]{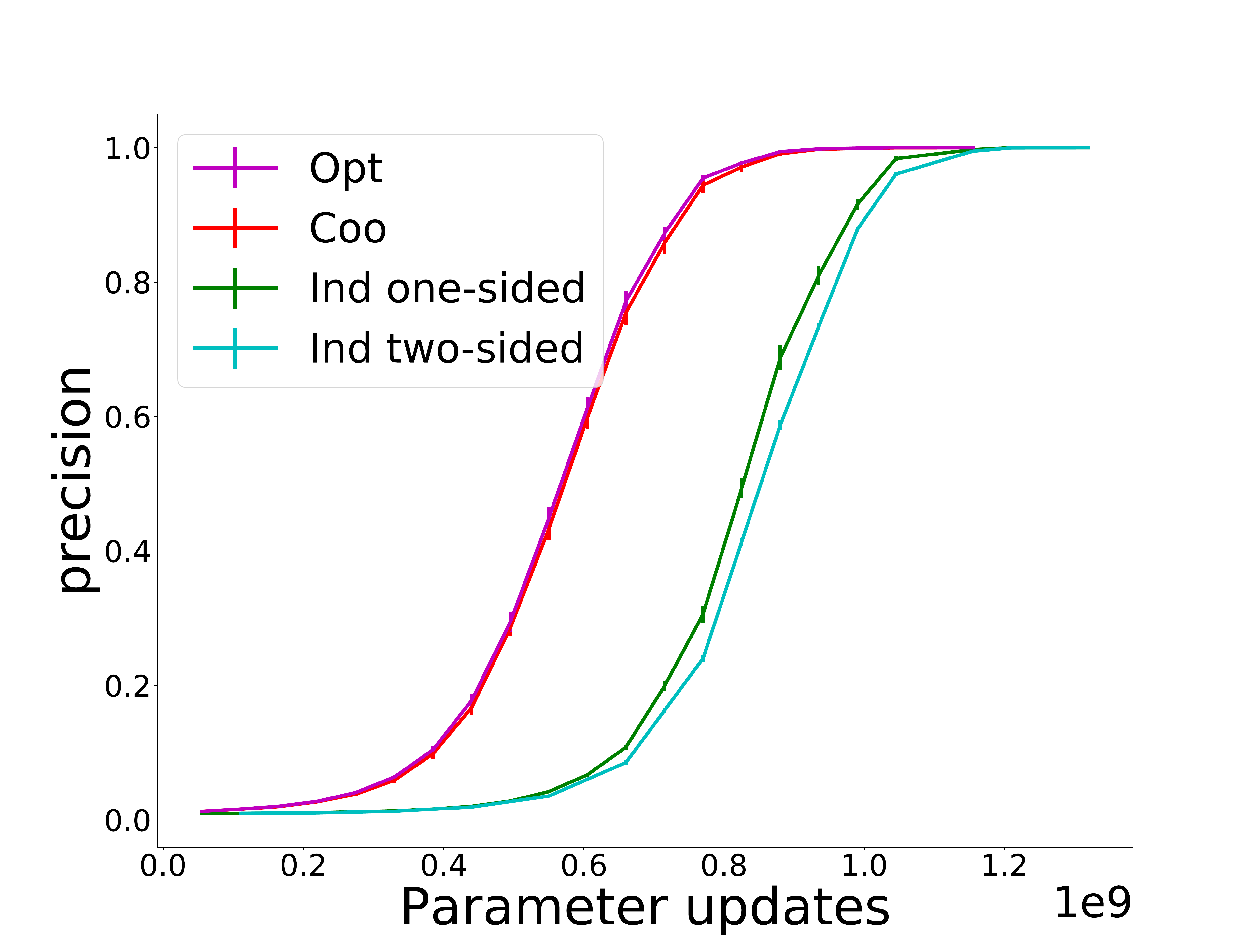}
\includegraphics[width=0.23\textwidth]{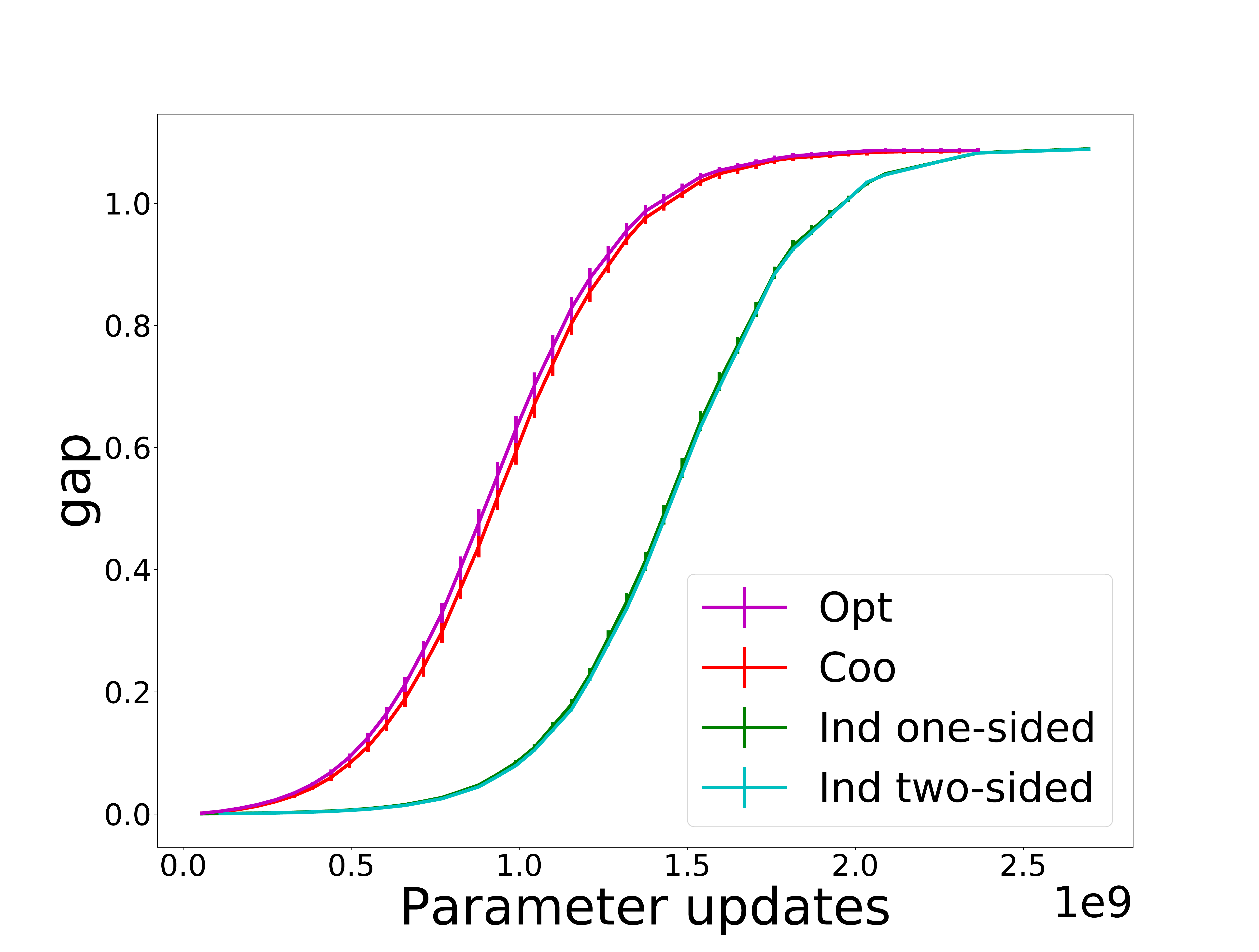}
\includegraphics[width=0.23\textwidth]{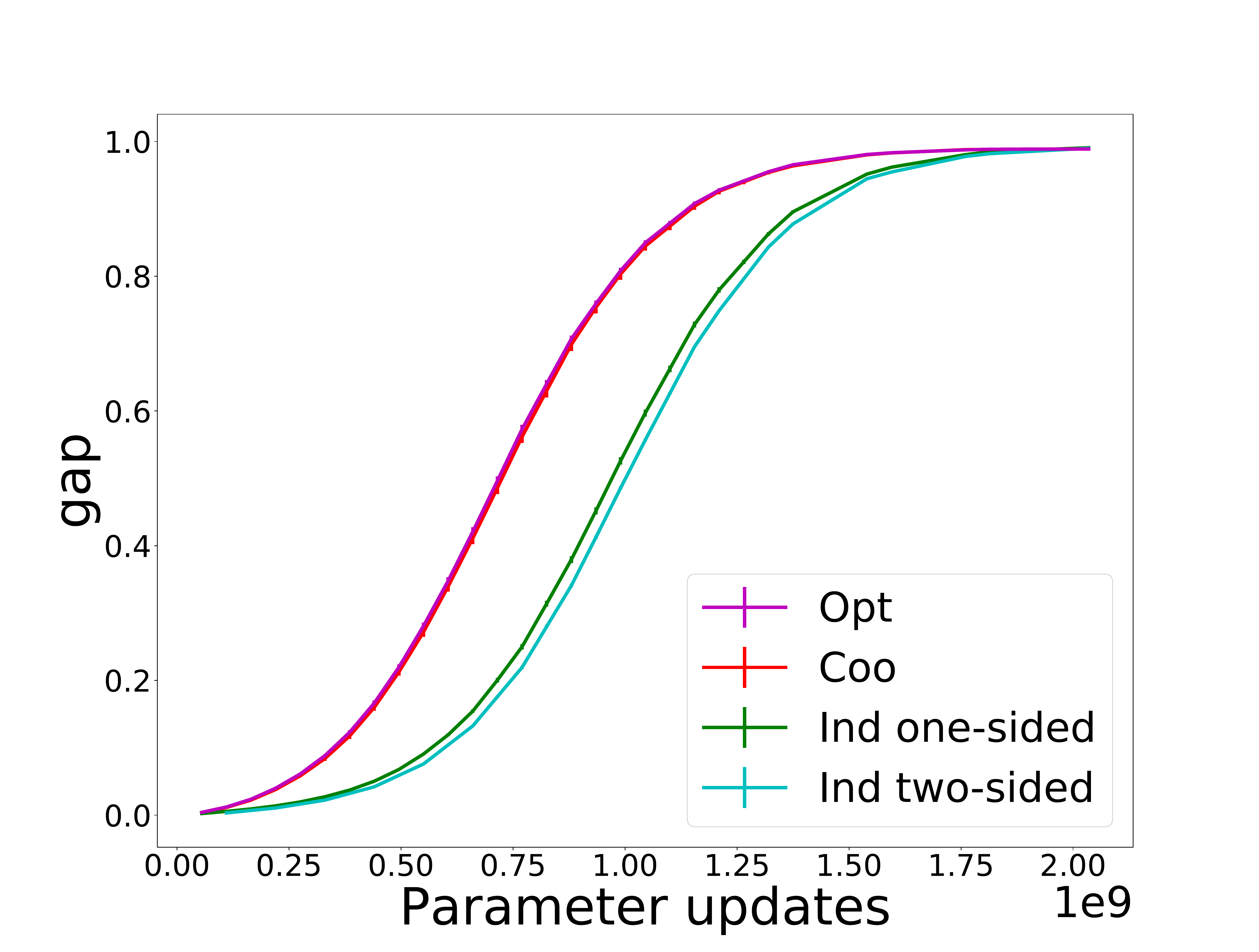}
\caption{Precision at $k=10$ (top) and Cosine gap (bottom)
 with \coo, \opt, \ind\ one-sided, and \ind\ two-sided training
($d=50$, $b=64$) for stochastic blocks 
  with $B\in\{10,100\}$. }\label{2sided_blocks:fig}
\end{figure}

\begin{figure}[H] 
  \center
\includegraphics[width=0.23\textwidth]{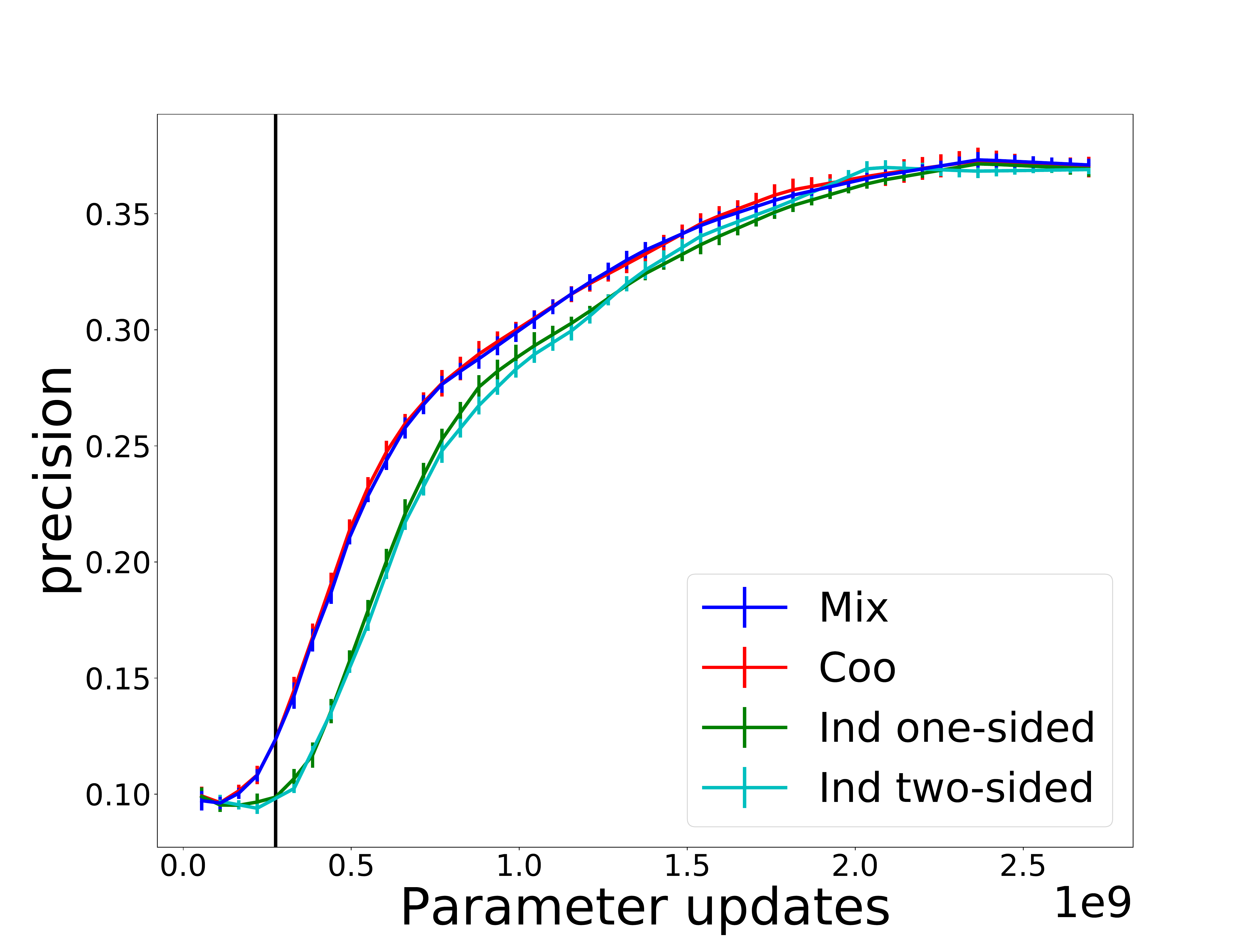}
\includegraphics[width=0.23\textwidth]{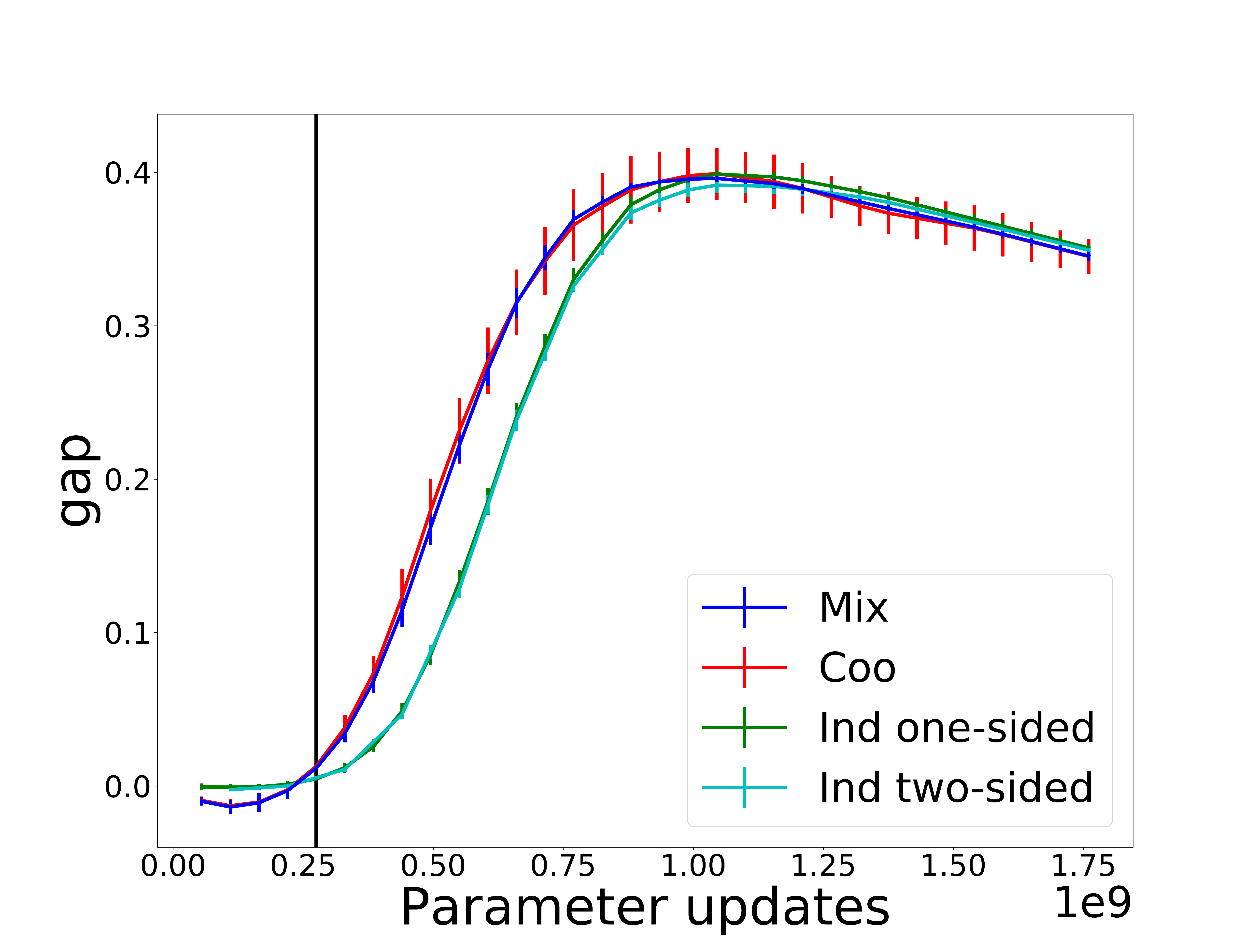}
\includegraphics[width=0.23\textwidth]{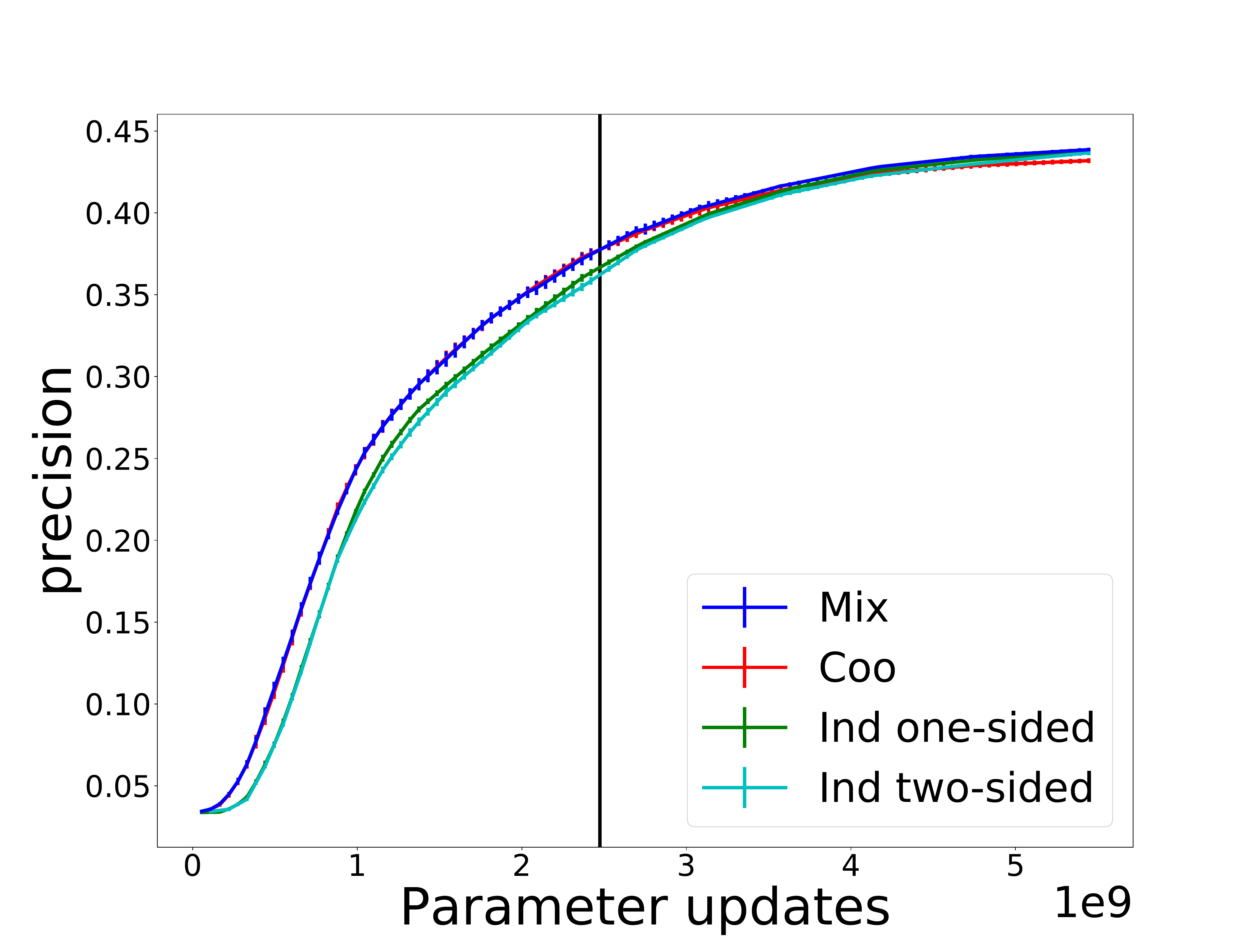}
\includegraphics[width=0.23\textwidth]{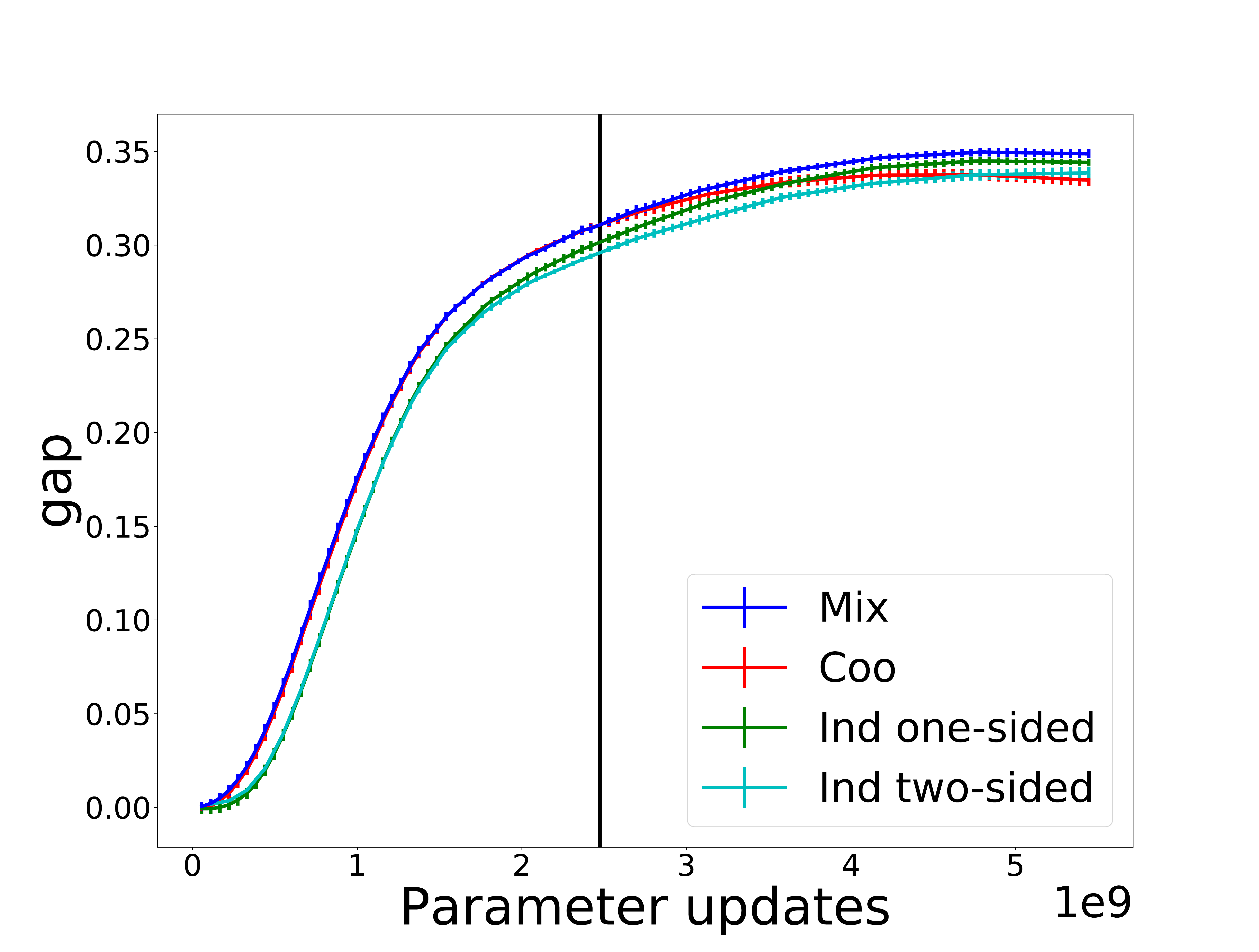}
\caption{Precision at $k=10$ (left) and cosine gap (right) 
  with \coo, \mix, and \ind\ (one-sided and two-sided) arrangements on \movielens (top) and \amazon
  (bottom) datasets.
($d=50$, $b=64$). The
  vertical lines indicate the switch point of \mix\ (from \coo\ to  (one-sided) \ind.}
\label{2sided_realdata:fig}
\end{figure}

  \end{document}